\documentclass[sigconf]{acmart}

\AtBeginDocument{%
  }

\copyrightyear{2024}
\acmYear{2024}
\setcopyright{acmlicensed}\acmConference[EAAMO '24]{Equity and Access in Algorithms, Mechanisms, and Optimization}{October 29--31, 2024}{San Luis Potosi, Mexico}
\acmBooktitle{Equity and Access in Algorithms, Mechanisms, and Optimization (EAAMO '24), October 29--31, 2024, San Luis Potosi, Mexico}
\acmDOI{10.1145/3689904.3694703}
\acmISBN{979-8-4007-1222-7/24/10}
\usepackage{acmart-taps}

\usepackage{url}
\usepackage{nicefrac}
\usepackage{microtype}
\usepackage{xcolor} 
\usepackage{subcaption}
\usepackage{multirow}
\usepackage{array}
\usepackage{listings}
\usepackage{amsmath, float, enumitem }
\usepackage[ruled,vlined]{algorithm2e}
\usepackage{wrapfig}
\usepackage{amsthm}
\usepackage[capitalize,noabbrev]{cleveref}
\usepackage{bm}
\usepackage{tikz-cd}
\usepackage[bb=boondox]{mathalfa}

\newtheorem{theorem}{Theorem}[section]
\newtheorem{axiom}{Axiom}
\crefname{axiom}{Axiom}{Axioms}
\Crefname{axiom}{Axiom}{Axioms}

\newcounter{lemmanum}
\newtheorem{lemma}{Lemma}[lemmanum]

\newtheorem{proposition}{Proposition}[section]

\newtheorem{Definition}{Definition}[section]

\DeclareMathOperator*{\argmin}{arg\,min}

\newcommand\indicator{\mathbb{I}}
\newcommand\basis{\mathbb{1}}
\newcommand\expect{\mathbb{E}}
\newcommand\bigO{{O}}
\newcommand\Data{\mathcal{D}}
\newcommand\prob{\mathbb{P}}
\newcommand{\relCand}{nRel}
\newcommand\EORanking{\sigma^{EOR}}
\newcommand\PRPRanking{\sigma^{PRP}}
\newcommand\PRPRankingGroup[1]{\sigma^{PRP,#1}}
\newcommand\DPRanking{\sigma^{DP}}
\newcommand\TSRanking{\sigma^{TS}}
\newcommand\EXPRanking{\sigma^{EXP}}
\newcommand\FourFifthRanking{\sigma^{\text{FourFifth}}}

\newcommand\UniformRanking{\sigma^{\text{unif}}}
\newcommand\PRP{\pi^{PRP}}
\newcommand\Uniform{\pi^{\text{unif}}}
\newcommand\EOR{\pi^{EOR}}
\newcommand\TS{\pi^{TS}}
\newcommand\DP{\pi^{DP}}
\newcommand\PRR{\pi^{PRR}}
\newcommand\EXP{\pi^{EXP}}
\newcommand\RA{\pi^{RA}}
\newcommand\FS{\pi^{FS}}
\newcommand\sigmaKPi{\sigma_k^\pi}
\newcommand{\real}{\mathbb{R}}

\begin{document}
\title{Fairness in Ranking under Disparate Uncertainty}

\author{Richa Rastogi}
\affiliation{
  \institution{Cornell University}
  \city{Ithaca}
  \state{NY}
  \country{USA}
}
\email{rr568@cornell.edu}

\author{Thorsten Joachims}
\affiliation{
  \institution{Cornell University}
  \city{Ithaca}
  \state{NY}
  \country{USA}
}
\email{tj@cs.cornell.edu}

\renewcommand{\shortauthors}{Rastogi and Joachims.}

\begin{abstract}
  Ranking is a ubiquitous method for focusing the attention of human evaluators on a manageable subset of options. Its use as part of human decision-making processes ranges from surfacing potentially relevant products on an e-commerce site to prioritizing college applications for human review. While ranking can make human evaluation more effective by focusing attention on the most promising options, we argue that it can introduce unfairness if the uncertainty of the underlying relevance model differs between groups of options. Unfortunately, such disparity in uncertainty appears widespread, often to the detriment of minority groups for which relevance estimates can have higher uncertainty due to a lack of data or appropriate features. To address this fairness issue, we propose Equal-Opportunity Ranking (EOR) as a new fairness criterion for ranking and show that it corresponds to a group-wise fair lottery among the relevant options even in the presence of disparate uncertainty. 
  EOR optimizes for an even cost burden on all groups, unlike the conventional \emph{Probability Ranking Principle}, and is fundamentally different from existing notions of fairness in rankings, such as \emph{demographic parity} and \emph{proportional Rooney rule} constraints that are motivated by proportional representation relative to group size. 
  To make EOR ranking practical, we present an efficient algorithm for computing it in time $O(n \log(n))$ and prove its close approximation guarantee to the globally optimal solution. 
  In a comprehensive empirical evaluation on synthetic data, a US Census dataset, and a real-world audit of Amazon search queries, we find that the algorithm reliably guarantees EOR fairness while providing effective rankings.

\end{abstract}

\begin{CCSXML}
<ccs2012>
<concept>
<concept_id>10002951.10003317.10003347.10003350</concept_id>
<concept_desc>Information systems~Rankings</concept_desc>
<concept_significance>500</concept_significance>
</concept>
</ccs2012>
\end{CCSXML}

\ccsdesc[500]{Information systems~Rankings; Top-k retrieval; Recommender systems; Decision support systems}

\keywords{ranking, fairness, disparate uncertainty, cost of opportunity}

\maketitle

\section{Introduction}
\label{sec:Intro}

Human decision-processes are increasingly augmented with algorithmic decision-support systems, which has created opportunities and challenges for addressing group-based disparities in decision outcomes. In this paper, we focus on selection processes where humans evaluators use rankings to organize the order of review under resource constraints. We argue that {\em disparities in uncertainty} can be a major source of group-based discrimination in this setting.

To illustrate the problem, consider the following example of college admissions at a highly selective institution. In this situation, there are far more qualified candidates than available spots. Under a fixed reviewing budget, the college could give all applications a brief review (but risk high error rates in human decision making), or use a ranking to focus reviewing efforts on the more promising applications. The latter is likely to decrease error rates in human review, but it risks that this prioritization unfairly favors some groups over others. For example, consider 12,000 applicants competing for 500 slots. In this example, 10,000 applicants are from a majority group with plenty of available data, and the model can quite accurately predict which students will be admitted by the human reviewers. In particular, it accurately assigns a probability of 0.9 to 1000 of the students, and 0.01 to the remaining 9,000. The remaining 2000 applicants are from a minority group, where the model is less informed about individual students and thus assigns 0.1 to everybody. When naively ranking students by this probability, the students with 0.9 from the majority group would be ranked ahead of all the students from the minority group - and the class will fill up with the expected 900 ($1000 \times 0.9$) qualified majority students before the admission staff even gets to any of the minority students. This is clearly unfair even if the predictions are perfectly calibrated for each group, since not even a single student of the expected 200 ($2000 \times 0.1$) qualified students in the minority group has a chance to be selected by the admissions staff. 

We aim to define a new way of ranking that does not introduce unfairness into a human decision-making process even if the predictive model shows differential uncertainty between groups. This goal recognizes that training models to have equal uncertainty across groups may be difficult in practice, since a lack of data and appropriate features for some groups may be difficult to overcome\footnote{Arguably, the same applies to instructing human evaluators to provide such ranking scores during a first phase of review.}. Importantly, a key principle behind our work is to leave the final decisions to human decision makers. We thus aim to design new ranking algorithms to most effectively support a fair human decision-making process, and not to replace the human decision maker. 

The main contributions of this paper are
\begin{itemize}[leftmargin=*]
    \item A \textbf{new fairness criterion} that provides a meaningful guarantee for rankings that are used to support human decision making in selection processes even \textbf{under disparities in uncertainty}. We motivate this fairness criterion with a fair lottery \cite{Goodwin1992-GOOJBL,Saunders2008-SAUTEO}, ensuring group-wise outcomes that are equivalent to allocating scarce resources based on a group-fair lottery among the relevant candidates.
    \item Based on this notion of fairness, we develop a new ranking procedure that is group-fair under disparate uncertainty. Motivated by its relation to the equality of opportunity framework \cite{NIPS2016_9d268236}, we name this \textbf{ranking procedure Equal Opportunity Ranking (EOR)}. We analyze EOR from the lens of the cost burden on each entity involved -- the principal decision maker and each of the candidate groups -- and formulate the cost to each entity as the lost opportunity of access given that the candidate was truly relevant. We show that this EOR procedure equalizes the cost burden between groups and present an efficient and practical algorithm for computing EOR rankings. This procedure always produces a near optimal and approximately EOR-fair solution. In particular, we prove an \textbf{approximation guarantee} showing that the gap in total cost to the principal compared to an optimal algorithm is bounded by a small amount. 
    \item In addition to these theoretical worst-case guarantees, we present \textbf{extensive experiments} benchmarking the EOR algorithm with various existing ranking algorithms under different settings of disparate uncertainty. We show that Demographic Parity \cite{DBLP:journals/corr/YangS16a, Zehlike_2017}, normative procedures like Proportional Rooney-rule-like constraints \cite{10.1145/3351095.3372858}, Exposure based fairness criteria \citep{10.1145/3219819.3220088}, and Thompson Sampling Policy \citep{NEURIPS2021_63c3ddcc} are not typically EOR-fair under disparate uncertainty.
    We find that these results hold on both a wide range of synthetic datasets, as well as on real-world US census data.
    Finally, we explore the use of our fairness criterion for auditing ranking systems, using a real-world dataset of Amazon shopping search queries. Our code can be accessed at \url{https://github.com/RichRast/DisparateUncertainty}.
\end{itemize}
These results have important societal implications. First, they provide evidence that naively applying existing fairness mechanisms in rankings under disparate uncertainty leads to unfairness in terms of one group bearing the majority of the cost of opportunity.  Second, even under high disparate uncertainty in the worst case, EOR guarantees an approximately equal cost burden among all groups with bounded additional cost to the human decision maker. 
Finally, we hope our results inform practitioners to collect data and appropriate features for candidates in all groups to build predictive models that reduce disparate uncertainty. 
As we will show, the EOR procedure elevates the candidates with high uncertainty in the rankings for human evaluation. This has the desirable effect of producing more equitable training data for future use.

We now highlight some important considerations here. First, our proposed method is grounded in the fairness of a lottery \citep{Saunders2018-SAUEIT-2}, which is a common technique for allocating scarce resources (e.g., admission slots among a large number of qualified candidates).
However moral and philosophical arguments debating the use of lottery and randomization for certain situations have also been made \citep{Henning2015-HENFCT}.
We hope this work can spark discussions on alternative notions of fairness in rankings that satisfy equality of opportunity under disparate uncertainty. 
Another important point is that our proposed EOR procedure reduces unfairness due to disparate uncertainty, which often but not necessarily coincides with the historically disadvantaged group. 
Since EOR doesn't require the designation of the disadvantaged group, the guarantees we provide are not making a normative statement about any historically disadvantaged group. 
To that end, we emphasize the careful consideration of historical and social context that needs to be taken into account by the human decision maker as well as the way groups are defined in the first place. 

\section{Related Works} \label{sec:related_works}
While the issue of fairness has been heavily studied in the classification setting, its counterpart -- the ranking setting has received relatively less attention. Below we highlight key areas related to our work and leave a more detailed discussion of these and other related works to \Cref{sec:extended_related}.

\emph{Fairness in Rankings and Selection Processes}: While there exist several notions of fairness in rankings \citep{DBLP:journals/corr/abs-2103-14000}, predominantly, they are variations of two fairness mechanisms in existing literature -- representation by size
\citep{Celis2017RankingWF, inproceedings, Zehlike_2017, 10.1016/j.ipm.2021.102707} and equitable allocation of exposure \citep{10.1145/3219819.3220088, Singh2017EqualityOO, DBLP:journals/corr/abs-1805-01788, DBLP:journals/corr/abs-2102-05996, 48758}. 
We propose a new criterion different from either of the two and our central point is that under disparate uncertainty between groups, it is more fair to take an equal proportion of relevance in expectation rather than equality by size or exposure.
Proportional representation in the form of diversity constraints like demographic parity \citep{DBLP:journals/corr/YangS16a} or affirmative action such as the Rooney Rule \citep{10.1145/3351095.3372858} guarantee a minimum proportion by group size in selection processes. Exposure based formulations in rankings ensure that groups of candidates are allocated exposure in an equitable way such as in proportion of amortized relevance over the full ranking \citep{DBLP:journals/corr/abs-1805-01788}.
In this work, we demonstrate that fairness of representation by size and exposure, are not sufficient under disparate uncertainty. 

\emph{Fairness in Rankings under Uncertainty}: Our work builds on \citep{NEURIPS2021_63c3ddcc}, in which the authors establish that uncertainty in relevance probabilities is a primary cause of unfairness for rankings. They propose a Thompson sampling policy that randomizes relevances drawn from the predictive posterior distribution. 
Separately, \citep{EMELIANOV2022103609} studies the role of affirmative action in the presence of differential variance between groups in rankings. 
Differential variance implies that there is more certainty about the true quality (scores) of candidates in a group with less variance in the estimated quality and vice versa for a group with higher variance. 
In contrast, we work with relevance probabilities instead of scores and focus on the certainty of relevance of a candidate, which is determined by how close the predicted relevance probabilities are to 1 or 0. 
For instance, a group is highly certain (if the probabilities are all close to 1.0) or highly uncertain (if the probabilities are all close to 0.5) while both groups could have similar variance in probabilities. 
Fairness under uncertainty has also been studied with respect to calibration of probabilities \citep{NIPS2017_b8b9c74a, kleinberg_et_al:LIPIcs.ITCS.2017.43, Flores2016FalsePF, DBLP:journals/bigdata/Chouldechova17}. Classical literature in this area studies whether group-wise calibration is a necessary condition for fairness, or not \citep{10.1145/3531146.3533245}. Our work is orthogonal to the question of the necessity of calibration for fairness and we only require group-wise calibration as a sufficient condition for the EOR criterion we propose. 

Our work complements and extends prior research on fairness in rankings under uncertainty, contributing uniquely in several ways.
In particular, we provide a formal framework for analyzing the unfairness that differential uncertainty induces in rankings.
Additionally, our approach involves accounting for the differential uncertainty directly at the ranking stage, unlike prior work that involves learning the uncertainty \citep{10.1145/3539597.3570469} or correcting the noisy relevance estimates \citep{10.1145/3539597.3570474}. Finally, our proposed EOR criterion is non-amortized for every prefix $k$ of the ranking, which is strictly stronger than the probabilistic but amortized notions of fairness \citep{Singh2017EqualityOO, DBLP:journals/corr/abs-1805-01788, 10.1145/3219819.3220088} shown to be problematic \citep{UnFair}.

\section{Un-fairness due to Disparate Uncertainty in Rankings} 
\label{sec:uncertainty_merits}
We want to design a ranking policy $\pi$ that does not introduce unfairness into a human decision process due to disparate uncertainty. More formally, the task of $\pi$ is to compute a ranking $\sigma$ of $n$ candidates, where each candidate $i$ has a binary\footnote{We conjecture that our framework can be extended to categorical or real-valued relevances.} relevance $r_i \in \{0,1\}$ which is unknown to the ranking policy $\pi$, and true relevance can only be revealed through a human decision maker. When assessing the relevance, we assume that the human decision maker goes through the ranking $\sigma$ from the top to some a priori unknown position $k$. The goal of the decision maker (a.k.a. principal) is to find as many relevant candidates (e.g., relevant products, qualified students) as possible.

While the true relevances $r_i$ are unknown, we assume that the ranking policy $\pi$ has access to a predictive model of relevance $\prob(r_i|\Data)$, typically trained on prior human decisions $\Data$ and features of the candidates. Sorting the candidates in decreasing order of $p_i=\prob(r_i=1|\Data)$ is called the Probability Ranking Principle (PRP) \cite{article_pr}, and it is by far the most common way of computing a ranking. The justification for PRP ranking is that it maximizes the expected number of relevant candidates in any top-k prefix of the ranking. 
On the other hand, Demographic Parity (DP) is the dominant form of fairness mechanism in rankings, where candidates are selected from groups in proportion to the group size.
While PRP ranking is provably optimal according to the efficiency goal of the principal and DP ranking ensures representation by group size, the following elaborates how both PRP and DP can violate fairness.
\begin{figure*}[t]
\centering
    \includegraphics[width=1.0\linewidth, trim = 0.3cm 25.0cm 7.0cm 2.0cm, clip]
    {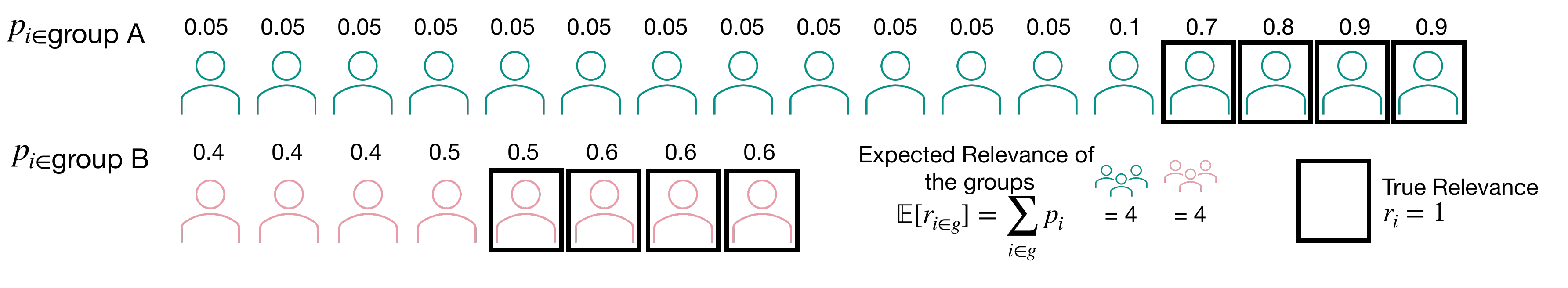}
    \vspace*{-5mm}
    \caption{\normalfont{The expected probability of relevance $p_i$ and their true relevance $r_i$ for all candidates in both groups.}}
  \label{fig:expected_probs}
\end{figure*}

\subsection{Illustrative Example} 
\label{sec:example}
Consider a medical setting, where candidates need to be evaluated for eligibility to participate in a controlled medical trial. While group A consists of candidates with a rich set of diagnostic tests that inform eligibility (e.g., candidates with health insurance), group B consists of candidates without prior access to such tests (e.g., candidates without health insurance). As a result, according to $\prob(r_i=1|\Data)$ in \Cref{fig:expected_probs}, the model can make very informed predictions for candidates in group A, while for group B the model cannot reliably differentiate between eligible and not eligible candidates. 
This means the model knows exactly which candidates in group A will be judged as eligible by the human decision maker, but it will make undifferentiated (but well-calibrated) predictions for candidates in group B. 

\begin{figure}[t]
\centering
    \includegraphics[width=1.0\linewidth, trim = 1.0cm 21.0cm 34.0cm 0.5cm, clip]
    {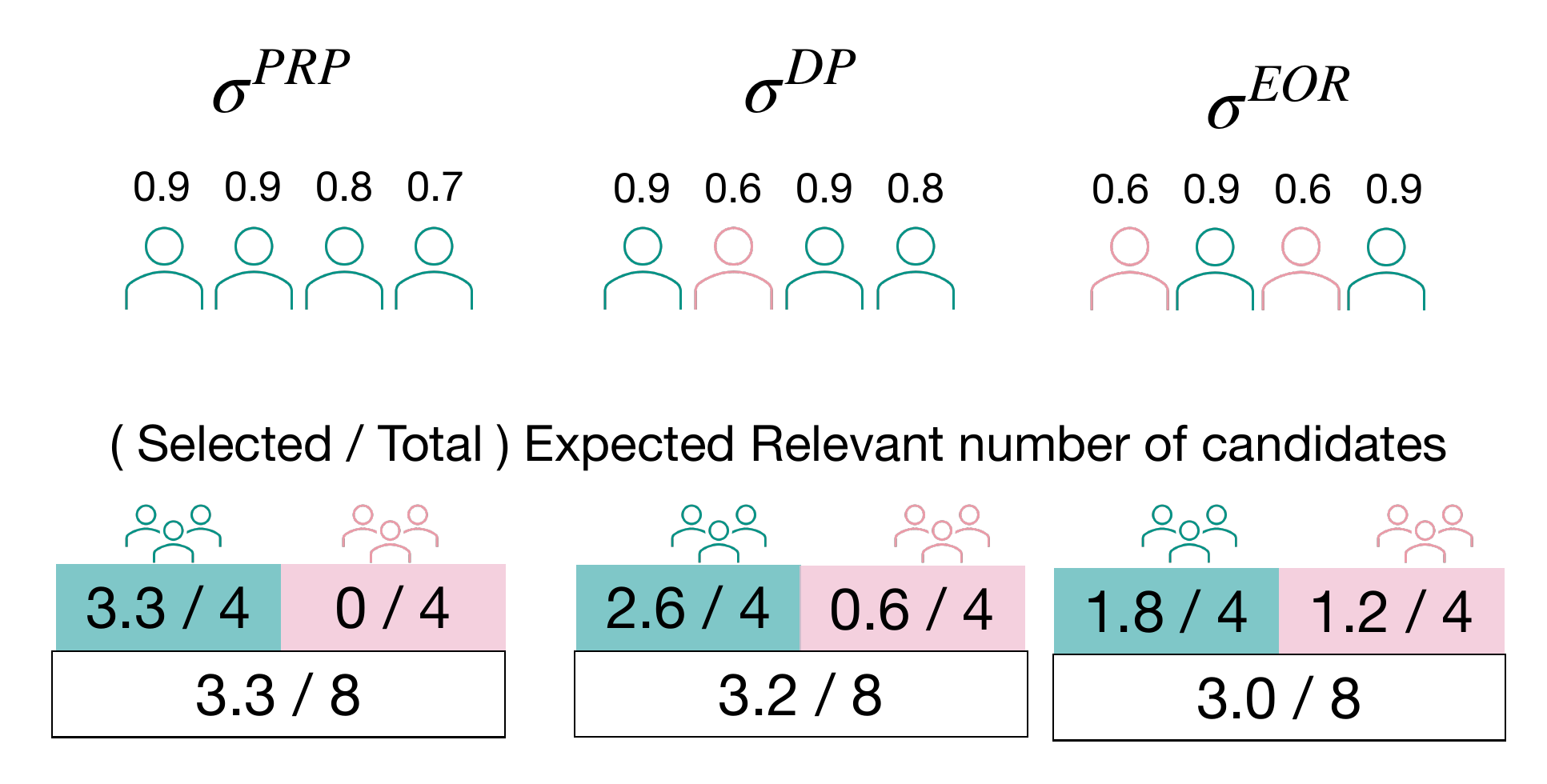}
    \vspace*{-5mm}
    \caption{\normalfont{Top-4 ranking for Probability Ranking Principle (PRP), Demographic Parity (DP), and our proposed EOR for the example in \Cref{fig:expected_probs}. Selected relevant number of candidates in expectation and total relevant number of candidates in expectation are shown corresponding to each ranking.}}
  \label{fig:example_ranking}
\end{figure}

\Cref{fig:example_ranking} shows that the PRP ranking is oblivious to this disparity between groups.
If the principal needs to find four eligible candidates based on the PRP ranking, they are all selected from group A. However, by summing the probabilities in group B, our model tells us that we can also expect four eligible candidates in group B. We argue that deterministically selecting only candidates from group A is unfair since it is not consistent with the outcome of a group-fair lottery for the four spots among the eight eligible candidates. 
Now, consider the DP ranking in \Cref{fig:example_ranking}. Since group A has 17 candidates and group B has 8 candidates, DP will select roughly one candidate from group B for every two candidates from group A. We argue that in this setting, DP is also unfair, (though less in comparison to PRP) as it selects three eligible candidates from group A and only one from group B. In expectation, it selects 2.6 out of 4 relevant candidates from group A, but only 0.6 out of 4 relevant candidates from group B.
We show empirically later that other fairness mechanisms motivated by representation of size such as proportional Rooney Rule or threshold-based formulations have the same failure mode. 
Importantly, note that it is not evident whether group A or B should be the majority group. 

We argue that a more principled and fair way would be to select an equal fraction of relevant candidates from each group in expectation. 
Consider the last ranking in \Cref{fig:example_ranking}, which approximately fulfills the EOR fairness we formally introduce later. 
In expectation, this ranking selects a more equal number of relevant candidates from both groups, making it similar to a fair lottery. In particular, it selects 1.8 out of 4 relevant candidates from group A and 1.2 out of 4 relevant candidates from group B. This EOR ranking, however, comes at an increased evaluation cost to the principal as it selects 3.0 expected relevant candidates from both the groups, compared to 3.2 with DP and 3.3 with PRP. As a result, the principal needs to review more candidates to select the same number of relevant candidates with EOR ranking. However, it is still far more effective than a lottery, which selects the candidates in a uniform random order. 

Our key insight is that EOR ranking is more fair not because it takes an equal ``number'' of candidates from each group but it is more fair because it takes an equal fraction of ``relevant'' candidates in expectation from each group.
This accounts for predictive uncertainty in the relevance probabilities because even when one group has sharp and the other group has non-sharp $p_i$, it takes approximately equal fraction of relevance from each of the groups.

This example illustrates the intuition behind the EOR principle we formalize in the following, and we will show how to efficiently compute rankings that fulfill EOR fairness.

\vspace*{-1mm}
\subsection{Sources of Disparate Uncertainty} \label{sec:posterior_desc}
It remains to show that disparate uncertainty is a fundamental problem when estimating the relevance probabilities $\prob(r_i=1|\Data)$ that is not easily remedied by improved learning methods. The following illustrates that even a Bayes-optimal procedure is vulnerable to producing disparate uncertainty. 

\begin{figure}[t]
\centering
    \includegraphics[width=1.0\linewidth, trim = 0.0cm 0.0cm 0.0cm 0.0cm, clip]
    {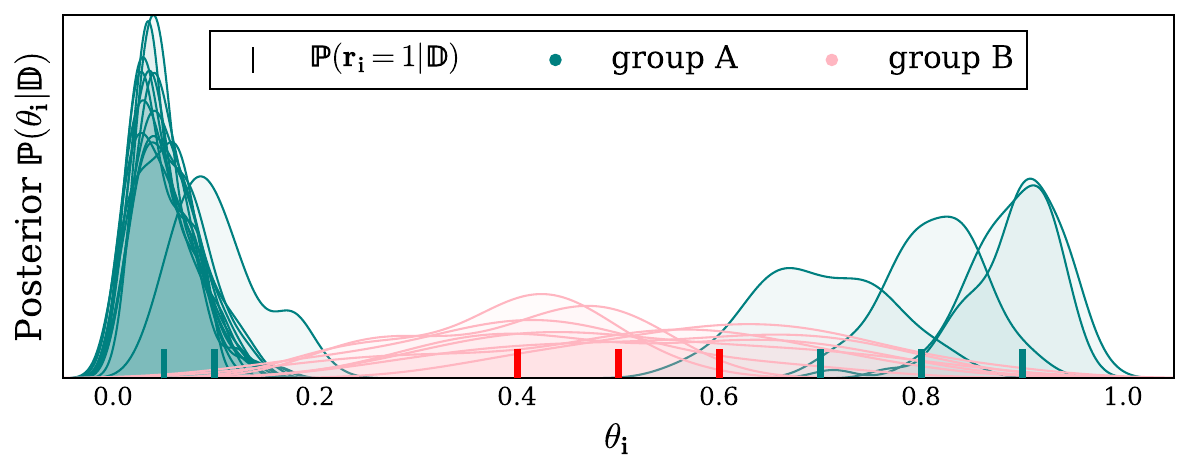}
    \vspace*{-5mm}
    \caption{\normalfont{An illustration of disparate uncertainty between groups from a Bayesian perspective for all the candidates of \Cref{fig:expected_probs}. The candidates in group A have peaky posteriors, while those in group B have relatively flat posteriors.}}
  \label{fig:posterior}
\end{figure}
Consider the posterior distribution illustrated in \Cref{fig:posterior}, which shows the uncertainty $\prob(\theta_i|\Data)$ that a Bayesian model has about the relevance probability $\theta_i$ of candidate $i$, where $\theta_i$ is the parameter of a Bernoulli distribution. For group A, the posterior $\prob(\theta_i|\Data)$ is peaked, meaning that the model can accurately pinpoint the correct relevance probabilities. For group B, the posterior is flat, which is to be expected if group B is smaller and thus has less data. The Bayes-optimal way of handling this uncertainty is to infer $\prob(r_i|\Data)$ via the posterior predictive distribution 
 \[\prob(r_i=1|\Data) = \int \prob(r_i=1|\theta_i) \:\prob(\theta_i|\Data) \: d{\theta_i} = \int \theta_i \: \prob(\theta_i|\Data) \: d{\theta_i} \]
\Cref{fig:posterior} shows how even this Bayes-optimal procedure leads to disparate uncertainty between groups,
where the $\prob(r_i=1|\Data)$ is closer to zero or one for candidates in group A (i.e., highly informative), and middling for group B (i.e., less informative).

Note that there is ample evidence that non-Bayesian methods also produce such disparities (e.g., \citep{pmlr-v81-buolamwini18a, tatman-2017-gender, wilson2019predictive}).  Furthermore, disparate amounts of data are not the only cause for disparity. For example, in college admissions, disparately more URM candidates may miss AP grades because their school does not offer AP classes. Their epistemic uncertainty \citep{DBLP:journals/ml/HullermeierW21} of qualification will thus be higher since the model has less information about these students. This higher uncertainty does not mean individual students are not qualified, and elevating them in the ranking for human evaluation can accurately reveal qualification through additional information (e.g., an interview, deep reading of the SOP, or recommendation letters). But if they are never selected for human review, then they do not have a chance for an admission spot.

\section{Equality of Opportunity in Ranking} \label{sec:cost_fairness}
In this section, we first discuss the assumptions and modeling choices and then formulate the cost that the uncertainty of the predictive model imposes on the principal and the relevant candidates from the different groups. 

Our first assumption includes access to group-wise calibration \cite{NIPS2017_b8b9c74a, doi:10.1177/0049124118782533} with the probability estimates calibrated within groups. To simplify notation, we do not differentiate between $\prob(r_i|\Data)$ and a group-wise calibrated score $\prob(r_i|s, A,\Data)=s$ and we only require this group-wise calibration as a sufficient condition for our framework.
Additionally, we assume that the true relevance $r_i$ is revealed perfectly to the human decision-maker upon review, and we do not model any bias in the human decision-making review process.
Finally, we assume that candidates have group membership to a single protected attribute and do not consider intersectional group membership, which is a practically important consideration in fairness. 
Relaxing these three assumptions for future work could allow modeling even more real-world complexities.

To formulate the cost of opportunity, we first recognize that any group-wise calibrated model
allows us to compute the \textbf{expected number of relevant candidates $\relCand(.)$ of a particular group $g$} -- no matter how well the model can differentiate relevant and non-relevant candidates in that group. 
\[\relCand(g)= \sum_{i \in g} \expect_{r_i \sim \prob(r_i|\Data)}[r_i] = \sum_{i \in g} \prob(r_i=1|\Data)\]
Extending this to rankings, the expected number of relevant candidates from group $g$ for any prefix $k$ of ranking $\sigma$ that only depends on $\prob(r_i=1|\Data)$ to ensure unconfoundedness is  
\[\relCand(g|\sigma_k) 
= \sum_{i \in g\cap \sigma_k} \expect_{}[r_i]
= \sum_{i \in g\cap \sigma_k}\prob(r_i=1|\Data) \]
Further extending this to a potentially stochastic ranking policy $\pi$ that represents a distribution over rankings for a particular query leads to 
\begin{eqnarray}   
    \relCand(g|\pi_k) &=& \sum_{i \in g} \expect_{r_i \sim \prob(r_i|\Data), \sigma_k \sim \pi}[r_i \indicator_{i \in \sigma_k}]  
    \nonumber \\
       &=&  \sum_{i \in g} \prob(i \in \sigmaKPi) \prob(r_i=1|\Data) \label{eq:expected_rel_policy}  
\end{eqnarray}
where $\prob(i \in \sigmaKPi)= \expect_{\sigma_k \sim \pi}[ \indicator_{i \in {\sigma_k}}]$ is the probability that policy $\pi$ ranks candidate $i$ into the top k.
As a side note notation-wise, for a specific policy, for example, $\EOR$, we denote the corresponding ranking $\sigma_k^{\EOR}$ in the abbreviated form as $\EORanking_k$.

The ability to compute these expected numbers of relevant candidates from each group allows us to reason about the cost resulting from the uncertainty of the model that each ranking imposes on the respective groups, which we detail in the following.

\subsection{Cost Burden to Candidate Groups and the Principal} \label{sec:cost_fairness_cand}
We define the \textbf{cost $c(.)$ to candidate $i$} as missing out on the opportunity to be selected if the candidate was truly relevant. For a ranking policy $\pi$ that produces rankings $\sigma \sim \pi$ based on $\prob(r_i|\Data)$, and a principal that reviews the top $k$ candidates, the cost to a relevant candidate $i$ is the probability of not being included in the top $k$.
\begin{equation}
    c(i| \pi_k,r_i) = r_i (1 - \prob(i \in \sigmaKPi))
    \label{eq:cost_candidate}
\end{equation}
Note that only relevant candidates can incur a cost, since non-relevant candidates will be rejected by human review and thus draw no utility independent of whether they are ranked into the top $k$. Also, note that $\prob(i \in \sigmaKPi)$ can be estimated by Monte-Carlo sampling even for complicated ranking policies that have no closed-form distribution.

While determining the cost to a specific individual $i$ is difficult since it involves knowledge of the true relevance $r_i$, getting a measure of the aggregate cost to the group is more tractable. In particular, we define the \textbf{group cost} as the expected cost to the relevant candidates in the group, normalized by the expected number of relevant candidates.
\begin{eqnarray}
    c(g| \pi_k) &=&  \frac{\sum_{i \in g} \expect_{r_i \sim \prob(r_i|\Data)}[c(i|\pi_k,r_i=1)]}{\relCand(g)} 
    \nonumber \\
    &=& \frac{\sum_{i \in g }(1 - \prob(i \in \sigmaKPi)) \prob(r_i=1|\Data)}{\relCand(g)}
    \nonumber \\
    &=& 1- \frac{\relCand(g| \pi_k)}{\relCand(g)}\label{eq:subgroup_cost}
\end{eqnarray}
The last equality in \eqref{eq:subgroup_cost} follows directly from Eq.~\eqref{eq:expected_rel_policy}.
We normalize the expected group cost with the total expected number of relevant candidates in the group so that the above approximates the fraction of relevant candidates from that group that miss out on the opportunity of being selected by the human reviewers.

The {\bf principal incurs a cost} whenever the ranking misses a relevant candidate, independent of group membership. For a principal that reviews the top $k$ applications from two groups -- A and B, the \textbf{total cost} can thus be quantified via the expected number of relevant candidates that are overlooked.
\begin{eqnarray}
    c(\text{Principal}|\pi_k) 
    &=& \frac{\sum_{i}(1 - \prob(i \in \sigmaKPi)) \prob(r_i=1|\Data)}{\relCand(A) + \relCand(B)} \label{eq:total_cost}
\end{eqnarray}
We again normalize this quantity to make it proportional to the total expected number of relevant candidates. Note that Eq.~\eqref{eq:total_cost} is related to the conventional metric of Recall@k.

\subsection{Equality of Opportunity Ranking (EOR) Criterion} \label{sec:fairness_criteria}
We now formally define our EOR fairness criterion
and argue that a disparity in uncertainty should not lead to disparate costs for any of the groups. We have already seen that $\PRP$ and $\DP$ can violate this goal. 
For a possible solution, we turn to the principle of random lottery that has been historically used to justify fair allocation of resources \cite{Goodwin1992-GOOJBL, Saunders2008-SAUTEO}.
Take, for example, the uniform ranking policy $\Uniform$, which ignores $\prob(r_i|\Data)$ and picks a ranking uniformly at random. Use of $\Uniform$ ensures that any relevant candidate has an equal chance of being evaluated and selected since any top $k$ of the ranking contains a uniform random sample of the {\em relevant} candidates -- independent of group membership. While the ranking effectiveness of $\Uniform$ is bad, it has the attractive property that the fraction of relevant candidates that get selected from each group is equal in expectation. For example, if both group A and group B contain 100 relevant candidates in expectation and if $\Uniform$ selects $l$ relevant candidate in expectation from group A, it also selects $l$ relevant candidates in expectation from group B. Similarly, if group A contains 200 relevant candidates and group B contains 100, the selection ratio will be 2 to 1 in expectation. 
We formalize this property of the uniform lottery as our key fairness axiom.
\begin{axiom}[EOR Fair Ranking Policy]\label{axm:axm_uniform}
For two groups of candidates A and B, a ranking policy $\pi$ is Equality-of-Opportunity fair, if for every $k$ the top-k subsets $\pi_k$ contain in expectation an equal fraction of the relevant candidates from each group. More precisely:
\begin{equation}
    \forall k \quad \frac{\relCand(A|\pi_k)}{\relCand(A)} = \frac{\relCand(B|\pi_k)}{\relCand(B)} \label{eq:uniform_ranking} 
\end{equation}
\end{axiom}

While this fairness property of $\Uniform$ is desirable, its completely uninformed rankings come at a cost to the principal and the relevant candidates from both groups, since only a few relevant candidates will be found. The uniform policy $\Uniform$ is particularly inefficient when the fraction of relevant candidates is small. The key question is thus whether we can define an alternate ranking policy that retains the group-wise fairness properties of $\Uniform$, but retains as much effectiveness in surfacing relevant candidates as possible. 

To illustrate that such rankings exist, which are both EOR fair and more effective, consider our motivating example of \Cref{fig:expected_probs}, where $\EORanking=[\overset{B}0.6 , \overset{A}0.9 , \overset{B}0.6,  \overset{A}0.9,  \overset{B}0.6,  \overset{B}0.5,  \overset{A}0.8,  \overset{B}0.5,  \overset{A}0.7,  \overset{B}0.4,  \overset{A}0.1,  \overset{B}0.4,  \overset{A}0.05, \overset{A}0.05,
  \overset{A}0.05, \\ \overset{A}0.05, \overset{A}0.05, \overset{A}0.05, \overset{A}0.05, \overset{A}0.05, \overset{B}0.4,  \overset{A}0.05, \overset{A}0.05, \overset{A}0.05, \overset{A}0.05]$ has the property that the proportion of relevant candidates for each group in the top $k$ never differs by more than $0.6/4$ for any value of $k$ in expectation.
In one way, this guarantee is even stronger than what is defined in Axiom~\ref{axm:axm_uniform}, since it holds for the specific ranking $\EORanking$ without the need for stochasticity in the ranking policy. This provides a non-amortized notion of fairness, which is particularly desirable for high-stakes ranking tasks that do not repeat, and we thus need to provide the strongest possible guarantees for the specific ranking $\sigma$ we present. However, a guarantee for an individual ranking makes the problem inherently discrete, which means that we require some tolerance (i.e., $0.6/4$ in the example above) in the fairness criterion depending on the choice of $k$. This leads to the following $\delta$-EOR Fairness criterion for an individual ranking $\sigma$.
\begin{Definition}[$\delta$-EOR Fair Ranking] \normalfont \label{def:eor_constraint}
For two groups of candidates A and B, a ranking $\sigma$ is $\delta$-EOR fair, if for every $k$ the top-k subset $\sigma_k$ differs in its fraction of expected relevant candidates from each group by no more than $\delta$. More precisely:
\vspace*{-1mm}
\begin{equation}
    \forall k \quad  \left|\frac{\relCand(A | \sigma_k) }{\relCand(A)}-\frac{\relCand(B | \sigma_k)}{\relCand(B)}\right|  \le \delta  \label{eq:EOR}
\end{equation}
\end{Definition}

Note that we can also define a specific ``slack'' $\delta(\sigma_k)$ for each position $k$. For a fair ranking $\sigma$, this slack should ideally oscillate close to zero as we increase $k$, and so minimizing its deviation from zero would translate to ensuring $\delta$-EOR fairness. Formally, we can define $\delta(\sigma_k)$ as 
\begin{equation}
    \forall k \quad  \delta (\sigma_k) = \frac{\sum_{i \in A\cap \sigma_k}\prob(r_i|\Data)}{\sum_{i \in A} \prob(r_i|\Data)}-\frac{\sum_{i \in B\cap \sigma_k}\prob(r_i|\Data)}{\sum_{i \in B} \prob(r_i|\Data)}  \label{eq:EOR_compute}
\end{equation}
$\delta$-EOR fairness balances the selection of candidates from the two groups, accounting for predictive uncertainty in their estimation of relevances. 
If for instance, the ML model is less certain in its predictions for group B, but both groups have the same total expected relevance, the $\delta$-EOR criterion will rank candidates from group B higher to ensure fairness. Importantly, note how this produces more human relevance labels of candidates from groups with high uncertainty, which has the desirable side-effect of producing new training data that allows training of more equitable relevance models for future use.

Finally, note how the $\delta$-EOR fair ranking provides a means for ensuring procedural fairness and avoiding \emph{disparate treatment}. Importantly, we leave the decision of which candidates to select to the human decision maker, and EOR fairness does not require the designation of a disadvantaged group. Instead, the EOR fair condition in Eq.~\eqref{eq:EOR} is symmetrical w.r.t.\ both groups and by definition treats both groups similarly, and its intervention in the ranking process is entirely driven by the predictive model $\prob(r_i|\Data)$. Even though it uses group membership, EOR-fairness is thus fundamentally different from demographic parity \citep{10.1145/2090236.2090255, DBLP:journals/corr/YangS16a} and affirmative action rules like Rooney rule \citep{brian_collins_rooney, 10.1145/3351095.3372858}, $\frac{4}{5}$th rule (selection rate for a protected group must be at least 80\% of the rate for the group with the highest rate)\footnote{Uniform Guidelines on Employment Selection Procedures, 29 C.F.R.\S 1607.4(D) (2015)} or $\gamma$-based notions of fairness \citep{10.1145/3391403.3399482} and threshold based formulations such as FA$^{*}$IR \citep{Zehlike_2017}.

To illustrate the difference with existing fairness notions, we return to our running example from \Cref{fig:expected_probs}. For top-4 ranking in \Cref{fig:example_ranking}, the EOR criterion can be computed as $|\delta(\EORanking_{4})| = 0.15$, $|\delta(\DPRanking_{4})| = 0.5$ and $|\delta(\PRPRanking_{4})| = 0.83$, quantifying the unfairness of DP and PRP as compared to EOR.
While DP selects one candidate from group B for every two candidates from group A, applying $\frac{4}{5}$th rule with group B as the disadvantaged group will select roughly $4/5$ number of candidates from group B for every two candidates from group A. For top-4 ranking, the $\frac{4}{5}$th rule is $\FourFifthRanking_{4}=[\overset{A}0.9, \overset{B}0.6, \overset{A}0.9, \overset{A}0.8]$ with $|\delta(\FourFifthRanking_{4})| = 0.5$. If instead, group A is selected as the disadvantaged group, $\frac{4}{5}$th rule will select all four candidates from group A resulting in $|\delta(\FourFifthRanking_{4})| = 0.83$, same as that of PRP. The FA$^{*}$IR criterion ($\FS$) is similarly anchored on the principle that a top-k ranking is fair when the proportion of disadvantaged candidates selected doesn't fall far below a required minimum proportion and also requires the designation of a disadvantaged group. In this example, $\FS$ gives the exact same top-4 ranking and EOR criterion as shown for $\frac{4}{5}$th rule.
In summary, the predominant fairness criteria in rankings motivated by the representation of size perform very differently than the $\EOR$. 
As an example, consider the well-documented issue of female candidates not being selected for leadership positions primarily due to their small applicant pool size \cite{underrepresentationWomen}. If the female applicants have high disparate uncertainty (due to lack of historical data), affirmative action may still select far fewer (based on group size) of them than deserved (based on the number of relevant female candidates).

We now briefly consider two other notions of fairness in rankings for the running example. First, we look at the exposure-based formulations\citep{10.1145/3219819.3220088, DBLP:journals/corr/abs-1805-01788}. 
The principle of exposure is motivated by position bias in rankings and ensures the allocation of position in rankings in proportion to the expected total relevance. While the position of a selected candidate is certainly important, it does not take disparate uncertainty into consideration. 
$\EXP$ is a stochastic policy that allocates equal exposure between the two groups (in this example, both groups have an equal expected total relevance) over the full 25 positions of the ranking. $\EXP$ allocates most of the probability mass to candidates in group B for all of the top-4 positions (not because they have high uncertainty but because their group size is smaller than group A). This results in a high cost burden for group A and the EOR criterion is computed as $|\delta(\EXPRanking_{4})| = 0.58$ higher than both $\EOR$ and $\DP$. Later in \Cref{sec:experiments}, we demonstrate how $\EXP$ places a higher cost burden on the uninformative group instead when both groups have relatively the same size.

Finally, we discuss the Thompson Sampling based fairness in rankings \citep{NEURIPS2021_63c3ddcc}. For $\TS$, binary relevances are drawn according to $r_i \sim \prob(r_i|\Data)$, and candidates are sorted in decreasing order of relevance $r_i$ with their ranking randomized for the same value of relevance.
The EOR criterion for a top-4 ranking produced by $\TS$ can be computed as $|\delta(\TSRanking_{4})| = 0.29$ for the running example. While $\TS$ takes the predictive uncertainty of relevance into account by randomization of rankings, it is group oblivious and so does not account for the difference in the predictive uncertainty of relevance between groups. This explains the high EOR criterion of a specific $\TSRanking$ with median $\sum_{k=1}^{n}|\delta(\TSRanking_k)|$ as compared to that of the $\EORanking$.
While we discussed how EOR differs from existing fairness notions above, we will further demonstrate this comparison via extensive empirical evaluations in \Cref{sec:experiments}.

One of our key contributions includes formalizing the connection between $\delta$-EOR Fair Ranking described in Definition~\ref{def:eor_constraint} and the cost of opportunity in rankings described in \Cref{sec:cost_fairness_cand}. Both $\delta$-EOR Fair Ranking and cost of opportunity in rankings are derived separately -- the former from the axiom of fairness of a uniform lottery, the latter from the cost of errors that any realistic prediction model is bound to make. In the next section, we show that these two are elegantly related via theoretical results on cost optimality.

\section{Computing EOR-Fair Rankings} \label{sec:compute_fair_ranking}
We now turn to the question of how to compute a $\delta$-EOR fair ranking $\EORanking$ for any given relevance model $\prob(r_i|\Data)$. This ranking procedure needs to account for two potentially opposing goals. First, it needs to ensure that $\delta$-EOR fairness is not violated, ideally for a $\delta$ that is not larger than required by the discreteness of the ranking. Second, it should maximize the number of relevant candidates contained in the top $k$, for any a-priori unknown $k$.
While solving this optimization problem in the exponentially sized space of rankings is computationally inefficient, we show that Algorithm~\ref{alg:EOR_alg} is an efficient ranking method that provides a close-to-optimal solution.

Algorithm~\ref{alg:EOR_alg} uses as input the PRP rankings $\PRPRankingGroup{A}$ and $\PRPRankingGroup{B}$ for each of the groups A and B respectively. We denote $\PRPRankingGroup{g}[i]$ as the $i^{th}$ element in the PRP ranking of group $g$. The basic idea is to compare the highest relevance candidate from each group and select the candidate that would minimize the $\delta$ for the resultant ranking (breaking ties arbitrarily when selecting an element from either group results in the same $\delta$ for the resultant ranking). 
Consider our running example from \Cref{fig:expected_probs}. At $k=1$, selecting the first element from group A, $\PRPRankingGroup{A}[1]$, would result in a $\delta(\sigma_1)=0.9/4$ while selecting the first element from group B, $\PRPRankingGroup{B}[1]$, would result in a $\delta(\sigma_1)=-0.6/4$. To minimize $\left| \delta(\sigma_1) \right|$, the algorithm selects the first element from group B with $\EORanking_1=[\overset{B}0.6], \left| \delta(\sigma_1) \right| =0.6/4$. For $k=2$, the first element from group A, and the second element from group B are considered. It proceeds to select the first element from group A with  $\EORanking_2=[\overset{B}0.6, \overset{A}0.9], \left| \delta(\sigma_2) \right| = 0.3/4$ and so on. The Algorithm does not change the relative ordering between candidates within a group and its runtime complexity is $\bigO(n\log{n})$, since the elements from the two groups each need to be sorted once by $\prob(r_i|\Data)$. Composing the final EOR ranking $\EORanking$ by merging the two group-based rankings $\PRPRankingGroup{A}$ and $\PRPRankingGroup{B}$ takes only linear time since each computation per iteration is constant time per prefix $k$.

\begin{algorithm}[tb]
\begin{flushleft}
\textbf{Input}: Groups $ g \in \{A, B\}$; Rankings $\PRPRankingGroup{g}$ per group in the sorted (decreasing) order of relevance probabilities $\prob(r_i|\Data)$.\\
\textbf{\textit{Initialize}}: $j \leftarrow 0$; empty ranking $\EORanking$
\end{flushleft}
    \While{$j<k$}{
        $l_g \leftarrow \PRPRankingGroup{g}[1] \hspace{10pt} \forall g \in \{A, B\}$\\
        $g^{*} \leftarrow \underset{g \in \{A, B\}}{\arg\min} \left| \delta(\EORanking \cup \{l_g\})\right|$,\\
        where $\delta(.)$ is computed using \eqref{eq:EOR_compute} \label{alg:step_min_delta} \\
        $l_{g^{*}} \leftarrow \PRPRankingGroup{g^{*}}[1] ; \quad \PRPRankingGroup{g^{*}} \leftarrow \PRPRankingGroup{g^{*}}$ \textbackslash $\{l_{g^{*}}\}$\\
        $\EORanking \leftarrow \EORanking \cup \{l_{g^{*}}\}; \quad j \leftarrow j+1$
    }
\textbf{Return} $\EORanking$
\caption{EOR Algorithm}\label{alg:EOR_alg}
\end{algorithm}
While Algorithm~\ref{alg:EOR_alg} is inspired by existing algorithms such as \citep{Zehlike_2017} in that both select the top element from the PRP ranking of each group, they are fundamentally different. Existing methods including \citep{Zehlike_2017} ensure a form of demographic parity which we have already shown to be fundamentally different than the EOR criterion we propose. Additionally, while \citep{Zehlike_2017} requires a threshold input and the designation of a disadvantaged group, the EOR Algorithm does not require this normative designation and guarantees EOR fairness without requiring any tolerance $\delta$ as an input. We show this both theoretically and in empirical evaluations and provide a detailed description of baseline algorithms in Appendix~\ref{baselines}.

It remains to be shown that Algorithm~\ref{alg:EOR_alg} always produces a ranking $\EORanking$ with small $\delta$ while surfacing as many relevant candidates as possible in any top $k$ prefix. We break the proof of this guarantee into the following steps. First, we show that for any particular $k$ and its associated $\delta(\EORanking_k)$, the number of relevant candidates in the top-$k$ is close to optimal. Second, we provide an upper bound on $\delta(\EORanking_k)$ that is entirely determined a priori by the specific $\prob(r_i|\Data)$.
To address the first step, the following Theorem~\ref{Thm:Thm1}, shows that the rankings produced by Algorithm~\ref{alg:EOR_alg} have a cost to the principal that is close to optimal. 
\begin{theorem}[Cost Approximation Guarantee at $k$]\normalfont \label{Thm:Thm1}
The EOR fair ranking $\EORanking$ produced by Algorithm~\ref{alg:EOR_alg} is at least $\phi\delta(\EORanking_k)$ cost optimal for any prefix $k$, where $\phi = \frac{2}{\relCand(A)+\relCand(B)}\left| \frac {p_{A}-p_{B} }{q_{A}+q_{B}}\right|$, $q_{A}=\frac{p_{A}}{\relCand(A)}$, and $q_{B}=\frac{p_{B}}{\relCand(B)}$. Further, $p_{A}=\PRPRankingGroup{A}[k_A]$, $p_{B} = \PRPRankingGroup{B}[k_B]$, where $k_A$ is the last element from group A that was selected by EOR Algorithm for prefix $k$ and similarly for $k_B$.
\end{theorem} 
Proof Sketch:
We use linear duality to prove this theorem. To find a lower bound on the cost optimal ranking that satisfies the EOR fairness constraint, we formulate the corresponding Linear Integer Problem (ILP) for selecting the optimal top-k subset under the $\delta$-EOR constraint.
This leads to the following optimization problem, where $X \in \{0,1\}^n$ is the variable for whether the $i^{th}$ candidate was chosen or not, $P$ is the relevance probability for all candidates.

Minimize total cost as defined in Eq.~\eqref{eq:total_cost}
\begin{align} \min_{x \in \{0,1\}} & \quad 1- \frac{P^TX}{\relCand(A)+\relCand(B)}   \tag{ILP} \\ \text{s.t.}\quad & X^T\basis = k \quad  \tag{select up to $k$ candidates} \\ -\delta(\EORanking_k) \leq & \left( \frac{P\indicator_{A}}{\relCand(A)}-\frac{P\indicator_{B}}{\relCand(B)}\right)^T X \leq \delta(\EORanking_k)  \tag{EOR fairness from Eq.~\eqref{eq:EOR} must be satisfied $\forall k$} \end{align} 
We relax this ILP to a Linear Program (LP) by turning any integer constraints $x \in \{0,1\}$ in the primal into $0 \leq x \leq 1$. For the relaxed LP, we formulate its dual and construct a set of dual variables $\lambda$ corresponding to the solution from the EOR Algorithm. Using the dual value of the EOR solution and the relaxed LP solution, we obtain an upper bound of the duality gap. Since the upper bound on this duality gap is w.r.t.\ the relaxed LP solution, it is also an upper bound for the optimal ILP solution.
We provide a complete proof of the theorem and associated lemmas in \Cref{proof:Thm1}. \qed

Note that $\phi$ depends only on the relevance probabilities of the last elements selected from each group by the EOR Algorithm in the $k^{th}$ position. Furthermore, note that the solution of Algorithm~\ref{alg:EOR_alg} is the exact optimum for any $k$ where the unfairness $\delta(\EORanking_k)$ is zero, indicating that any suboptimality of the EOR algorithm is merely due to some (presumably unavoidable) discretization effects. 

While the previous theorem characterized cost optimality, the following Theorem~\ref{thm:thm_bound} shows that the magnitude of unfairness $\delta(\EORanking_k)$ is bounded by some $\delta_{max}$, providing an a priori approximation guarantee for both the amount of unfairness and the cost optimality of Algorithm~\ref{alg:EOR_alg}. 
\begin{theorem}[Global Cost and Fairness Guarantee]\normalfont \label{thm:thm_bound}
     Algorithm~\ref{alg:EOR_alg} always produces a ranking $\EORanking$ that is at least $\phi\delta_{max}$ cost optimal for any $k$, with $\delta_{max}=\frac{1}{2}\left(\frac{\PRPRankingGroup{A}[1]}{\relCand(A)}+\frac{\PRPRankingGroup{B}[1]}{\relCand(B)}\right)$.
\end{theorem}
 
Proof Sketch:
We show via an inductive argument that according to the EOR algorithm, minimizing $\left| \delta(\EORanking_k) \right|$ at every $k$ ensures that the resultant EOR ranking always satisfies $\delta(\EORanking_k) \leq \frac{1}{2}\left(\frac{\PRPRankingGroup{A}[1]}{\relCand(A)}+\frac{\PRPRankingGroup{B}[1]}{\relCand(B)}\right)$, that is bounded by the average of the relevance proportions from the first two elements considered in the selection from group A and B. We denote this global fairness guarantee by $\delta_{max}$. Using $\phi$ from Theorem~\ref{Thm:Thm1} the cost guarantee is given by \\
$\phi \delta_{max} =\frac{1}{\relCand(A)+\relCand(B)} \left| \frac {p_{A}-p_{B} }{q_{A}+q_{B}}\right| \left(\frac{\PRPRankingGroup{A}[1]}{\relCand(A)}+\frac{\PRPRankingGroup{B}[1]}{\relCand(B)}\right).$
Further, we show that if the EOR algorithm selects all the elements from one group at some position $k$, then selecting the remaining elements from the other group satisfies the $\delta_{max}$ constraint.
We provide a complete proof of this theorem in \Cref{proof:Thm2}. \qed

We now compare EOR with the Uniform ranking policy and analyze positions $k$ with $\delta=0$ to avoid discretization effects. 
\begin{proposition}[Costs from EOR vs. Uniform Policy]
\normalfont \label{prop:uniform_eor}
    The EOR ranking never has higher costs to the groups and total cost to the principal as compared to the Uniform Policy, for those $k$ where $\delta(\sigma_k)=0$. 
\end{proposition}
We provide the proof of Proposition~\ref{prop:uniform_eor} in \Cref{proof:prop_uniform_eor}.
In summary, we have shown that Algorithm~\ref{alg:EOR_alg} is an efficient algorithm that computes rankings close to the optimal solution, making it a promising candidate for practical use.

\section{Extension to $G$ Groups} \label{sec:extension_multiple_groups}
In this section, we discuss the extension of the EOR algorithm beyond two groups. In particular, we consider the general case where a candidate belongs to one of G groups $ g \in [1 \cdots G]$. 
From \Cref{sec:cost_fairness}, we can generalize 
the cost burden to the principal similar to Eq.~\eqref{eq:total_cost}, taking all the groups into account for the normalization factor as follows
\begin{eqnarray}
    c(\text{Principal}|\pi_k) &=& \frac{\sum_{i}(1 - \prob(i \in \sigmaKPi)) \prob(r_i=1|\Data)}{\sum_{g=1}^{G} \relCand(g)} \label{eq:total_cost_multiple_groups}
    \nonumber
\end{eqnarray}
To generalize Algorithm~\ref{alg:EOR_alg} for selecting top $k$ candidates from multiple groups, we define $\delta(\sigma)$ as the EOR criterion that captures the gap between the group with the maximum accumulated relevance proportion and the group with the minimum accumulated relevance proportion,
\begin{eqnarray}
    \delta(\sigma) = \max_g \left\{\frac{\relCand(g|\sigma)}{\relCand(g)}\right\} - \min_g \left\{\frac{\relCand(g|\sigma)}{\relCand(g)}\right\} \label{eq:delta_multiple_groups}
\end{eqnarray}
The following selection rule then provides the selected group $g^{*}$ and candidate $l_{g^{*}}$ to append to the EOR ranking. 
\begin{eqnarray}
    l_g &=& \PRPRankingGroup{g}[1] \hspace{10pt} \forall g \in \{1 \cdot G\} \nonumber \\
    g^{*} &=& \argmin_{g \in [1..G]} \delta(\EORanking \cup \{l_g\})  ; \hspace{5pt} l_{g^{*}} = \PRPRankingGroup{g^{*}}[1] \label{eq:l_multiple_g}
\end{eqnarray}
Note that the above selection rule is a strict generalization of Algorithm~\ref{alg:EOR_alg} and it reflects the intuition of minimizing the gap in relevance proportions for all the groups.
It can be verified that the runtime complexity 
with selection rule according to Eqs.~\eqref{eq:delta_multiple_groups}, \eqref{eq:l_multiple_g} for a constant number of groups $G$ is $\bigO(n\log{n} + Gn)$.
Furthermore, we can extend the cost-approximation guarantee to the multi-group case.

\begin{theorem}[Global Cost and Fairness Guarantee for multiple groups] \normalfont \label{corr:corr_multiple_groups}
    The EOR rankings are cost optimal up to a gap of $\phi\delta(\EORanking_k)$ for $G$ groups, with $\delta(\EORanking_k)$ bounded by $\delta_{max}$, such that,
    \begin{eqnarray}
      \phi &=& \frac{2}{(G-1)\sum_{g=1}^G \relCand(g)}  \left(\sum_{\{A,B\}}^{} \left|  \frac{p_A-p_B}{q_A+q_B}\right| \right) ~~~\forall k \nonumber \\
        \delta_{max} &=& \max_{g} \left\{\frac{\PRPRankingGroup{g}[1]}{\relCand(g)}\right\} 
        \nonumber
    \end{eqnarray}
    where $\{A, B\}$ are all $G$ choose $2$ possible pairs of groups.
\end{theorem}

Proof Sketch: We extend the LP formed in Theorem~\ref{Thm:Thm1} to include $G(G-1)$ $\delta$ constraints and construct feasible dual variables from the EOR solution for each pair of groups. We then show that the duality gap is bounded by $\phi\delta(\EORanking_k)$ for a particular prefix $k$. 
Note that the $\phi$ bound for multi-group reduces to the one presented in Theorem~\ref{Thm:Thm1} for two groups. Finally, 
we present the global a priori bound on $\delta(\EORanking_k)$ as $\delta_{max}$, which is a strict generalization of the two groups case. We provide complete proof of this theorem in \Cref{proof:thm_6_1}. \qed

\section{Experimental Evaluation} \label{sec:experiments}
We now evaluate the EOR framework and algorithm empirically and compare against several baselines -- namely Demographic (Statistical) Parity ($\DP$) \cite{DBLP:journals/corr/YangS16a}, FA$^*$IR Ranking Principle ($\FS$) \cite{Zehlike_2017}, Probability Ranking Principle ($\PRP$) \cite{article_pr}, Thompson Sampling Policy ($\TS$) \cite{NEURIPS2021_63c3ddcc}, Uniform Policy ($\Uniform$), Disparate Treatment of Exposure ($\EXP$) \citep{10.1145/3219819.3220088}, and Fair Rank Aggregation ($\RA$) \citep{10.1145/3593013.3594085} with proportional representation of exposure. We discuss implementation details of these baselines in Appendix~\ref{baselines}. 

\subsection{Synthetic Data} 

We first present results on synthetic data where we can control the level of disparate uncertainty.
We report a) unfairness and b) effectiveness of rankings for each scenario. The unfairness metric is defined as the area under the curve for the EOR criterion, given by $\sum_{k=1}^{n}|\delta(\sigma_k)|$. To measure the effectiveness of rankings, we report the improvement in total cost over the expected total cost of $\Uniform$, computed as $\sum_{k=1}^{n} c(\text{Prinicpal}|\Uniform_k) - c(\text{Prinicpal}|\pi^{(.)}_k)$.

\begin{table*}[t]
\begin{imageonly}    \begin{minipage}[ht]{0.55\linewidth}
    \vspace{10pt}
    \centering
    \centering
    \begin{tabular}{ccccccc} 
    \toprule 
     & \multicolumn{3}{c}{Un-fairness $\hspace{1pt} \downarrow$}  
                & \multicolumn{3}{c}{Effectiveness $\uparrow$} \\
    \cmidrule(r){2-4}\cmidrule(r){5-7}
    $\pi$ \textbackslash Disp. Unc. & High & Medium & Low & High & Medium & Low\\
    \cmidrule(r){1-1}\cmidrule(r){2-4}\cmidrule(r){5-7} 
    $\EOR$ & 1.07 {\small $\pm$0.01} & 1.02 {\small $\pm$0.00} & 1.02 {\small $\pm$0.00} & 10.44 {\small $\pm$0.15}  & 11.89 {\small $\pm$0.04} &14.58 {\small $\pm$0.10}  \\
    $\DP$ & 11.09 {\small $\pm$0.38} & 6.02 {\small $\pm$0.07} & 2.42 {\small $\pm$0.20} & 10.07 {\small $\pm$0.20}  & 11.33 {\small $\pm$0.04} &14.49 {\small $\pm$0.11}  \\
    $\PRP$ & 15.41 {\small $\pm$0.69} & 7.68 {\small $\pm$0.13} & 2.63 {\small $\pm$0.17} & 12.11 {\small $\pm$0.20}  & 12.00 {\small $\pm$0.02} &14.62 {\small $\pm$0.09}  \\
    $\TS$ & 11.77 {\small $\pm$0.57} & 4.96 {\small $\pm$0.07} & 4.49 {\small $\pm$0.45} & 7.66 {\small $\pm$0.04}  & 9.62 {\small $\pm$0.06} &12.81 {\small $\pm$0.69}  \\
    $\Uniform$ & 5.96 {\small $\pm$0.13} & 5.80 {\small $\pm$0.00} & 6.49 {\small $\pm$0.09} & 0.00 {\small $\pm$0.00}  & 0.00 {\small $\pm$0.00} &0.00 {\small $\pm$0.00}  \\
    $\EXP$ & 9.23 {\small $\pm$0.77} & 5.62 {\small $\pm$0.01} & 3.26 {\small $\pm$0.62} & 11.59 {\small $\pm$0.23}  & 11.97 {\small $\pm$0.03} &14.62 {\small $\pm$0.09}  \\
    $\RA$ & 13.97 {\small $\pm$0.71} & 6.57 {\small $\pm$0.16} & 2.40 {\small $\pm$0.00} & 12.02 {\small $\pm$0.19}  & 12.00 {\small $\pm$0.02} &14.60 {\small $\pm$0.00}  \\
    $\FS$ & 13.33 {\small $\pm$0.70} & 7.04 {\small $\pm$0.16} & 2.95 {\small $\pm$0.17} & 11.98 {\small $\pm$0.20}  & 12.00 {\small $\pm$0.02} &14.62 {\small $\pm$0.09}  \\
    \bottomrule
    \end{tabular} 
    \end{minipage}
    \hfill
    \begin{minipage}[ht]{0.38\linewidth}
    \raggedleft
      \hspace*{12pt}\includegraphics[width=0.85\columnwidth, trim = 0.15cm 0.2cm 0.2cm 0.2cm, clip] {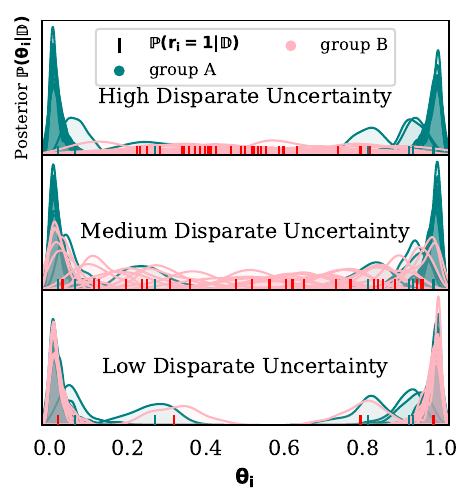}\hspace*{-12pt} \vspace{-35pt} 
    \end{minipage}\end{imageonly}
    \vspace{5pt}
    \caption{\rightskip170pt Left: \normalfont{ Effect of varying disparate uncertainty on Synthetic Dataset, \textbf{Right:} Posterior distribution and expected probabilities of relevance shown for a sample from each of high, medium, and low uncertainty setting.}}
    \label{tab:vary_disp_unc}
\end{table*}

\subsubsection{How does $\EOR$ compare against the baselines under varying amounts of disparate uncertainty?}
\Cref{tab:vary_disp_unc} (left) reports unfairness and effectiveness for $\EOR$ and the baselines in terms of mean and standard error over $100$ simulations, while \Cref{tab:vary_disp_unc} (right) demonstrates the posterior distribution formed by sampling an instance of each of high, medium and low disparate uncertainty settings. These posterior distributions similar to \Cref{fig:posterior} are for illustrative purposes since only the expected probability of relevance $p_i$ is used for rankings (refer to \Cref{sec:posterior_desc}). 
The different disparate uncertainty settings are generated synthetically to demonstrate how ranking policies behave if, for example, the Principal collects more data for group B thus reducing the disparate uncertainty among groups. Note, how in the low disparate uncertainty setting, the sharp $p_i$ (close to 0 or 1), would make the identification of relevant candidates easy for both groups. The synthetic generation involves sampling $p_i$ from sharp and flat distributions for group A and B respectively and gradually increasing the sharpness of $p_i$ for group B (implementation details in \Cref{synthetic_extension}).

As predicted by theory, $\EOR$ maintains low unfairness at all levels of disparate uncertainty, outperforming all the baselines $\PRP$, $\DP$, $\TS$, $\EXP$, $\RA$, and $\FS$. Note that $\EOR$ even outperforms the uniform policy $\Uniform$, since any individual ranking drawn from $\Uniform$ is likely to be unfair. In terms of effectiveness, the theoretically optimal skyline is given by $\PRP$. Across all levels of disparate uncertainty, $\EOR$ is at least competitive with the other baselines, indicating that the EOR fairness does not impose a disproportionate cost of fairness for the Principal.

Note how the gap in the unfairness between $\EOR$ and all other ranking policies is largest when disparate uncertainty is highest. At low levels of disparate uncertainty, $\EOR$ is still more fair as compared to other ranking policies (though the gap in unfairness is smaller) and the effectiveness of $\EOR$ is almost the same as that of $\PRP$.

\subsubsection{At which positions in the rankings do the policies incur unfairness?}
While the previous table summarized unfairness across the whole ranking, \Cref{fig:synthetic_main_disp_unc} (left) provides more detailed insights into how unfairness accumulates across positions in the ranking. The only method that is systematically fair across all positions $k$ is $\EOR$, keeping the unfairness $\delta(\sigma_k)$ from \Cref{def:eor_constraint} close to zero everywhere in the ranking. The baselines generally start accumulating unfairness towards one group right from the top of the ranking. Their unfairness only decreases once they run out of viable candidates from the group they prefer. The only exception is $\Uniform$, here for a specific ranking with median $\sum_{k=1}^{n}|\delta(\UniformRanking_k)|$. However, rankings from $\Uniform$ tend to stray much further from zero than the $\EOR$ ranking. Additional results for the medium and low disparate uncertainty settings in \Cref{fig:high_medium_low_disp_unc} of Appendix~\ref{synthetic_extension} further support these findings.

\begin{figure*}[t!]
\centering
    \begin{subfigure}[t]{1.0\linewidth}
    \centering
    \includegraphics[width=1.0\textwidth]
    {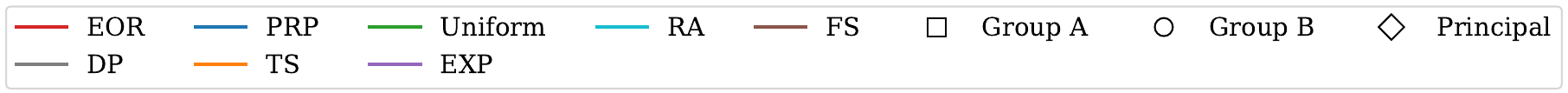}
    \end{subfigure}
    \begin{subfigure}[t]{1.0\linewidth}
        \centering
        \includegraphics[width=1.0\textwidth, trim = 0.0cm 0.2cm 0.0cm 0.2cm, clip]{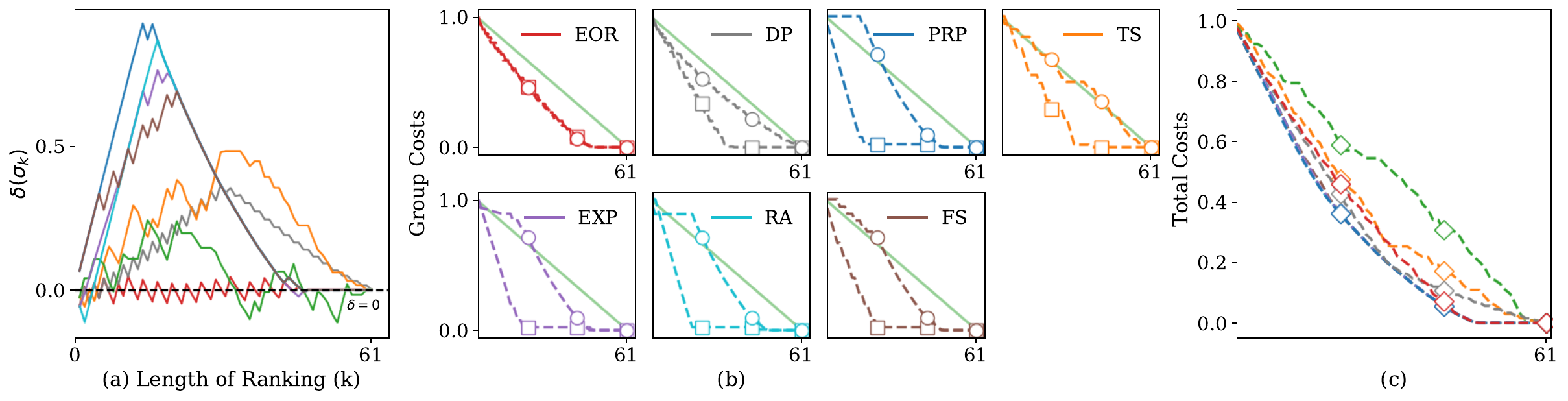}
    \end{subfigure}
  \vspace*{-6mm}
  \caption{ Left: \normalfont{EOR criterion $\delta(\sigma_k)$, \textbf{Middle}: group costs according to \eqref{eq:subgroup_cost}, \textbf{Right}: the principal's total cost according to \eqref{eq:total_cost} of the ranking policies for the synthetic dataset with high disparate uncertainty shown in top right of \Cref{tab:vary_disp_unc}. Group A consists of 30 candidates with sharp probabilities with $p_i \sim \text{Beta}(1/20, 1/20)$. This provides $\relCand(A)=14.96$ expected number of relevant candidates. Group B also has similar candidates, in particular, it has 31 candidates, with relatively flat probabilities $p_i \sim \text{Beta}(5,5)$, providing $\relCand(B)=14.94$ expected number of relevant candidates.}}
\label{fig:synthetic_main_disp_unc}
\vspace{-3mm}
\end{figure*}

\subsubsection{How do the ranking policies distribute the costs between the stakeholders?}
In \Cref{fig:synthetic_main_disp_unc} (middle) we investigate how the ranking policies distribute the cost $c(g|\pi_k)$ from Eq.~\eqref{eq:subgroup_cost} between group A and group B.
It shows that only $\EOR$ has an equal cost to both groups across the whole ranking, which can be seen from the overlapping cost curves for both groups. Furthermore, the cost is substantially lower for both groups than their expected cost under the uniform policy (diagonal line).\Cref{fig:synthetic_main_disp_unc} (right) shows the total cost to the principal, and again $\EOR$ is competitive with the baselines. 

All other baselines incur substantial disparate costs to the groups, some even worse than the uniform lottery. 
In particular, $\DP$ selects the candidates alternately between the two groups since group sizes are relatively similar, but this results in selecting a higher proportion of relevance from group A because the relevance probabilities are sharper for group A than for B. As a result, the cost burden is higher for group B. 
$\TS$ is fairer than $\PRP$, since it randomizes relevant candidates before sorting them in decreasing order of relevance, however being group oblivious, it still places an uneven cost burden. 

The exposure based policies $\EXP$, $\RA$ motivated by position bias in rankings also do not distribute the costs evenly.
$\EXP$ will stochastically allocate most of the top positions to candidates with sharp and high probabilities, close to 1.0 from group A, then to candidates of group B with flat and middle relevance probabilities, and finally the rest of the candidates from group A with sharp but low probabilities, close to 0.0 in the last positions. While this perfectly allocates exposure between group A and B over the full ranking of $61$ candidates, group B (the uninformative group) suffers from a high cost burden. Note how the direction of cost burden is opposite to the one $\EXP$ induced in the example of \Cref{fig:expected_probs}, where group B was smaller in size to group A.

\begin{figure*}[t!]
\begin{subfigure}[t]{0.8\linewidth}
    \includegraphics[width=0.9\textwidth]
    {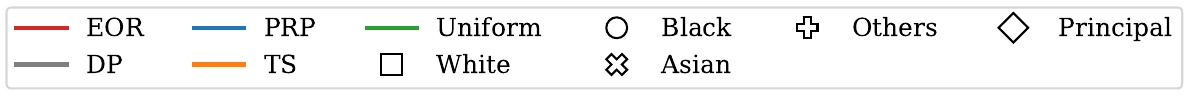}
    \label{fig:policy_legend_USCensus}
  \end{subfigure}
  \vspace*{-7mm}
  \begin{subfigure}[t]{1.0\linewidth}
        \centering
        \includegraphics[width=0.9\textwidth, trim = 0.2cm 0.2cm 0.2cm 0.2cm, clip]{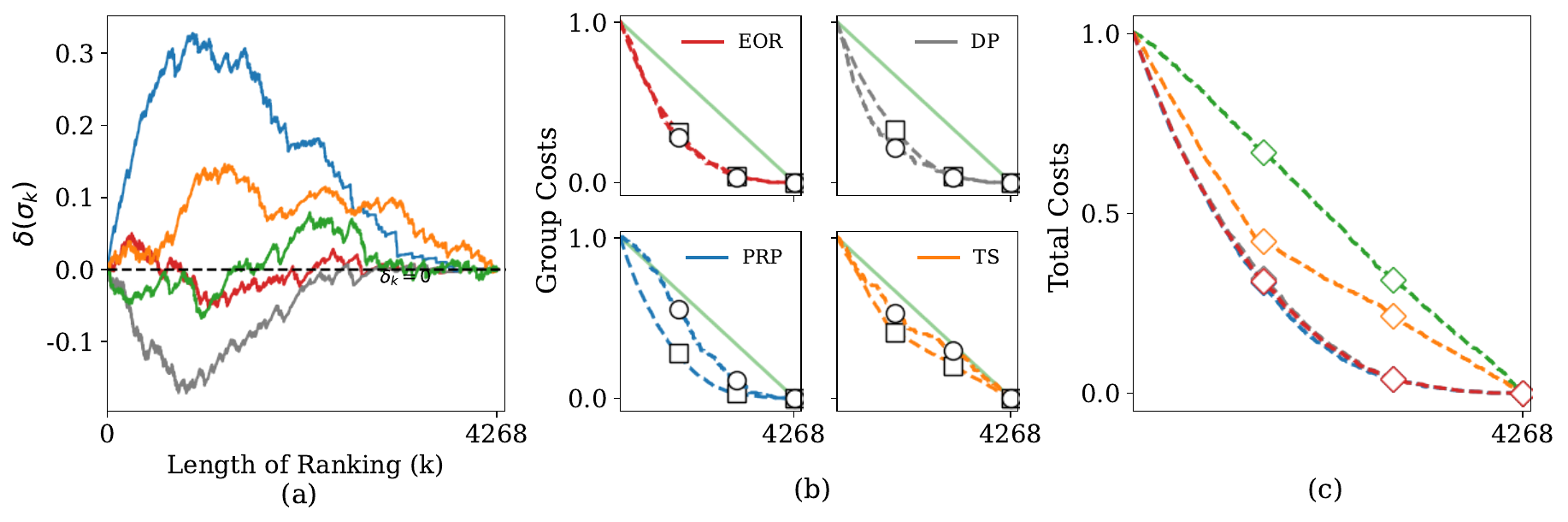}
        \phantomcaption{}
        \label{fig:USCensus_labels_AL_2_groups}
    \vspace{-2mm}
  \end{subfigure}
  \vspace{-2mm}
  \begin{subfigure}[t]{1.0\linewidth}
    \centering
    \includegraphics[width=0.905\textwidth, trim = 0.2cm 0.0cm 0.2cm 0.2cm, clip]{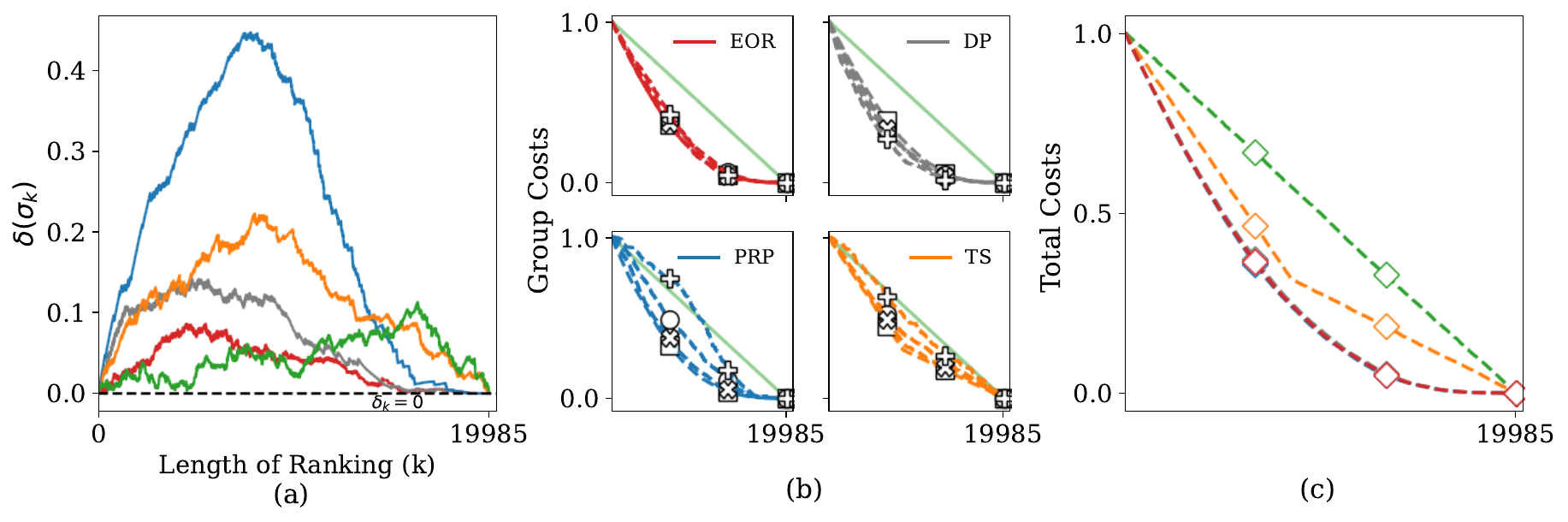}
    \phantomcaption{}
    \label{fig:USCensus_labels_4_groups}
  \end{subfigure}
  \vspace{-6mm}
  \caption{\normalfont{\textbf{US Census Dataset}: EOR criterion $\delta(\sigma_k)$ and cost of the ranking policies computed with true relevance labels from the test subset for the US Census dataset. 
  \textbf{Top: Two groups setting} using the White and Black/African American racial groups for the state of Alabama.  \textbf{Bottom: Multiple (four groups) setting} using White, Black/African American, Asian, and Other for the state of NY.}}
  \label{fig:USCensus_figs}
\vspace{-2mm}
\end{figure*}

\subsection{US Census Survey Data} \label{sec:UScensus_expts}
While the synthetic experiments provide insights into the behavior of ranking policies under varying conditions, we now investigate how far $\EOR$ can mitigate unfairness as it arises in real-world datasets where the relevance probabilities $\prob(r_i|\Data)$ are learned from data.
In particular, we consider the US Census Survey dataset \cite{NEURIPS2021_32e54441} for the year 2018 and the state of Alabama and New York, consisting of 22,268 and 103,021 records respectively. The task is to predict whether the income for an individual $>\$50K$ based on features such as educational attainment, occupation, class of worker etc. We use this task as a stand-in for some task where individuals receive a benefit from being evaluated positively. To get group-calibrated estimates of $\prob(r_i|\Data)$, we train a gradient boosting classifier followed by Platt Scaling on the validation subset of the data.
We evaluate the EOR criterion and costs on the test subset of these records. Full details for dataset pre-processing and training can be found in \Cref{USCensus_details}. Because these rankings are large (up to $\sim 20K$ size), $\EXP$ and $\FS$ are not computationally tractable. $\RA$ performs similarly to $\PRP$ and we include it in Appendix~\ref{USCensus_details} for completeness.
 
\subsubsection{How do the ranking policies compare when using learned probability estimates?}
To evaluate the two-group EOR algorithm, we first only rank individuals labeled as White and Black or African American. Figure~\ref{fig:USCensus_figs} (top) shows that EOR ranking is effective even with estimated probabilities. In particular, while the ranking algorithms only use estimated probabilities, the EOR criterion, and costs are evaluated on the true relevance labels from the test set. 
Nevertheless, $\EOR$ still evaluates $\delta$ close to zero and distributes costs among the stakeholders more evenly than the other baseline policies $\PRP$, $\DP$, and even $\Uniform, \TS$ for a specific ranking with median $\sum_{k=1}^{n}|\delta(\sigma_k)|$.    
Additional experiments in Appendix~\ref{USCensus_details} further confirm these findings.

\subsubsection{How does EOR Ranking perform for more than two groups?}
Figure~\ref{fig:USCensus_figs} (bottom) shows results on the US Census Dataset for four groups, again using estimated relevances for ranking but evaluating against the true relevance labels from the test dataset.
Note that for more than two groups, the EOR constraint defined according to \eqref{eq:delta_multiple_groups} will always be non-negative as it measures the absolute difference in relevance proportions between the groups that are furthest apart. We observe that similar to the results with two groups, the EOR ranking keeps the unfairness $\delta$ lower (close to zero) as compared to other policies in Figure~\ref{fig:USCensus_figs} (left). Additionally, $\EOR$ also distributes the costs evenly among all stakeholders for the generalized case of more than two groups, as noted by the overlapping of dashed lines for the four group costs (middle). Finally, $\EOR$ is competitive with the optimal $\PRP$ in terms of total cost for the principal.

\subsection{Amazon Shopping Audit} \label{sec:amazon_expts}  
In the final experiment, we investigate how the EOR framework can be used for auditing. To illustrate this point, we use a dataset of Amazon shopping queries \citep{reddy2022shopping}, 
which includes a baseline model for predicting the relevance of products given a search query. 
We further augment this dataset with logged rankings from the Amazon website as collected for the Markup report \citep{Yin2021}, which investigated Amazon's placement of its own brand products as compared to other brands based on star ratings, reviews etc. The Markup data consists of popular search query-product pairs along with logged rankings of these products on Amazon's platform, but it does not contain human-annotated relevance labels. 
We focus the audit on bias between the group of Amazon-owned brands (group A) or any other brand (group B). As the first step of the audit, we calibrate $p_i$ by fitting a Platt-scaling calibrator using validation data for both groups. 
Figure~\ref{fig:calibration_amazon} shows that the calibrated $p_i$ on the test dataset binned across 20 equal-sized bins, lies close to the perfectly calibrated line.
\begin{figure}[t!]
\centering
    \begin{subfigure}[b]{0.49\linewidth}
    \includegraphics[width=1.0\textwidth]{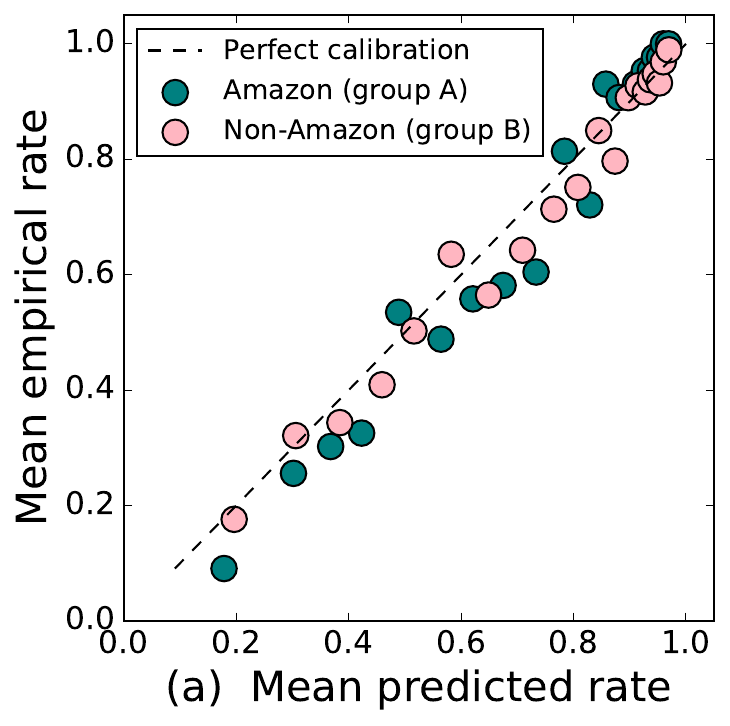}
    \phantomcaption{}
    \label{fig:calibration_amazon}
  \end{subfigure}
    \begin{subfigure}[b]{0.49\linewidth}
    \includegraphics[width=1.0\textwidth]{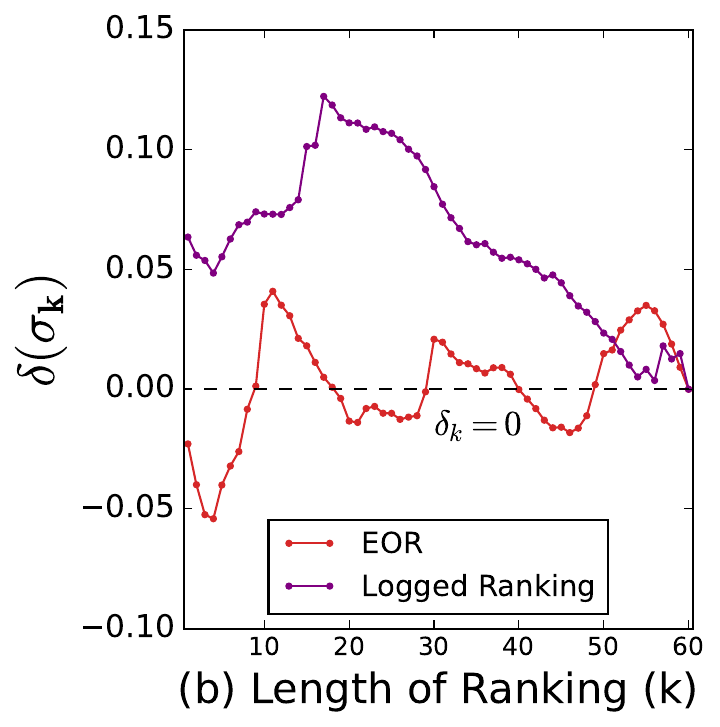}
    \phantomcaption{}
    \label{fig:markup_eo}
    \end{subfigure}
  \vspace*{-2mm}
    \caption[]{\textbf{Left: } \normalfont{Group-wise calibration of $\prob(r_i|\Data)$ for Amazon shopping queries on the test set according to the baseline model after Platt Scaling. 
    } \textbf{Right: } Fairness of logged Amazon rankings compared to EOR rankings in terms of $\delta(\sigma_k)$ averaged over queries. 
    }
  \label{fig:Amazon_expt}
  \vspace{-2mm}
  \end{figure}
As the second step of the audit, we use the Markup dataset with logged rankings\footnote{\href{https://github.com/the-markup/investigation-amazon-brands/tree/master}{https://github.com/the-markup/investigation-amazon-brands}} and compute $p_i$ using the calibrated baseline relevance prediction model. The EOR criterion \eqref{eq:EOR_compute} is averaged over queries for the logged rankings, and the EOR rankings are produced by Algorithm~\ref{alg:EOR_alg}. 
Figure~\ref{fig:markup_eo} shows that there exists a ranking $\EORanking$ that has $\delta(\EORanking_k)$ closer to zero for most prefix $k$.
The logged rankings from Amazon's platform show estimated $\delta(\sigma_k)$ that are farther away from zero for at least some prefixes of $k$, reflecting a potential favoring of Amazon brand products. A \nobreak limitation of this analysis is that unlike in a real audit where the auditor has access to the production model of $p_i$, our baseline model may be subject to hidden confounding, and thus does not provide conclusive evidence of unfairness. In particular, the production rankings may depend on other features beyond product titles (e.g, product descriptions, bullet points, star ratings, etc.). However, the analysis does demonstrate how the EOR criterion can be used for auditing, if the auditor is given access to the production ranking model to avoid confounding.
We provide further details in Figure~\ref{amazon_details} and our source code with experiment implementation can be found here.\footnote{\href{https://github.com/RichRast/DisparateUncertainty}{https://github.com/RichRast/DisparateUncertainty}} 
\section{Conclusion}
This paper studies the problem of disparate uncertainty across groups as a source of unfairness in ranking when these rankings are used as part of a human decision-making process. In particular, this paper introduces a framework that formalizes this unfairness by relating it both to a fair lottery and to the costs that an imperfect model imposes on the various stakeholders. Recognizing that it may be difficult to avoid disparate uncertainty in real-world models, the paper develops the EOR procedure to produce rankings that provably mitigate the effects of disparate uncertainty between groups. Beyond its strong theoretical guarantees, we find that the EOR method outperforms existing methods for fair ranking across a wide range of settings. Furthermore, we illustrate that the EOR criterion can also be used as a tool to audit a real-world system. We conjecture that this combination of theoretical grounding, computational efficiency, and strong empirical performance provides viable conditions for making the proposed framework and algorithm accessible for thoughtful use in practice.

\section{Ethical Considerations}
This work explicitly addresses the potentially negative societal impact of machine learning predictions that include disparities between groups in the context of ranking interfaces. However, as pointed out by previous research \citep{10.1145/3287560.3287598, Mkander2021EthicsBasedAO}, we do not prescribe distilling down the fairness of a system into a single metric -- the fairness criterion we propose. We emphasize that it is important to carefully consider the domain specifics and the particular situation where our method may be deployed. 

We also note that while our EOR algorithm does not worsen the fairness within each group (i.e., within group ordering is maintained), it doesn't improve within-group fairness either. Exploring this dichotomy of satisfying within and between group fairness simultaneously in the presence of differential uncertainty is an important open question.

\begin{acks}
This research was supported in part by NSF Awards IIS-2008139 and IIS-2312865. All content represents the opinion of the authors, which is not necessarily shared or endorsed by their respective employers and/or sponsors. We thank Kate Donahue, Marios Papachristou, Aaron Tucker, Sarah Dean, Luke Wang, Emily Ryu, Ashudeep Singh, Taran Pal Singh, and Woojeong Kim for helpful comments and discussions. We also thank the anonymous reviewers at the Epistemic AI, UAI workshop for helpful feedback.
\end{acks}

\bibliographystyle{ACM-Reference-Format}
\bibliography{sample-base}

\clearpage
\newpage

\appendix
\onecolumn
\section{Notation Summary}

\begin{align*}
    & n && \text{number of candidates}\\
    & i \in \{1, \cdots, n\} && \text{candidate}\\
    & G && \text{number of groups}\\
    & g \in \{1,2, \cdots G \}&& \text{group}\\
    & {k} && \text{ranking prefix}\\
    & S(g) && \text{size of group $g$}\\
    & \relCand(g) \in \real && \text{expected number of relevant candidates for group $g$}\\
    & \cline{1-2} \\
    & r_i \in \{0,1\} && \text{binary relevance of candidate $i$}\\
    & \theta_i \in [0,1] && \text{probability of relevance of candidate $i$}\\
    & \Data  && \text{historical data}\\
    & \prob(\theta_i|\Data)  && \text{posterior distribution}\\
    &p_i = \prob(r_i|\Data) \in [0,1] && \text{expected probability of relevance of candidate $i$}\\
    &P=(p_i)_{i \in \{1, \cdots, n\}} && \text{relevance probability vector}\\
    &X && \text{vector indicating whether candidate $i$ was selected}\\
    &\indicator_{g} \in \{0,1\}^{n} && \text{indicator if candidate $i$ belongs to group $g$}\\
    & \cline{1-2} \\
    &\pi && \text{policy}\\
    &\sigmaKPi && \text{top $k$ ranking }\sigma_k \sim \pi\\
    &\PRPRankingGroup{g}[i]&& i^{th} \text{ candidate in the PRP ranking of group g}\\
    & \delta(\sigma) && \text{EOR measure for ranking $\sigma$}
\end{align*}

\section{Extended Related Work} \label{sec:extended_related}
Our work complements and extends prior research on fairness in rankings \cite{DBLP:journals/corr/abs-2103-14000}. The classical fairness desiderata considered are variations of \emph{proportional representation} \citep{10.1145/2090236.2090255, DBLP:journals/corr/YangS16a}. Broadly, proportional representation ensures representation 
by group size in top $k$ selection or at every prefix $k$ of the ranking.
Other popular notions include diversity based constraints \citep{Celis2017RankingWF, doi:10.1089/big.2016.0054, inproceedings} like Rooney Rule and affirmative action that ensure representation of the designated disadvantaged group, and threshold based formulations \citep{Zehlike_2017, 10.1016/j.ipm.2021.102707, Wang2022ImprovingSP} that ensure a minimum number of candidates to be selected from the disadvantaged group. 

Another prominent class of fairness notions in rankings corresponds to \emph{exposure} based formulations. Exposure \citep{10.1145/3219819.3220088, 10.1145/3308560.3317595, 10.1145/3366424.3380048} quantifies the amount of attention allocated to candidates individually or from a particular group. These formulations include equity of exposure, disparate treatment of exposure that allocates exposure proportional to amortized relevance, and disparate impact of exposure that allocates exposure proportional to impact (e.g. economic impact of ranking) among other variations. See \citep{DBLP:journals/corr/abs-1805-01788} for a similar concept of equity of attention.
\emph{Proportional representation}, \emph{diversity constraints}, and \emph{exposure} are motivated by representation by group size, normative designation of disadvantaged group, and allocation of attention respectively. Our work, on the other hand, is motivated by unfairness due to differential uncertainty between groups and is grounded in the axiomatic fairness of a lottery system. 

Our problem setup involves aggregating candidates from groups and
while research on fair rank aggregation appears related, the goal there is much different. In particular, fair rank aggregation achieves maximum consensus accuracy when multiple voters rank all candidates subject to fairness constraints of group exposure \citep{10.1145/3593013.3594085} or p-fairness \citep{10.1145/3514221.3517865}. Work on multi sided fairness \citep{10.1145/3471158.3472260, DBLP:journals/corr/Burke17aa} similarly considers diversity constraints or exposure-based formulations.
Finally, while \citep{DBLP:journals/corr/abs-1805-01788, NEURIPS2021_63c3ddcc, 10.1145/3534678.3539353, DBLP:journals/corr/abs-2102-05996} propose an amortized notion of fairness, our work proposes a non-amortized fairness criterion at every position $k$ of the ranking.

Recently, there has been a growing interest in the study of fairness in rankings under uncertainty.
The classical desideratum in this literature studies the relation of group-wise calibration for fairness \citep{NIPS2017_b8b9c74a, kleinberg_et_al:LIPIcs.ITCS.2017.43, Flores2016FalsePF, DBLP:journals/bigdata/Chouldechova17, 10.1145/3531146.3533245}. Our work is orthogonal to this discussion. In particular, we only assume that calibrated probability of relevance is given and instead focus on how differential sharpness of probabilities can cause unfairness.
\cite{NEURIPS2021_63c3ddcc} introduced an approximate notion of fairness that is violated if the principal ranks candidates that appear more than a certain proportion of their estimated relevance distribution. One way to achieve this in expectation is through randomization of relevances drawn from the predictive posterior distribution. 
Other works have introduced methods that quantify uncertainty in rankings \cite{10.1145/3539597.3570469} to update and learn better estimates of relevances iteratively \cite{10.1145/3539597.3570474}. These works do not consider the unfairness caused due to differential uncertainty between groups. While methods that reduce uncertainty for all groups are needed, we also need to account for unfairness due to the existing disparate uncertainty that is unfortunately widespread in practical settings.

Another line of research focuses on statistical discrimination and the study of noisy estimates of relevances for selection problems \citep{10.2307/1806107, RePEc:pri:indrel:30a}.
This literature establishes that the differential accuracy of models causes unfairness \citep{Hashimoto2018FairnessWD, article, wilson2019predictive, pmlr-v81-buolamwini18a, tatman-2017-gender} for individuals based on their group membership. 
Recently, \citep{EMELIANOV2022103609, 10.1145/3442188.3445889} studied the role of affirmative action in the presence of differential variance between groups in rankings. Their method \citep{EMELIANOV2022103609} corrects the bias in noisy relevance estimates given the variance of the true relevance distribution.
Fairness in selection processes has also been extensively studied in the presence of group-based implicit bias \citep{kleinberg2018selection, 10.1145/3351095.3372858, 10.1145/3391403.3399482, 10.1145/3442188.3445930}, uncertainty in preferences \citep{shen2023fairness} and in the presence of noisy sensitive attributes \citep{mehrotra2022fair}. This line of research analyzes the effect of affirmative actions like the Rooney rule on the utility to the principal or how implicit bias affects the diversity of the selection set.

Our work is also motivated by Equality of opportunity framework, first introduced by \citep{NIPS2016_9d268236} in the classification setting. It has provided a compelling notion of balancing the cost burden among stakeholders \citep{10.1145/3097983.3098095, pmlr-v108-awasthi20a, DBLP:journals/bigdata/Chouldechova17}. For rankings, there has been some work in transferring the idea of equalized odds with learning a ranking function during training \citep{10.1145/3366424.3380048} to reduce disparate exposure or augmenting the training loss with regularizers that minimize costs for both groups \citep{DBLP:journals/corr/abs-2102-05996, 48758}. Our work extends this literature to introduce a framework connecting the unfairness in rankings due to the disparate uncertainty to the distribution of cost burden among stakeholders by anchoring on the fairness of random lottery.

\section{Proofs}
\label{Proofs:section 6}
\subsection{Proof of Theorem~\ref{Thm:Thm1}}
\label{proof:Thm1}
\begin{proof}
We use linear duality for proving this theorem. In order to find a lower bound on the cost optimal ranking that satisfies the EOR fairness constraint, we relax the corresponding Integer Linear Problem (ILP) to a Linear Program (LP) by turning any integer constraints $X \in \{0,1\}$ in the primal into $0 \leq X \leq 1$. For the relaxed LP, we formulate its dual and construct a set of dual variables $\lambda$ corresponding to the solution from the EOR Algorithm. With the dual solution of EOR and the relaxed LP solution, we obtain an upper bound of the duality gap. Since the upper bound on this duality gap is w.r.t.\ the relaxed LP, it will also be an upper bound for the optimal ILP.

We define the primal of the LP for finding a solution $X$ as follows
\begin{align} \max_{X \geq 0}  \quad f(X) &= \frac{P^TX}{\relCand(A)+\relCand(B)}  \tag{Primal} \label{primal_obj}\\ \text{s.t.} \quad X &\leq 1  \label{x_constraint}
    \\
    X^T\basis &\leq k \tag{select up to k elements} \label{k_constraint}
    \\
    Q_{A,B}^TX &\leq \delta(\EORanking_k) \label{delta_max_constraint1}
    \\
    Q_{B,A}^TX &\leq \delta(\EORanking_k) \label{delta_max_constraint2} \end{align} 
We define $Q_{A,B} \in \real^{n}$ where each element of $Q_{A,B}$ is $q_i(\indicator_{A} - \indicator_{B})_{i} $,  $q_{i \in g} = \frac{p_{i}}{\relCand(g)}$ and $Q_{A,B}= -Q_{B,A}$. Note that the \ref{primal_obj} objective is equivalent to minimizing the total cost = $1-\frac{P^TX}{\relCand(A)+\relCand(B)}$.

The first constraint \eqref{x_constraint} ensures valid values for $X$ (with corresponding dual variables $\lambda_i'$ ). The second constraint is for selecting $k$ candidates (dual variable $\lambda_k$ ) and the last two constraints \eqref{delta_max_constraint1} and \eqref{delta_max_constraint2} ensure that the ranking solution is EOR-fair optimal (dual variables $\lambda_{A,B}$, $\lambda_{B,A}$). The Dual LP is formed as follows

\begin{align} \min_{\lambda \geq 0} &\quad g(\lambda) = \delta(\EORanking_k)(\lambda_{A,B} + \lambda_{B,A}) + k\lambda_k + \sum_{i=1}^{n}\lambda_{i}' \tag{Dual} \\ \text{s.t.} &\quad Q_{A,B}(\lambda_{A,B}-\lambda_{B,A}) + \lambda_k + \lambda' \geq \frac{P}{\relCand(A)+\relCand(B)} \label{eq:dual_constraint} \end{align} 

We construct a feasible point of the dual from the EOR solution as follows. The key insight here is to reason w.r.t the last elements selected (or the first elements available if no element from the group has been selected) by the EOR Algorithm at prefix $k$ from each of the groups A and B, namely $k_A, k_B$ respectively. 

\begin{eqnarray} 
   \lambda_{A,B} &=& \frac{1}{\relCand(A)+\relCand(B)}\left[\frac{p_A-p_B}{q_A+q_B}\right]_{+} \label{eq:lambda_1} \\
    \lambda_{B,A} &=& \frac{1}{\relCand(A)+\relCand(B)}\left[-\left(\frac{p_A-p_B}{q_A+q_B} \right)\right]_{+} \label{eq:lambda_2}
\end{eqnarray}

Using \eqref{eq:lambda_1} and \eqref{eq:lambda_2} 
we know that only ever one of $\lambda_{A,B}$ or $\lambda_{B,A}$ is non zero. If $p_A\geq p_B$, then $\lambda_{A,B} \geq 0$ and $\lambda_{B,A} = 0$. Similarly, if $p_B\geq p_A$, then $\lambda_{B,A} \geq 0$ and $\lambda_{A,B} = 0$. 

We construct $\lambda_k$ and $\lambda_i'$ as follows
\begin{eqnarray} 
\lambda_k &=& \left[\frac{p_A}{\relCand(A)+\relCand(B)}-q_A(\lambda_{A,B} - \lambda_{B,A})\right] = \left[\frac{p_B}{\relCand(A)+\relCand(B)}-q_B(\lambda_{B,A} - \lambda_{A,B})\right] \label{eq:lambda_3}\\
\lambda_{i \in A}' &=& \left[\frac{p_i}{\relCand(A)+\relCand(B)}- \lambda_k -q_i(\lambda_{A,B} - \lambda_{B,A})\right]_{+} \label{eq:lambda'}\\
\lambda_{i\in B}' &=& \left[\frac{p_i}{\relCand(A)+\relCand(B)}- \lambda_k -q_i(\lambda_{B,A} - \lambda_{A,B})\right]_{+} \label{eq:lambda'_B}
\end{eqnarray}

We prove that the constructed dual variables $\lambda$ are non-negative in Lemma~\ref{lemma:6.2} and that $\lambda'=0$ for any element not selected in the EOR ranking. In Lemma~\ref{lemma:6.3}, we prove that the constructed dual variables $\lambda$ are feasible. Given the feasibility of dual variables, we analyze the \textbf{duality gap} given by

\begin{displaymath}
    g(\lambda^*)-f(X)
    =\delta(\EORanking_k)(\lambda_{A,B} + \lambda_{B,A}) + k\lambda_k + \sum_{i=1}^{n}\lambda_{i}' - \frac{P^TX}{\relCand(A)+\relCand(B)} 
\end{displaymath}

From Lemma~\ref{lemma:6.2},  $\lambda_i'=0$ for $i > k_A$, $\lambda_j'=0$ for $j> k_B$ , where $k_A$ elements are selected from group A, $k_B$ from group B by the EOR Algorithm and $k=k_A+k_B$. 
Substituting the values for $\lambda'$ from \eqref{eq:lambda'}, \eqref{eq:lambda'_B}, the duality gap is

\begin{eqnarray}
\nonumber
& = &\delta(\EORanking_k)(\lambda_{A,B} + \lambda_{B,A}) + k\lambda_k + \sum_{i=1}^{k_A}\left(\frac{p_i}{\relCand(A)+\relCand(B)}-\lambda_k-q_i(\lambda_{A,B} - \lambda_{B,A}) \right) \\
\nonumber
& + & \sum_{j=1}^{k_B} \left(\frac{p_j}{\relCand(A)+\relCand(B)}-\lambda_k-q_j(\lambda_{B,A} - \lambda_{A,B}) \right)- \frac{P^TX}{\relCand(A)+\relCand(B)} 
\end{eqnarray}

We know that $\sum_{i=1}^{k_A}\lambda_k+\sum_{j=1}^{k_B}\lambda_k = k\lambda_k$ and $P^TX = \sum_{i=1}^{k_A}p_i + \sum_{j=1}^{k_B}p_j$. 
Further, only one of $\lambda_{A,B}$ or $\lambda_{B,A}$ is non-negative according to \eqref{eq:lambda_1}, \eqref{eq:lambda_2}. 

If $\lambda_{A,B} \geq0$, then the duality gap can be written as 
\begin{displaymath}
    =\delta(\EORanking_k)\lambda_{A,B} -\sum_{i=1}^{k_A}q_i\lambda_{A,B}+\sum_{j=1}^{k_B}q_j\lambda_{A,B}
    =\lambda_{A,B} \left(\delta(\EORanking_k)- \left(\sum_{i=1}^{k_A}q_i-\sum_{j=1}^{k_B}q_j \right) \right)
\end{displaymath} 

Since we have $-\delta(\EORanking_k) \leq \sum_{i=1}^{k_A}q_i-\sum_{j=1}^{k_B}q_j \leq \delta(\EORanking_k)$ from Lemma~\ref{lemma:6.1}, 
\begin{equation} \label{eq:thm1_1}
    \text{Duality Gap}=\lambda_{A,B} \left(\delta(\EORanking_k)- \left(\sum_{i=1}^{k_A}q_i-\sum_{j=1}^{k_B}q_j \right) \right) \leq  2\lambda_{A,B} \delta(\EORanking_k)
\end{equation}

If  $\lambda_{B,A}\geq0$, then the duality gap can be written as
\begin{displaymath}
    =\delta(\EORanking_k)\lambda_{B,A} +\sum_{i=1}^{k_A}q_i\lambda_{B,A}-\sum_{j=1}^{k_B}q_j\lambda_{B,A}
    =\lambda_{B,A} \left(\delta(\EORanking_k)+(\sum_{i=1}^{k_A}q_i-\sum_{j=1}^{k_B}q_j)\right)
\end{displaymath}

and again, since $-\delta(\EORanking_k) \leq \sum_{i=1}^{k_A}q_i-\sum_{j=1}^{k_B}q_j \leq \delta(\EORanking_k)$ from Lemma~\ref{lemma:6.1},
\begin{equation} \label{eq:thm1_2}
    \text{Duality Gap}=\lambda_{B,A} \left(\delta(\EORanking_k)+(\sum_{i=1}^{k_A}q_i-\sum_{j=1}^{k_B}q_j)\right) \leq 2\lambda_{B,A} \delta(\EORanking_k)
\end{equation}

From Eqs.~\eqref{eq:lambda_1}, \eqref{eq:lambda_2}, \eqref{eq:thm1_1}, \eqref{eq:thm1_2} 
, the duality gap between EOR solution and the optimal solution is bounded by \[ \frac{2\delta(\EORanking_k)}{\relCand(A)+\relCand(B)}\left|\frac{p_A-p_B}{q_A+q_B}\right| \]

This proves that the EOR solution can only be ever as worse as $\phi\delta(\EORanking_k)$ when compared with the optimal solution, where $\phi =\frac{2}{\relCand(A)+\relCand(B)}\left|\frac{p_A-p_B}{q_A+q_B}\right|$
\end{proof}

\begin{lemma}
\label{lemma:6.1}
\normalfont
EOR ranking is $\delta(\EORanking_k)$ fairness optimal, implying that $-\delta(\EORanking_k) \leq \sum_{i=1}^{k_A}q_i-\sum_{j=1}^{k_B}q_j \leq \delta(\EORanking_k)$.

Since $\sum_{i=1}^{k_A}q_i-\sum_{j=1}^{k_B}q_j = \frac{\relCand(A|\sigma_k)}{\relCand(A)} - \frac{\relCand(B|\sigma_k)}{\relCand(B)}$, the lemma follows directly from the definition of $\delta(\sigma_k)$ in Eq.~\eqref{eq:EOR_compute} and the EOR ranking principle of choosing the candidate that minimizes $\delta(\sigma_k)$. 

\qed
\end{lemma}

\begin{lemma}
\label{lemma:6.2}
\normalfont
The constructed dual variables $\lambda \geq 0$. In particular, for any $i > k_A $ in group A and $j > k_B$ in group B, it holds that $\lambda_i' = 0$ and $\lambda_j' = 0$ and for any $i \leq k_A $ and $j \leq k_B$, it holds that $\lambda_i' \geq 0$ and $\lambda_j' \geq 0$.
\end{lemma}

\begin{proof} \normalfont \label{proof:lemma_2}
In this Lemma, we show that $\lambda'=0$ for the elements not selected by the EOR Algorithm and $\lambda'\geq0$ for the elements that were selected. Without loss of generality, we consider the element at index $i$ that belongs to group $A$. 
\begin{eqnarray}
    \nonumber
    \lambda_{i \in A}' &=& \left[\frac{p_i}{\relCand(A)+\relCand(B)}- \lambda_k -q_i(\lambda_{A,B}-\lambda_{B,A})\right]_{+} \\
    \nonumber
    &=& \left[\frac{p_i}{\relCand(A)+\relCand(B)} - \frac{p_A}{\relCand(A)+\relCand(B)} +q_A(\lambda_{A,B}-\lambda_{B,A}) -q_i(\lambda_{A,B}-\lambda_{B,A})\right]_{+} \\
    &=& \left[\frac{p_i- p_A}{\relCand(A)+\relCand(B)} + (q_i-q_A)(\lambda_{B,A}-\lambda_{A,B})\right]_{+} \label{eq:lambda_2_ge_0}
\end{eqnarray}

The second equality above is obtained by substituting $\lambda_k$ from Eq.~\eqref{eq:lambda_3} and the last equality by rearranging. We now consider two cases -- for elements not selected and selected by the EOR Algorithm respectively.

\textbf{Case I:} Elements not selected by the EOR Algorithm.

We have i) $p_i \leq  p_A$ and $q_i \leq q_A$ as EOR selects in decreasing order of probabilities, and ii) either $\lambda_{A,B} \geq 0 $ or $\lambda_{B,A} \geq 0$ as only one of them can be nonzero from \eqref{eq:lambda_1}, \eqref{eq:lambda_2}. 

In Eq.~\eqref{eq:lambda_2_ge_0}, if $\lambda_{B,A} \geq 0$, then $\lambda_{A,B}=0$ and with $p_i \leq  p_A$, $q_i \leq q_A$ the resultant quantity would be negative, which would result in $\lambda_i'$ clipped to $0$. 
\begin{eqnarray}
    \nonumber
    \lambda_i' & = & \left[\frac{p_i- p_A}{\relCand(A)+\relCand(B)}+(q_i-q_A)\lambda_{B,A}\right]_{+} \\
    \nonumber
    & \leq & 0
\end{eqnarray}

In Eq.~\eqref{eq:lambda_2_ge_0}, if $\lambda_{A,B} \geq 0$, then $\lambda_{B,A} = 0$. We can then substitute 
$\lambda_{A,B} = \frac{1}{\relCand(A)+\relCand(B)} \left(\frac{p_A-p_B}{q_A+q_B} \right)$ in Eq.~\eqref{eq:lambda_2_ge_0},
\begin{eqnarray}
    \nonumber
    \lambda_i' & = & \left[\frac{p_i- p_A}{\relCand(A)+\relCand(B)}-(q_i-q_A)\lambda_{A,B}\right]_{+} \\
    \nonumber
    & = & \left[\frac{p_i- p_A}{\relCand(A)+\relCand(B)} - \frac{(q_i-q_A)}{\relCand(A)+\relCand(B)} \left(\frac{p_A-p_B}{q_A+q_B} \right)\right]_{+}\\
    \nonumber
    & = & \frac{1}{\relCand(A)+\relCand(B)} \left[\frac{p_B(q_i-q_A)+q_B(p_i-p_A)}{q_A+q_B}\right]_{+} \\
    \nonumber
    & = & 0
\end{eqnarray}
The second last term evaluates to $\leq 0$ and so the last equality holds because $\lambda_{i}'$ is clipped to $0$.

Thus, for any element not been selected by the EOR Algorithm i.e. $i > k_A$, the corresponding dual variable $\lambda_i' = 0$.
Analogously, for any element $j > k_B$ in group B it can be shown that $\lambda_j' = 0$. We have shown that for any element not  selected by the EOR Algorithm the corresponding dual variable $\lambda' = 0$.

\textbf{Case II:} Elements selected by the EOR Algorithm.

We have i) $p_i \geq  p_A$ and $q_i \geq q_A$ as EOR selects in decreasing order of probabilities, and ii) $\lambda_{A,B} \geq 0 $ or $\lambda_{B,A} \geq 0$ as only one of them can be non zero. 

In Eq.~\eqref{eq:lambda_2_ge_0}, if $\lambda_{B,A} \geq 0$, then $\lambda_{A,B} = 0$ and with $p_i \geq  p_A$, $q_i \geq q_A$ the resultant quantity in \eqref{eq:lambda_2_ge_0} would be $\geq 0$, so that $\lambda_i'\geq 0$.
\begin{eqnarray}
    \nonumber
    \lambda_i' & = & \left[\frac{p_i- p_A}{\relCand(A)+\relCand(B)}+(q_i-q_A)\lambda_{B,A}\right]_{+} \\
    \nonumber
    & \geq & 0
\end{eqnarray}

In Eq.~\eqref{eq:lambda_2_ge_0}, if $\lambda_{A,B} \geq 0$, then $\lambda_{B,A} = 0$. We can then substitute $\lambda_{A,B} = \frac{1}{\relCand(A)+\relCand(B)} \left(\frac{p_A-p_B}{q_A+q_B} \right)$ in \eqref{eq:lambda_2_ge_0},
\begin{eqnarray}
    \nonumber
    \lambda_i' & = & \left[\frac{p_i- p_A}{\relCand(A)+\relCand(B)}-(q_i-q_A)\lambda_{A,B}\right]_{+} \\
    \nonumber
    & = & \left[\frac{p_i- p_A}{\relCand(A)+\relCand(B)} - \frac{(q_i-q_A)}{\relCand(A)+\relCand(B)} \left(\frac{p_A-p_B}{q_A+q_B} \right)\right]_{+}\\
    \nonumber
    & = & \frac{1}{\relCand(A)+\relCand(B)}\left[\frac{p_B(q_i-q_A)+q_B(p_i-p_A)}{q_A+q_B}\right]_{+} \\
    \nonumber
    &\geq& 0
\end{eqnarray}
The second last term evaluates to $\geq 0$, so the last equality holds.
Thus, for any element selected by the EOR Algorithm in group $A$ i.e. $i\leq k_A$, the corresponding dual variable $\lambda' \geq 0$. 
Analogously, for any element $j \leq k_B$ in group B,  $\lambda_i' \geq 0$.
We have shown that for any element  selected by the EOR Algorithm the corresponding dual variable $\lambda' \geq 0$.

We now show that $\lambda_k \geq 0$. From Eq.~\eqref{eq:lambda_3},
\begin{eqnarray} 
\lambda_k &=& \left[\frac{p_A}{\relCand(A)+\relCand(B)}-q_A(\lambda_{A,B} - \lambda_{B,A})\right]  \label{eq:lambda_k_A}
\end{eqnarray}
If $\lambda_{A,B} \geq 0$ , then $\lambda_{B,A} = 0$. Substituting $\lambda_{A,B} = \frac{1}{\relCand(A)+\relCand(B)} \left(\frac{p_A-p_B}{q_A+q_B} \right)$ in Eq.~\eqref{eq:lambda_k_A},
\begin{eqnarray} 
 \lambda_k&=& \left[\frac{p_A}{\relCand(A)+\relCand(B)}-q_A\lambda_{A,B}\right] \nonumber \\
 &=& \left[\frac{p_A}{\relCand(A)+\relCand(B)}-\frac{q_A}{\relCand(A)+\relCand(B)}\left(\frac{p_A-p_B}{q_A+q_B}\right)\right] \nonumber\\
  & = & \frac{1}{\relCand(A)+\relCand(B)} \left(\frac{p_Aq_B+q_Ap_B}{q_A+q_B}\right) \nonumber \\
  &\geq& 0 \nonumber
\end{eqnarray}
The last inequality follows since each of the terms $p_A, q_A, p_B, q_B$ are $\geq 0$. If $\lambda_{B,A} \geq 0$, then $\lambda_{A,B} = 0$. By substituting $\lambda_{B,A} = \frac{1}{\relCand(A)+\relCand(B)} \left(-\frac{p_A-p_B}{q_A+q_B} \right)$ in Eq.~\eqref{eq:lambda_k_A}, we similarly get $\lambda_k \geq 0$.

The two duals $\lambda_{A,B}, \lambda_{B,A}$ are $\geq 0$ by their construction in Eqs.~\eqref{eq:lambda_1}, \eqref{eq:lambda_2}. Thus, we have shown that all the constructed dual variables $\lambda \geq 0$.
\end{proof}

\begin{lemma}
    \label{lemma:6.3}
    \normalfont
    The dual variables $\lambda=[\lambda_1' \cdots \lambda_n', \lambda_k, \lambda_{A,B}, \lambda_{B,A}] $ are always feasible.
\end{lemma}
\begin{proof} \label{proof:lemma_3}
In Lemma~\ref{lemma:6.2}, we proved that the constructed $\lambda \geq 0$. We now show that they satisfy the duality constraint. 

For some element $i$, the duality constraint implies that
\begin{equation} \label{eq:duality_proof}
    q_i(\lambda_{A,B}-\lambda_{B,A}) + \lambda_k + \lambda_i' \geq \frac{p_i}{\relCand(A)+\relCand(B)}
\end{equation}

Without loss of generality, we consider  element at index $i$ that belongs to group $A$. Similar to Lemma~\ref{lemma:6.2}, we consider two cases. 

\textbf{Case I:} Elements not selected by the EOR Algorithm.

Using the fact that $\lambda_i'=0$ for $i > k_A$ from Lemma~\ref{lemma:6.2}, and substituting $\lambda_k$ from Eq.~\eqref{eq:lambda_3}, we get
\begin{eqnarray}
    \nonumber
     q_i(\lambda_{A,B}-\lambda_{B,A}) + \lambda_k + \lambda_i' &=& q_i(\lambda_{A,B}-\lambda_{B,A}) + \lambda_k\\
     \nonumber
     &=& q_i(\lambda_{A,B}-\lambda_{B,A}) + \frac{p_A}{\relCand(A)+\relCand(B)}-q_A(\lambda_{A,B}-\lambda_{B,A})\\
     \nonumber
     &=& \frac{p_A}{\relCand(A)+\relCand(B)}+(q_i - q_A)(\lambda_{A,B}-\lambda_{B,A}) 
\end{eqnarray}

We have i) $p_i \leq  p_A$ and $q_i \leq q_A$ as EOR selects in decreasing order of probabilities, and ii) either $\lambda_{A,B} \geq 0 $ or $\lambda_{B,A} \geq 0$ as only one of them can be nonzero.
If $\lambda_{A,B} \geq 0$, then substituting $\lambda_{A,B}$ , 
\begin{eqnarray}
    \nonumber
    q_i(\lambda_{A,B}-\lambda_{B,A}) + \lambda_k + \lambda_i'
    &=& \frac{p_A}{\relCand(A)+\relCand(B)}+(q_i-q_A)\lambda_{A,B}\\
    \nonumber
    &=& \frac{p_A}{\relCand(A)+\relCand(B)}+\frac{(q_i-q_A)}{\relCand(A)+\relCand(B)}(\frac{p_A-p_B}{q_A+q_B})\\
    \nonumber
    &=& \frac{1}{\relCand(A)+\relCand(B)}\left[p_i + \frac{p_B(q_A-q_i) + q_B(p_A-p_i)}{q_A+q_B}\right]\\
    \nonumber
    &\geq& \frac{p_i}{\relCand(A)+\relCand(B)}
\end{eqnarray}
Similarly, we can show that the dual constraint is satisfied if $\lambda_{B,A} \geq 0$.
Thus, for any element not selected by the EOR Algorithm i.e. $i > k_A$, the corresponding dual constraint is satisfied.
Analogously, for any element $j > k_B$ in group B it can be shown that the corresponding dual constraint is satisfied. We have shown that for any element not selected by EOR Algorithm the corresponding dual constraint is satisfied.

\textbf{Case II:} Elements selected by the EOR Algorithm.

Using the fact that $\lambda_i' \geq 0$ for $i \leq k_A$ from Lemma~\ref{lemma:6.2}, and substituting $\lambda_k$ from \eqref{eq:lambda_3}, $\lambda_i'$ for $i \leq k_A$ in \eqref{eq:duality_proof}, we get 
\begin{eqnarray}
    \nonumber
    q_i(\lambda_{A,B}-\lambda_{B,A}) + \lambda_k + \lambda_i' 
    = q_i(\lambda_{A,B}-\lambda_{B,A}) + \lambda_k + \frac{p_i}{\relCand(A)+\relCand(B)} -\lambda_k -q_i(\lambda_{A,B}-\lambda_{B,A}) 
    = \frac{p_i}{\relCand(A)+\relCand(B)}
\end{eqnarray}
Thus, for any element selected by the EOR Algorithm i.e. $i \leq k_A, j \leq k_B$, the corresponding dual constraint is satisfied.
\end{proof}

\begin{figure*}[t]
\centering
    \includegraphics[width=0.5\textwidth]{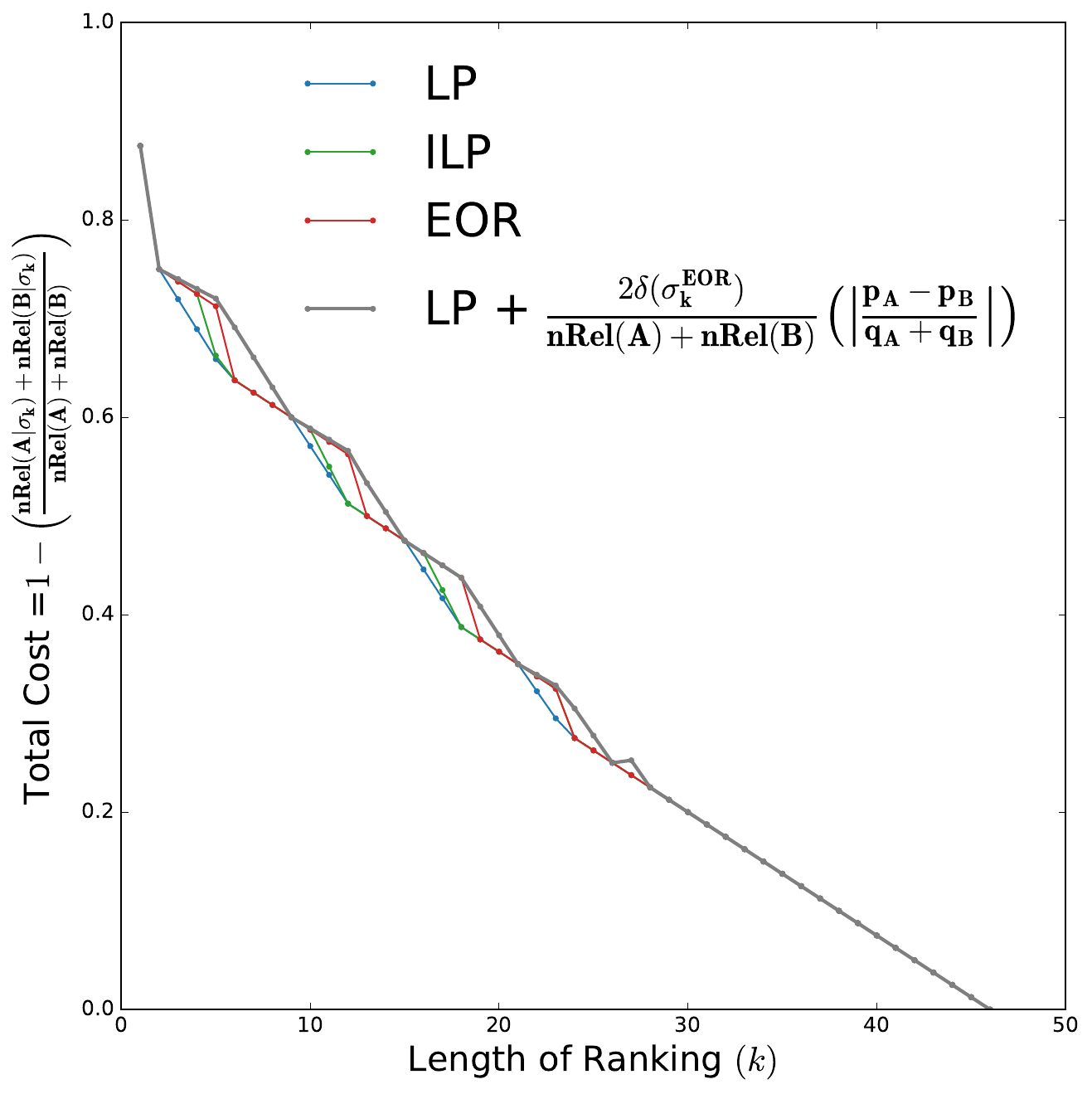}
    \caption{\normalfont Cost Optimality Gap of a synthetic example with $p_{i \in A}=[1, 0.6, 0.5, 0.5, 0.4, 0.1, \cdots 0.1], \relCand(A)=4, S(A)=15$, and $p_{i \in B}=[1, 0.1 \cdots 0.1], \relCand(B)=4, S(B)=31$. The cost from EOR ranking is nearly optimal to the ILP or even the relaxed LP solution. Further the bound obtain in Theorem~\ref{Thm:Thm1} (in grey) is tight for many $k$ prefixes.}
  \label{fig:thm1_iluustration}
  \vspace*{-4mm}
\end{figure*}
\vspace*{-2mm}

We demonstrate the cost optimality bound proved in Theorem~\ref{Thm:Thm1} in Figure~\ref{fig:thm1_iluustration} that shows an example with a ranking produced by Linear Program (LP), Integer Linear Program (ILP), and the EOR algorithm along with the upper bound on the cost computed from the duality gap proved in Theorem~\ref{Thm:Thm1}. The example is constructed such that $\prob(r_i|\Data)_{i \in A}=[1, 0.6, 0.5, 0.5, 0.4, 0.1, \cdots 0.1], \relCand(A)=4$, and $\prob(r_i|\Data)_{i \in B}=[1, 0.1 \cdots 0.1], \relCand(B)=4$.  Figure~\ref{fig:thm1_iluustration} shows that at most prefixes $k$, the EOR cost (in red) is optimal coinciding with the cost from ILP solution (in green) as well as with the LP solution (in blue). Further, when the EOR ranking does not coincide with the LP solution, the upper bound $\frac{2\delta(\EORanking_k)}{\relCand(A)+\relCand(B)}\left|\frac{p_A-p_B}{q_A+q_B}\right|$ is relatively small as is shown by the LP + duality gap (in grey).

We now present the proof for the global a priori bound on $\delta(\EORanking_k)$ for two groups A,B.

\subsection{Proof for Theorem \ref{thm:thm_bound}}
\begin{proof} 
\label{proof:Thm2}

Let $\PRPRankingGroup{A},\PRPRankingGroup{B}$ be the PRP rankings for elements in group A and B respectively.
We show by induction that for any given prefix $k$, EOR algorithm selects the element such that $\left| \delta(\EORanking_k) \right| \leq \delta_{max}$
and as a consequence of Theorem~\ref{Thm:Thm1}, we get a global cost guarantee of $\phi\delta_{max}$.

In the remaining proof, we drop the superscript of EOR for simplicity and $\sigma_j$ refers to $\EORanking_j$.

Consider the base case of $k=1$. Algorithm~\ref{alg:EOR_alg} will select $\argmin\left\{\frac{\PRPRankingGroup{A}[1]}{\relCand(A)}, \frac{\PRPRankingGroup{B}[1]}{\relCand(B)}\right\}$ 
resulting in the lower $\delta(\sigma_{k=1})$. If $\frac{\PRPRankingGroup{A}[1]}{\relCand(A)} \leq \frac{\PRPRankingGroup{B}[1]}{\relCand(B)}$, then 
$\delta(\sigma_1) = \frac{\PRPRankingGroup{A}[1]}{\relCand(A)} \leq \frac{1}{2}\left(\frac{\PRPRankingGroup{A}[1]}{\relCand(A)} + \frac{\PRPRankingGroup{B}[1]}{\relCand(B)} \right)$. Similarly, if $\frac{\PRPRankingGroup{B}[1]}{\relCand(B)} \leq \frac{\PRPRankingGroup{A}[1]}{\relCand(A)}$, then $\delta(\sigma_{k=1})$ denoted in short by $\delta(\sigma_1) = \frac{\PRPRankingGroup{B}[1]}{\relCand(B)} \leq \frac{1}{2}\left(\frac{\PRPRankingGroup{A}[1]}{\relCand(A)} + \frac{\PRPRankingGroup{B}[1]}{\relCand(B)} \right)$. Thus, at $k=1$, by selecting the element with lower $\delta$, EOR constraint is satisfied, i.e. $\delta(\sigma_1) \leq \delta_{max}$.

We assume that for a given $k-1, \left| \delta(\sigma_{k-1}) \right| \leq \delta_{max}$. Further, without loss of generality, we assume that $\delta(\sigma_{k-1}) \geq 0$. We now show that at $k, \left| \delta(\sigma_{k}) \right| \leq \delta_{max}$ by considering the following cases. First, we show that if adding the element from one of the groups violates the $\delta_{max}$ constraint, then adding the element from the other group guarantees the satisfaction of $\delta_{max}$ constraint because EOR Algorithm selects the element that minimizes $\delta$. Secondly, in the case where adding an element from either group does not violate the $\delta_{max}$ constraint, EOR algorithm will select the element that minimizes $\left| \delta(\sigma_k) \right|$ resulting in $\left| \delta(\sigma_k) \right| \leq \delta_{max}$.
Finally, we show that when all the elements have run out from one of the groups at $k-1$, adding remaining elements from the other group will always satisfy the $\delta_{max}$ constraint.

We assume that adding the element from group A with relevance probability $p_i$ at $k$, exceeds the $\delta_{max}$ constraint.
\begin{equation} 
    \label{induction_a_violate}
    \delta(\sigma_{k-1}) + \frac{p_i}{\relCand(A)} > \delta_{max}
\end{equation}

Adding the element $p_j$ from B at this prefix,
\begin{equation} \label{induction_b_1}
    \delta(\sigma_{k}) = \delta(\sigma_{k-1}) - \frac{p_j}{\relCand(B)} \leq \delta(\sigma_{k-1}) \leq \delta_{max} 
\end{equation}
The last inequality holds by the induction assumption at $k-1$.

Further, since $\frac{p_j}{\relCand(B)} \leq \frac{\PRPRankingGroup{B}[1]}{\relCand(B)}$, and $\delta_{max}=\frac{1}{2}\left(\frac{\PRPRankingGroup{A}[1]}{\relCand(A)}+\frac{\PRPRankingGroup{B}[1]}{\relCand(B)}\right)$, the above can be reduced to 
\begin{eqnarray}
    \nonumber
    \delta(\sigma_{k}) & = & \delta(\sigma_{k-1}) - \frac{p_j}{\relCand(B)} \geq \delta(\sigma_{k-1}) - \frac{\PRPRankingGroup{B}[1]}{\relCand(B)} \\
    \delta(\sigma_{k}) & \geq & \delta(\sigma_{k-1}) + \frac{\PRPRankingGroup{A}[1]}{\relCand(A)} - 2\delta_{max} \label{induction_b_2}
\end{eqnarray}
Now using $\frac{p_i}{\relCand(A)} \leq \frac{\PRPRankingGroup{A}[1]}{\relCand(A)}$, and Eq.~\eqref{induction_b_2} above,
\begin{equation} \label{induction_b_3}
    \delta(\sigma_k) \geq \delta(\sigma_{k-1}) + \frac{\PRPRankingGroup{A}[1]}{\relCand(A)} - 2\delta_{max} \geq \delta(\sigma_{k-1}) + \frac{p_i}{\relCand(A)} - 2\delta_{max} 
\end{equation}

Using Eqs.~\eqref{induction_b_3} and \eqref{induction_a_violate}, 
\begin{equation}
    \delta(\sigma_k) \geq \delta(\sigma_{k-1}) + \frac{p_i}{\relCand(A)} - 2\delta_{max} > -\delta_{max} \label{induction_b_4}
\end{equation}

We have shown that, if $\left| \delta(\sigma_k) \right|$ exceeds $\delta_{max}$ by adding the element from group A (from \eqref{induction_a_violate}), then the element in group B will satisfy $\left| \delta(\sigma_k) \right| \leq \delta_{max}$ (from \eqref{induction_b_1} and \eqref{induction_b_4}). Since the EOR algorithm minimizes $\left| \delta(\sigma_k) \right|$, it will select the element from group B at prefix $k$ rather than the element from group A. Thus, $\vert \delta(\sigma_k) \vert \leq \delta_{max}$ in this case. 

Similarly, we can show that if $\vert \delta(\sigma_k) \vert$ exceeds $\delta_{max}$ by adding the element from group B, then adding the element from group A would result in $\vert \delta(\sigma_k) \vert \leq \delta_{max}$ and would be selected by the EOR algorithm at prefix $k$.

Finally, we consider the case where all the elements in a particular group have already been selected. Without loss of generality, let's assume that this is true with all the elements in group B added by prefix $k-1$. We need to show that adding from the remaining elements in group A would still satisfy $\left| \delta \right| \leq \delta_{max}$ for the remaining prefixes.

From our assumption, $\frac{\relCand(B|\sigma_{k-1})}{\relCand(B)} = 1$ since all elements from group B were selected at prefix $k-1$. From the inductive hypothesis $\left| \delta(\sigma_{k-1})\right| \leq \delta_{max}$,

\begin{equation}
    \left| \delta(\sigma_{k-1}) \right|= \left| \frac{\relCand(A|\sigma_{k-1})}{\relCand(A)} - \frac{\relCand(B|\sigma_{k-1})}{\relCand(B)} \right | \leq \delta_{max} \label{induction_c1}
\end{equation}

Since $\frac{\relCand(A|\sigma_{k-1})}{\relCand(A)} \leq 1$ as some elements remain in group A, 

\begin{equation}
    \delta(\sigma_{k-1}) = \frac{\relCand(A|\sigma_{k-1})}{\relCand(A)} -1 \geq  -\delta_{max} \label{induction_c2}
\end{equation}

After adding the element $p_i$ from group A at prefix $k$ and from \eqref{induction_c2},
\begin{eqnarray}
    \delta(\sigma_k) &=& \delta(\sigma_{k-1}) + \frac{p_i}{\relCand(A)}  = \frac{\relCand(A|\sigma_{k-1})}{\relCand(A)} - 1 + \frac{p_i}{\relCand(A)} \geq -\delta_{max} + \frac{p_i}{\relCand(A)}  \nonumber \\
    \delta(\sigma_k) &\geq& -\delta_{max} \label{eq:delta_k} 
\end{eqnarray}

Additionally, since $\frac{\relCand(A|\sigma_n)}{\relCand(A)} =1$ implying $\frac{\relCand(A|\sigma_k)}{\relCand(A)} \leq 1$,   
\begin{equation}
    \delta(\sigma_k)  = \frac{\relCand(A|\sigma_k)}{\relCand(A)} -1 \leq 0 \label{eq:delta_k2}
\end{equation}

From Eqs.~\eqref{eq:delta_k} and \eqref{eq:delta_k2}, $-\delta_{max} \leq \delta(\sigma_k) \leq 0 $ and thus EOR algorithm will add all the remaining elements from group A resulting in $\left| \delta(\sigma_k)  \right| \leq \delta_{max}$. Analogously, it can be shown that if all the elements from group A had been added by prefix $k$, adding the next element from group B would satisfy $\left| \delta(\sigma_k)  \right| \leq \delta_{max}$.

Thus, we have shown that Algorithm~\ref{alg:EOR_alg} provides rankings such that for any prefix $k$,  $\vert \delta(\sigma_k)  \vert \leq \delta_{max}$, where $\delta_{max} =  \frac{1}{2}\left(\frac{\PRPRankingGroup{A}[1]}{\relCand(A)}+\frac{\PRPRankingGroup{B}[1]}{\relCand(B)}\right)$. 
As a consequence of this and Theorem~\ref{Thm:Thm1}, EOR rankings have total cost bounded by $\phi\delta_{max}$ for any prefix $k$ of the ranking, where $\phi=\frac{2}{\relCand(A)+\relCand(B)}\left|\frac{p_A-p_B}{q_A+q_B}\right|$. 

\end{proof}

Next, we present the proof comparing costs from $\EOR, \Uniform$ at prefix $k$, where $\delta(\EORanking_k)=0$.

\subsection{Proof for Proposition \ref{prop:uniform_eor}}
\label{proof:prop_uniform_eor}
\begin{proof}
When $\delta(\EORanking_k)=0$, by the definition of EOR fairness, we have that $\frac{\relCand(A|\sigma^{EOR}_k)}{\relCand(A)} =  \frac{\relCand(B|\sigma^{EOR}_k)}{\relCand(B)}$. As a result, the total cost ($1-\frac{\relCand(A|\sigma^{EOR}_k)+\relCand(B|\sigma^{EOR}_k)}{\relCand(A)+\relCand(B)}$) as well as subgroup cost would be equal to
\begin{equation} \label{delta_k_0}
    1-\frac{\relCand(A|\sigma^{EOR}_k)}{\relCand(A)} =  1-\frac{\relCand(B|\sigma^{EOR}_k)}{\relCand(B)}    
\end{equation}

We also know that $\frac{{\relCand(A|\sigma^{EOR}_k)}}{k_A} \geq \frac{\relCand(A)}{S(A)}$ and $\frac{{\relCand(B|\sigma^{EOR}_k)}}{k_B} \geq \frac{\relCand(B)}{S(B)}$, since the EOR algorithm selects top $k_A, k_B$ elements from each of the groups (with $k_A+k_B=k, S(A) + S(B) =n$), having a higher mean relevance than that of the group itself. 
\begin{eqnarray} 
\label{eor_rel_A}
\frac{\relCand(A|\sigma^{EOR}_k) S(A)}{\relCand(A)} \geq k_A
\\
\label{eor_rel_B}
\frac{\relCand(B|\sigma^{EOR}_k) S(B)}{\relCand(B)} \geq k_B
\end{eqnarray}
Adding Eqs.~\eqref{eor_rel_A}, \eqref{eor_rel_B} and using \eqref{delta_k_0}, we get that 
\begin{eqnarray}
\frac{\relCand(A|\sigma^{EOR}_k) (S(A) + S(B) )}{\relCand(A)} & \geq & k
\nonumber \\
\frac{\relCand(A|\sigma^{EOR}_k)}{\relCand(A)} & \geq & \frac{k}{n}
\Leftrightarrow 1-\frac{\relCand(A|\sigma^{EOR}_k)}{\relCand(A)}  \leq  1-\frac{k}{n} \nonumber
\end{eqnarray}
This and Eq.~\eqref{delta_k_0} are sufficient to claim that the total cost and subgroup costs of uniform policy given by $1-\frac{k}{n}$ will always be higher than the total cost and subgroup costs given by EOR ranking when $\delta(\EORanking_k)=0$.
\end{proof}

\section{Extension to Multiple Groups $G$} 
\label{EOR_alg_extended_details}

In the following, we prove the global cost and fairness guarantee for multiple groups $G$. 

\subsection{Proof for Theorem \ref{corr:corr_multiple_groups}} \label{proof:thm_6_1}
\begin{proof}
The overall strategy for this proof is to consider each pair of groups among the $\frac{G(G-1)}{2}$ pairs and reduce each term of the duality gap to the two group case in Theorem~\ref{Thm:Thm1}. Fortunately, we can achieve such a reduction by careful construction of the dual variables.

The LP  to find a solution $X$ for this problem is formulated as follows

\begin{align} \max_{x \geq 0}  \quad f(x) &= \frac{P^TX}{\sum_{g=1}^G \relCand(g)}  \tag{Primal} \\ \text{s.t.} \quad X &\leq 1  \label{x_constraint_corr} \\ X^T\mathbb.{1} &\leq k \tag{select up to k elements} \label{k_constraint_corr} \\  \text{$G(G-1)$ constraints}&  \left\{
    \begin{array}{rcr}
        \nonumber
        Q_{A,B}^TX &\leq& \delta(\EORanking_k)
        \\ \nonumber
        Q_{B,A}^TX &\leq& \delta(\EORanking_k)
        \\ \nonumber
        \vdots 
    \end{array}     
\right. \end{align} 

The above LP is analogous to the two group case in Theorem~\ref{Thm:Thm1}, with the addition of $G(G-1)$ pairwise constraints ensuring EOR-fairness for all pairs of groups. 

We can construct the dual problem as follows

\begin{align} \min_{\lambda \geq 0} & \quad g(\lambda) = \delta(\EORanking_k)\overbrace{\sum_{\{A,B\}}(\lambda_{A,B} + \lambda_{B,A})}^{G(G-1)/2} + k\lambda_k + \sum_{i=1}^{n}\lambda_{i}' \tag{Dual} \\ \text{s.t.} \quad  &\sum_{\{A,B\}}Q_{A,B}(\lambda_{A,B}-\lambda_{B,A}) + \lambda_k + \lambda' \geq \frac{P}{\sum_{g}\relCand(g)} \label{eq:dual_constraint_corr} \end{align}

We have pairs of dual variables that are constructed from the EOR solution as following 
\begin{eqnarray} 
   \lambda_{A,B} &=& \frac{1}{(G-1)}\frac{1}{\sum_{g}\relCand(g)}\left[\frac{p_{A}-p_{B}}{q_{A}+q_{B}}\right]_{+} \label{eq:lambda_1_corr} \\
    \lambda_{B,A} &=& \frac{1}{(G-1)}\frac{1}{\sum_{g} \relCand(g)}\left[-\left(\frac{p_{A}-p_{B}}{q_{A}+q_{B}}\right)\right]_{+} \label{eq:lambda_2_corr} \\
    \nonumber \vdots \\ \nonumber
    G(G-1) \quad \lambda's 
\end{eqnarray}

We construct $\lambda'_i$ corresponding to constraint \eqref{x_constraint_corr} and $\lambda_k$ corresponding to constraint \eqref{k_constraint_corr} below.

\begin{eqnarray} 
\lambda_k &=& \left[\frac{p_{A}}{\sum_{g}\relCand(g)}-q_{A}\overbrace{\sum_{g \neq A}(\lambda_{A,g}-\lambda_{g,A})}^{(G-1) \text{terms}}\right]    = \left[\frac{p_{B}}{\sum_{g}\relCand(g)}-q_{B}\overbrace{\sum_{g \neq B}(\lambda_{B,g}-\lambda_{g,B})}^{(G-1) \text{terms}}\right] = \cdots \text{for each of } G \text{ groups}\label{eq:lambda_3_corr}\\
\lambda_{i \in g'}' &=& \left[\frac{p_i}{\sum_{g}\relCand(g)}- \lambda_k -q_i \overbrace{\sum_{g \neq g'}(\lambda_{g',g}-\lambda_{g,g'})}^{(G-1) \text{terms}}\right]_{+} 
\label{eq:lambda'_corr}
\end{eqnarray}
For instance, if $i \in A$ then,
\begin{eqnarray} 
\lambda_{i \in A}' &=& \left[\frac{p_i}{\sum_{g}\relCand(g)}- \lambda_k -q_i \sum_{g \neq A}(\lambda_{A,g}-\lambda_{g,A})\right]_{+} \nonumber
\end{eqnarray}

We show that the constructed dual variables are non-negative in Lemma~\ref{lemma:7.2} and always feasible in Lemma~\ref{lemma:7.3}. Additionally, we have $\lambda' = 0$ for any element not selected in the EOR ranking from Lemma~\ref{lemma:7.2}. 

The duality gap can now be formulated as follows

\begin{displaymath}
    g(\lambda^*)-f(X)
    =\delta(\EORanking_k)\sum_{\{A,B\}}(\lambda_{A,B} + \lambda_{B,A}) + k\lambda_k + \sum_{i=1}^{n}\lambda_{i}'  - \frac{P^TX}{\sum_{g}\relCand(g)} 
\end{displaymath}

Substituting the values for $\lambda'$ from \eqref{eq:lambda'_corr} and breaking the $k$ elements selected into $k_{A}$ from group $A$,  $k_{B}$ from group $B$, and so on from every group, we have the above duality gap as

\begin{eqnarray}
\nonumber
& = &\delta(\EORanking_k)\sum_{\{A,B\}} (\lambda_{A,B} + \lambda_{B,A}) + k\lambda_k + \left(  \overbrace{\sum_{i=1}^{k_{A}}\left(\frac{p_i}{\sum_{g}\relCand(g)}-\lambda_k-q_i \sum_{g \neq A}(\lambda_{A,g}-\lambda_{g,A}) \right) + \sum_{}^{k_{B}}\left(.\right) + \cdots}^{\text{$G$ terms, one for each group}}\right) - \frac{P^TX}{\sum_{g}\relCand(g)} 
\end{eqnarray}

In the above $G$ terms, we can collect $\sum_{}^{k_A}\lambda_k + \sum_{}^{k_B}\lambda_k + \cdots = k\lambda_k$ and $\left( \sum_{}^{k_{A}}\frac{p_i}{\sum_{g}\relCand(g)}+ \sum_{}^{k_{B}}\frac{p_j}{\sum_{g}\relCand(g)}+ \cdots \right) = \frac{P^TX}{\sum_{g}\relCand(g)}$. This reduces the duality gap to 
\begin{eqnarray}
    \nonumber
    & = &\delta(\EORanking_k)(\sum_{\{A,B\}} (\lambda_{A,B} + \lambda_{B,A})) - \overbrace{\sum_{i=1}^{k_{A}}q_{i} \sum_{g \neq A}(\lambda_{A,g}-\lambda_{g,A}) - \sum_{j=1}^{k_{B}}q_{j}  \sum_{g \neq B}(\lambda_{B,g}-\lambda_{g,B}) - \cdots }^{\text{$G$ terms}}\\
    \nonumber
    & = &\sum_{\{A,B\}}\left(\delta(\EORanking_k)(\lambda_{A,B} + \lambda_{B,A}) - (\sum_{i=1}^{k_{A}}q_{i} - \sum_{j=1}^{k_{B}}q_{j}) (\lambda_{A,B}-\lambda_{B,A})\right)
\end{eqnarray}

For each pair of groups $A, B$, the term inside the summation reduces to the two group case in Theorem~\ref{Thm:Thm1}. We also have that $-\delta(\EORanking_k) \leq \sum_{}^{k_{A}}q_{i} - \sum_{}^{k_{B}}q_{j} \leq \delta(\EORanking_k)$ from Lemma~\ref{lemma:7.1}. 
\begin{eqnarray}
\nonumber
\text{Duality gap} & \leq & \sum_{\{A,B\}} 2\lambda_{A,B}\delta(\EORanking_k) \\
\nonumber
& \leq & \frac{2\delta(\EORanking_k)}{(G-1)\sum_{g}\relCand(g)}\sum_{\{A,B\}} \left|\frac{p_{A}-p_{B}}{q_{A}+q_{B}}\right| 
\end{eqnarray}

This proves that the EOR solution can only be ever as worse as $\phi\delta_{max}$ when compared with the optimal solution, where \\
$\phi = \frac{2}{(G-1)\sum_{g=1}^G \relCand(g)}\left( \sum_{\{A,B\}} \left|  \frac{p_{A}-p_{B}}{q_{A}+q_{B}}\right| \right)$ and $\delta_{max} = \max_g\left\{ \frac{\PRPRankingGroup{g}[1]}{\relCand(g)}\right\}$ from Lemma~\ref{lemma:7.4}.
\end{proof}

\stepcounter{lemmanum}
\begin{lemma}
\label{lemma:7.1}
\normalfont
EOR ranking is $\delta(\EORanking_{k})$ fairness optimal, implying that for all $G$ choose 2 possible pairs of groups $A$, $B$ $\in \{1, \cdots G\}$, we have $-\delta(\EORanking_k) \leq \sum_{i=i}^{k_{A}}q_i - \sum_{j=1}^{k_{B}}q_j \leq \delta(\EORanking_k)$.

This lemma follows directly from the EOR ranking principle of choosing the candidate that minimizes $\delta(\EORanking_k)$ defined according to Eq.~\eqref{eq:delta_multiple_groups}. 

\qed
\end{lemma}

\begin{lemma}
\label{lemma:7.2}
\normalfont
The constructed dual variables $\lambda \geq 0$. In particular, for any $i > k_g $ in group g, where $g \in \{1, \cdots G\}$, it holds that $\lambda_i' = 0$ and for any $i \leq k_g $ it holds that $\lambda_i' \geq 0$.
\end{lemma}

\begin{proof} \normalfont \label{proof:lemma_2_corr}
In this Lemma, we show that $\lambda'=0$ for the elements not selected and $\lambda' \geq 0$ for the selected elements by the EOR Algorithm. Without loss of generality, we consider the element at index $i$ that belongs to group $A$. 
\begin{eqnarray}
    \nonumber
    \lambda_{i \in A}' &=& \left[\frac{p_i}{\sum_{g}\relCand(g)}- \lambda_k -q_i \sum_{g \neq A}(\lambda_{A,g}-\lambda_{g,A})\right]_{+} \\
    \nonumber
    &=& \left[\frac{p_i}{\sum_{g}\relCand(g)} - \frac{p_{A}}{\sum_{g}\relCand(g)} +q_{A} \sum_{g \neq A}(\lambda_{A,g}-\lambda_{g,A}) -q_i \sum_{g \neq A}(\lambda_{A,g}-\lambda_{g,A}) \right]_{+} \\
    &=& \left[\frac{p_i- p_{A}}{\sum_{g}\relCand(g)} + (q_i-q_{A})\sum_{g \neq A}(\lambda_{g,A}-\lambda_{A,g}) \right]_{+}\nonumber\\
    &=& \left[\sum_{g \neq A}\left(\frac{p_i- p_{A}}{(G-1)\sum_{g}\relCand(g)} + (q_i-q_{A})(\lambda_{g,A}-\lambda_{A,g}) \right)\right]_{+} \label{eq:lambda_2_ge_0_corr}
\end{eqnarray}

For every pair of $\lambda_{A,g}$ and $\lambda_{g,A}$, where $g \in \{1, \cdots G\}$ and $g\neq A$, only one of $\lambda_{A,g}, \lambda_{g,A}$ is $\geq 0$. Each of the $G-1$ terms inside the summation in Eq.~\eqref{eq:lambda_2_ge_0_corr} reduces to the two group case as follows. 
For $i > k_A$ and each \{$A, g$\}, the term evaluates to $\leq 0$ using Lemma~\ref{lemma:6.2} and thus $\lambda_i'$ is clipped to $0$. Similarly, for $i \leq k_A$ and each \{$A, g$\} the term evaluates to $\geq 0$ and thus $\lambda_i' \geq 0$.

We have shown that for any element not selected by EOR Algorithm the corresponding dual variable $\lambda' = 0$, and for any element selected by the EOR Algorithm the corresponding dual variable $\lambda' \geq 0$.

We now show that $\lambda_k \geq 0$. From Eq.~\eqref{eq:lambda_3_corr},
\begin{eqnarray} 
\lambda_k &=& \left[\frac{p_{A}}{\sum_{g}\relCand(g)}-q_{A}\sum_{g \neq A}(\lambda_{A,g}-\lambda_{g,A})\right] \nonumber\\
&=&\sum_{g \neq A}\left(\frac{p_{A}}{(G-1)\sum_{g}\relCand(g)} + q_{A}(\lambda_{g,A}-\lambda_{A,g}) \right)\label{eq:lambda_k_A_corr}
\end{eqnarray}

Each of the $G-1$ terms inside the summation in Eq.~\eqref{eq:lambda_k_A_corr} reduces to the two group case. For each \{$A, g$\}, the term evaluates to $\geq 0$ using Lemma~\ref{lemma:6.2} and thus $\lambda_k \geq 0$.

The $G(G-1)$ duals $\lambda_{A,B}$ are $\geq 0$ by their construction in \eqref{eq:lambda_1_corr}. Thus, we have shown that all the constructed dual variables $\lambda \geq 0$.
\end{proof}

\begin{lemma}
    \label{lemma:7.3}
    \normalfont
    The dual variables $\lambda=[\lambda_1' \cdots \lambda_n', \lambda_k, \lambda_{A,B}, \lambda_{B,A}, \cdots] $ are always feasible.
\end{lemma}

\begin{proof} \label{proof:lemma_3_corr}

For some element $i \in A$, the duality constraint implies that
\begin{equation} 
    q_i \left(\overbrace{ \sum_{g}^{} (\lambda_{A,g}-\lambda_{g,A})}^{G-1 \text{ terms}}\right) + \lambda_k + \lambda_i'  \geq \frac{p_i}{\sum_{g}\relCand(g)} \label{eq:duality_proof_corr}
\end{equation}

Without loss of generality, we consider element $i \in A$.

\textbf{Case I:} Elements not selected by the EOR Algorithm.

Using the fact that $\lambda_i'=0$ for $i > k_A$ from Lemma~\ref{lemma:7.2}, and substituting $\lambda_k$ from Eq.~\eqref{eq:lambda_3_corr},  we get
\begin{eqnarray}
    \nonumber
     q_i \sum_{g \neq A} (\lambda_{A,g}-\lambda_{g,A}) + \lambda_k + \lambda_i' 
     &=& q_i \sum_{g \neq A}(\lambda_{A,g}-\lambda_{g,A}) + \frac{p_{A}}{\sum_{g}\relCand(g)}-q_{A} \sum_{g \neq A} (\lambda_{A,g}-\lambda_{g,A})\\
     \nonumber
     &=& \frac{p_{A}}{\sum_{g}\relCand(g)} + (q_i-q_A) \sum_{g \neq A}(\lambda_{A,g}-\lambda_{g,A}) \\
     &=& \sum_{g \neq A}\left(\frac{p_{A}}{(G-1)\sum_{g} \relCand(g)} + (q_i-q_A) (\lambda_{A,g}-\lambda_{g,A})\right) \label{eq:lambda_3_ge_0_corr}\\
     &\geq& \sum_{g \neq A}\frac{p_{i}}{(G-1)\sum_{g} \relCand(g)} \geq  \frac{p_{i}}{\sum_{g} \relCand(g)} \label{eq:lambda_3_ge_0_corr_redn}
\end{eqnarray}
Each of the $G-1$ terms inside the summation in Eq.~\eqref{eq:lambda_3_ge_0_corr} reduces to the two group case. For each \{$A, g$\}, the term evaluates to $\frac{p_{i}}{(G-1)\sum_{g} \relCand(g)}$ using Lemma~\ref{lemma:6.3} and thus the corresponding duality constraint is satisfied.

\textbf{Case II:} Elements selected by the EOR Algorithm.

Using the fact that $\lambda_i' \geq 0$ for $i \leq k_A$ from Lemma~\ref{lemma:7.2}, and substituting $\lambda_k$ from Eq.~\eqref{eq:lambda_3_corr}, $\lambda_i'$ for $i \leq k_A$ in \eqref{eq:lambda'_corr}, we get 
\begin{eqnarray}
    \nonumber
     q_i \sum_{g \neq A} (\lambda_{A,g}-\lambda_{g,A}) + \lambda_k + \lambda_i' 
     &=& q_i \sum_{g \neq A}(\lambda_{A,g}-\lambda_{g,A}) + \lambda_k + \frac{p_{A}}{\sum_{g}\relCand(g)} -\lambda_k - q_{i} \sum_{g \neq A} (\lambda_{A,g}-\lambda_{g,A})\\
     \nonumber
     &=& \frac{p_{A}}{\sum_{g}\relCand(g)} \geq  \frac{p_{i}}{\sum_{g}\relCand(g)}
\end{eqnarray}
Thus, for elements selected by the EOR Algorithm i.e. $i \leq k_A$, the corresponding dual constraint is satisfied.
\end{proof}

We now present the proof for the global a priori bound on $\delta(\EORanking_{k})$ for $G$ groups.

\begin{lemma}
    \label{lemma:7.4}
    \normalfont
    The global a priori bound on $\delta(\EORanking_{k})$ for $G$ groups is given by $\delta_{max} = \max_{g} \left\{\frac{\PRPRankingGroup{g}[1]}{\relCand(g)}\right\}$
\end{lemma}

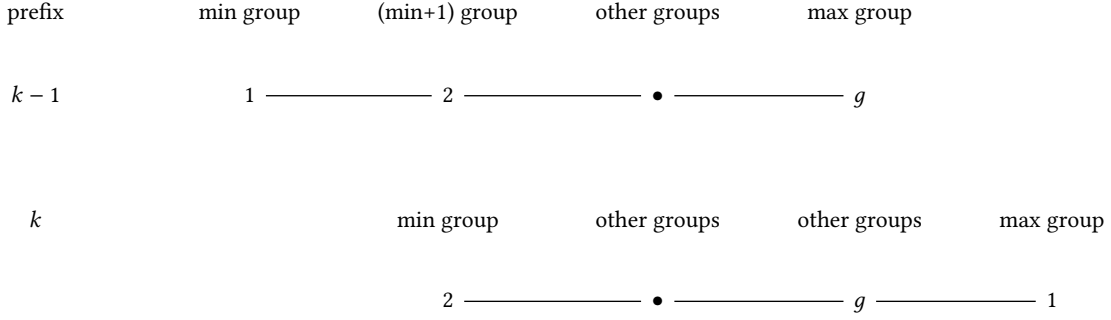
\begin{figure*}[!hbtp]
\centering
\begin{tikzcd}[ampersand replacement=\&]
	{\text{prefix}} \&\& {\text{min group}} \& {\text{(min+1) group} } \& {\text{other groups}} \& {\text{max group}} \\
	{k-1} \&\& 1 \& 2 \& \bullet \& g \\
	\\
	k \&\&\& {\text{min group}} \& {\text{other groups}} \& {\text{other groups}} \& {\text{max group}} \\
	\&\&\& 2 \& \bullet \& g \& 1
	\arrow[no head, from=2-3, to=2-4]
	\arrow[no head, from=2-4, to=2-5]
	\arrow[no head, from=2-5, to=2-6]
	\arrow[no head, from=5-4, to=5-5]
	\arrow[no head, from=5-5, to=5-6]
	\arrow[no head, from=5-6, to=5-7]
\end{tikzcd}
\caption{Illustration for the case of Multiple groups}
\label{fig:multiple_groups_proof}
\end{figure*}

\begin{proof}
    We will show that for $G$ groups, the value of $\delta_{max}$ such that a feasible ranking will be provided and that always satisfies $\delta(\EORanking_k) \leq \delta_{max}$ for every given $k$ is given by 
    \begin{align}
        \delta_{max}=\max \left(\frac{\PRPRankingGroup{1}[1]}{\relCand(1)}, \frac{\PRPRankingGroup{2}[1]}{\relCand(2)}, \cdots,\frac{\PRPRankingGroup{g}[1]}{\relCand(g)} \cdots \frac{\PRPRankingGroup{G}[1]}{\relCand(G)} \right)
    \end{align}

    In the remaining, we drop the superscript of EOR for simplicity and $\sigma_j$ refers to $\EORanking_j$.
    
    We argue by an inductive argument similar to the proof of Theorem~\ref{thm:thm_bound}.
    Consider the base case of $k=1$, when the first element is to be selected. 
    The EOR algorithm will select according to Eq.~\eqref{eq:delta_multiple_groups}
    resulting in the lower $\delta(\sigma_{k=1})$. Thus, $\delta(\sigma_{k=1})$ is clearly $\leq \delta_{max}$.

    We assume that for a given $k-1$, $\delta(\sigma_{k-1}) \leq \delta_{max}$ and show that at $k$, $\delta(\sigma_{k}) \leq \delta_{max}$.
    
    Consider the general case as depicted in Figure~\ref{fig:multiple_groups_proof}, where a group $1$ has the lowest accumulated proportion and group $g$ has the highest at prefix $k-1$. Since $\delta(\sigma_{k-1}) \leq \delta_{max}$ from inductive assumption, we have 
     \begin{eqnarray}
        \nonumber
         \frac{\relCand(g|\sigma_{k-1})}{\relCand(g)}-\frac{\relCand(1|\sigma_{k-1})}{\relCand(1)} \leq \delta_{max}  \label{eq:eor_extended_satisfy}
    \end{eqnarray}

    At the next prefix $k$, if the group that is selected has $\frac{\relCand(.|\sigma_{k})}{\relCand(.)} \leq \frac{\relCand(g|\sigma_{k-1})}{\relCand(g)}$, then $\delta(\sigma_{k}) \leq \delta_{max}$. Note that $\delta(\sigma_{k})$ is always non-negative by definition from Eq.~\eqref{eq:delta_multiple_groups}.

    We now consider the case when a group $g'$ is selected at the next prefix $k$ such that $\frac{\relCand(g'|\sigma_{k})}{\relCand(g')} > \frac{\relCand(g|\sigma_{k-1})}{\relCand(g)}$. Let us first consider that $g'$ is group $1$. We have $\frac{\relCand(1|\sigma_{k})}{\relCand(1)} > \frac{\relCand(g|\sigma_{k})}{\relCand(g)}$.
    Selecting group $1$ at $k$ means that the rest of the groups have the same accumulated relevance proportion $\frac{\relCand(.|\sigma_{k})}{\relCand(.)}$ at prefix $k$ as $k-1$.
    We analyze the difference of $\frac{\relCand(.|\sigma_{k})}{\relCand(.)}$ between the group that was most behind-- group $1$ and the group that was second most behind -- group $2$ and whether that remains within $\delta_{max}$. 
    If the added element from group $1$ is denoted by $p_i$, the EOR constraint value at $k$ is
    \begin{eqnarray}
        \delta(\sigma_{k}) & = & \frac{\relCand(1|\sigma_{k-1})}{\relCand(1)}+\frac{p_i}{\relCand(1)}- \frac{\relCand(2|\sigma_{k})}{\relCand(2)} \label{eq:eor_extended_delta_k} \\
        & = & \frac{p_i}{\relCand(1)} - \left(\frac{\relCand(2|\sigma_{k})}{\relCand(2)} -\frac{\relCand(1|\sigma_{k-1})}{\relCand(1)} \right) =  \frac{p_i}{\relCand(1)} - \left(\frac{\relCand(2|\sigma_{k-1})}{\relCand(2)} -\frac{\relCand(1|\sigma_{k-1})}{\relCand(1)} \right) \nonumber \label{eq:eor_extended_delta_k_rearranged}
    \end{eqnarray}

    Eq.~\eqref{eq:eor_extended_delta_k} holds since group $1$ is now the group with maximum relevance proportion after adding $p_i$ - the top most current element from group $1$. Group $2$ becomes the group with minimum relevance proportion. 

    Since $\frac{p_i}{\relCand(1)} \leq \frac{\PRPRankingGroup{1}[1]}{\relCand(1)} \leq \delta_{max}$ and because group $1$ was behind group $2$ at prefix $k-1$, we have $\frac{\relCand(2|\sigma_{k-1})}{\relCand(2)} \geq \frac{\relCand(1|\sigma_{k-1})}{\relCand(1)}$ since . As a result,
    \begin{eqnarray}
        \nonumber
        \delta(\sigma_{k}) \leq & \frac{p_i}{\relCand(1)} \leq \frac{\PRPRankingGroup{1}[1]}{\relCand(1)} \leq \delta_{max} 
    \end{eqnarray}
    We have shown above that if the group with lowest relevance proportion at prefix $k-1$ (group $1$ in this case) is selected and its relevance proportion now exceeds the group with the highest relevance proportion at prefix $k-1$ (group $g$ in the case above), then $\delta(\sigma_{k}) \leq \delta_{max}$. Thus, we can say that at least one group exists that satisfies $\delta_{max}$ EOR constraint at prefix $k$. 
    This completes the proof that the EOR algorithm always provides a feasible ranking that satisfies 
    $\delta_{max}=\max_{g \in \{ 1 \cdots G\}} \left\{ \frac{\PRPRankingGroup{g}[1]}{\relCand(g)}\right\}$ for $G$ groups.
\end{proof}

\section{Experiment Details} \label{exp_details}

\subsection{Baselines} \label{baselines}

We compare rankings from Algorithm~\ref{alg:EOR_alg} with the following baselines
\paragraph{Probability Ranking Principle ($\PRP$)} Candidates are selected in decreasing order of relevance independent of their group membership.

\paragraph{Uniform Policy ($\Uniform$)} 
Candidates are selected randomly independent of their group membership or relevance.

\paragraph{Thompson Sampling Ranking Policy ($\TS$)  \citep{NEURIPS2021_63c3ddcc}} 
For $\TS$, binary relevances are drawn according to $r_i \sim \prob(r_i|\Data)$, and candidates are sorted in decreasing order of relevance $r_i$ with their ranking randomized for the same value of relevance $r_i$. 
\[ \TS \sim \arg \text{sort}_{i} [r_i] \quad \textit{s.t} \quad  r_i \sim \prob(r_i|\Data)\] 
$\TS$ ranks each candidate $i$ in position $k$ with probability that $i$ has $k^{th}$ highest relevance.

For both $\TS$ and $\Uniform$, we compute expectation over $100$ rankings $\UniformRanking \sim \Uniform$ or $\TSRanking \sim \TS$ respectively and compute $\delta(\sigma_k)$ used in \Cref{tab:vary_disp_unc} as
\begin{equation}
    \delta(\sigma_k) = \expect_{\sigma \sim \pi}\left[\max_g  \left\{ \frac{\relCand(g|\sigma_k)}{\relCand(g)} \right\}  - \min_g \left\{ \frac{\relCand(g|\sigma_k)}{\relCand(g)}\right\} \right] \nonumber
\end{equation}

In order to plot a single ranking $\UniformRanking$, $\TSRanking$ for all experiments, we select the ranking with median $\sum_{k=1}^{n}|\delta(\sigma_k)|$

\paragraph{Demographic Parity ($\DP$)} Candidates in each group are sorted in decreasing order of $\prob(r_i|\Data)$ and selected such that the following constraint is minimized. This constraint is similar to the statistical parity variations introduced in \cite{DBLP:journals/corr/YangS16a}.

\begin{equation}
    \forall k \quad \frac{S(A|\sigma_k) }{S(A)}-\frac{S(B|\sigma_k)}{S(B)} \label{eq:demographic_parity}
\end{equation}
where $S(.)$ represents the size of the group. For a fair comparison with $\EORanking$, we use Algorithm~\ref{alg:EOR_alg} and instead of minimizing Eq.~\eqref{eq:EOR}, we minimize the above demographic parity constraint \eqref{eq:demographic_parity}.
We now discuss other variations of proportional representation constraints that have been introduced in prior literature \cite{Celis2017RankingWF, 10.1145/3351095.3372858, 10.1145/3442188.3445930}. Generally, these constraints require that the disadvantaged group selected is at least a specific proportion $\alpha$ of top k.
\begin{eqnarray}
S(B|\sigma_k) \geq \alpha k \label{proportional_constraints}
\end{eqnarray}
where $\alpha=\frac{S(B)}{S(A)+S(B)}$ and Eq.~\eqref{proportional_constraints} is used as the fairness constraint while maximizing the utility to the principal.
This type of representational constraint by definition requires the designation of a disadvantaged group. By designating B as the disadvantaged group, the constraint for proportional Rooney-Rule policy \cite{shen2023fairness}, which we denote by $\PRR$ is as follows 
\begin{equation}
\nonumber
    \forall k \quad \frac{S(B|\sigma_k) }{k} \geq \frac{S(B)}{S(A)+S(B)} \label{eq:proportional_constraint}
\end{equation}
We empirically compare $\PRR$ baseline with other ranking policies in Figure~\ref{fig:DP_variations} and as expected, find that it is similar to the baseline of $\DP$, where $\PRR$ and $\DP$ almost overlap.
Thus for a fair and analogous comparison with $\EOR$, we use \eqref{eq:demographic_parity} as the $\DP$ baseline for all empirical evaluations.
\begin{figure*}[t!]
    \centering
    \begin{subfigure}[t]{1.0\linewidth}
        \centering
        \includegraphics[width=1.0\textwidth]{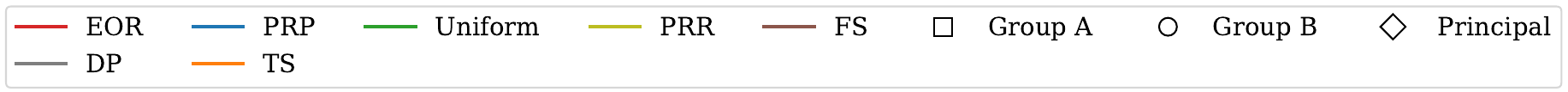}
        \label{fig:DP_variations_legend}
    \end{subfigure}
    \begin{subfigure}[t]{1.0\linewidth}
        \centering
        \includegraphics[width=1.0\textwidth]{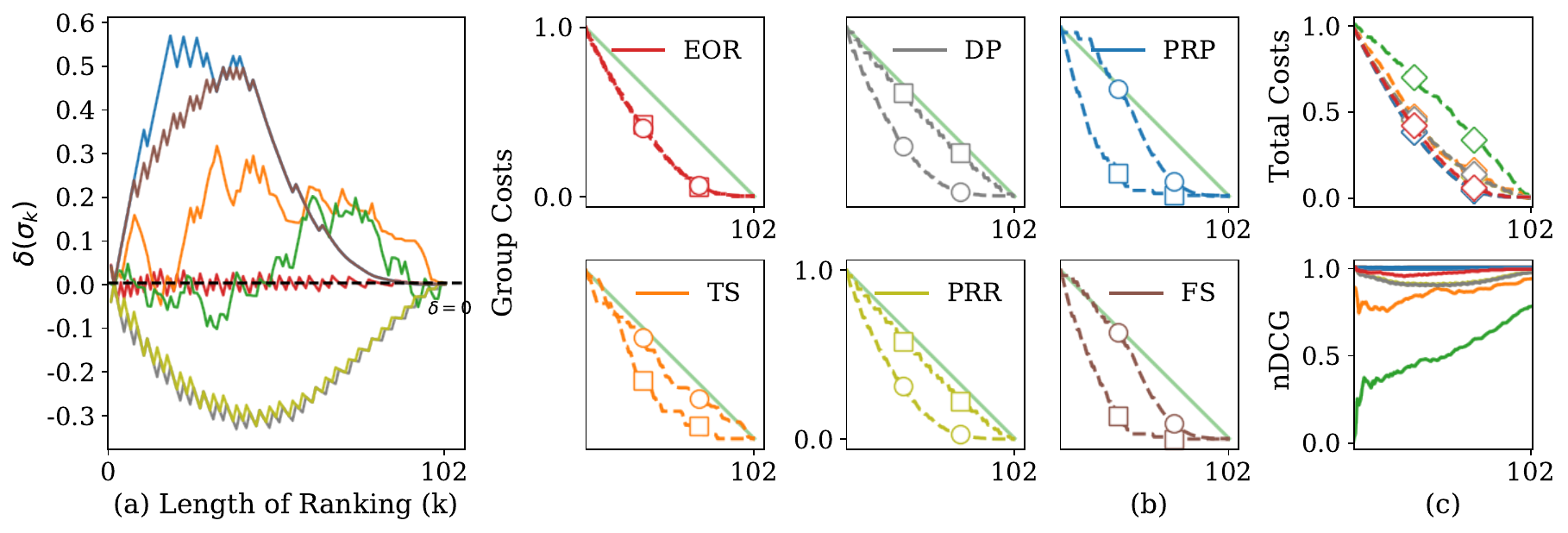}
   \end{subfigure}
  \caption{\normalfont EOR criterion $\delta(\sigma_k)$, costs of the ranking policies, and DCG Utility for Synthetic dataset with proportional Rooney-Rule like constraint, $\PRR$. For group A we draw $S(A)=30$ relevance probabilities from $Powerlaw(\eta=5)$, and then draw for group B from $Powerlaw(\eta=0.5)$ until $\relCand(A) \approx \relCand(B)$.}
  \label{fig:DP_variations}
\end{figure*}
For more than two groups, we extend the DP baseline with the selection rule based on group size as follows. In particular,
\begin{eqnarray}
    \nonumber
     \delta(\DPRanking) &=& \max_g \left\{\frac{S(g|\sigma_k)}{S(g)}\right\} - \min_g \left\{\frac{S(g|\sigma_k)}{S(g)}\right\} \label{eq:DP_multiple_groups} \\
    \nonumber
    l_g &=& \PRPRankingGroup{g}[1] \hspace{10pt} \forall g \in \{1 \cdot G\} \nonumber \\
    g^{*} &=& \argmin_{g \in [1..G]} \delta(\DPRanking \cup \{l_g\}) ; \hspace{5pt}   l_{g^{*}} = \PRPRankingGroup{g^{*}}[1]\label{eq:DP_l_multiple_g}
\end{eqnarray}

\paragraph{FA$^{*}$IR  Ranking Principle ($\FS$)} This criterion is anchored on the principle that a top-k ranking is fair when the proportion of disadvantaged candidates selected doesn't fall far below a required minimum proportion $p$. This is formalized with a Binomial distribution, and a confidence level ($1-\alpha$).
A function of the binomial cdf is computed apriori and is used as an input in the FA$^{*}$IR Algorithm.
Since Binomial(p=0.5,n) corresponds to a ranking where at each position, a candidate from either group is selected randomly, FA$^{*}$IR is a "softened" version of demographic parity (DP). As a result, FA$^{*}$IR is fundamentally different from Axiom~\ref{axm:axm_uniform} and Definition~\ref{def:eor_constraint} derived from the uniform lottery fairness because, unlike DP, the uniform lottery is anchored on selecting an equal fraction of relevance from each group.
Unlike $\EOR$, $\FS$ is oblivious to the relevance distribution and thus cannot take disparate uncertainty into account. FA$^{*}$IR also requires the normative designation of a disadvantaged group.

Consider the following example for top k=4 selection, with the probability of relevance for 
group A = [0.7, 0.7, 0.7, 0.7, 0.1, 0.1], group size = 6, relevant candidates = 3.0. Similarly, the probability of relevance for
group B = [0.5, 0.5, 0.5, 0.5, 0.5, 0.5], group size = 6, relevant candidates = 3.0.
The EOR Ranking for top-4 is [0.5, 0.7, 0.5, 0.7] with 2 candidates from group A, and 2 from group B, resulting in $\delta(\sigma_4^{EOR})=0.13$. 
The $\FS$ Algorithm with Binomial(p=0.5, n=12), k=4 and $\alpha=0.1$ requires that at least 1 candidate be selected from the disadvantaged group while maximizing the utility to the principal. 
FA$^{*}$IR ranking with group B as the disadvantaged group is $\sigma_4^{FS}=[0.7, 0.7, 0.7, 0.5]$. It selects 3 candidates from group A, and 1 from group B, resulting in $\delta(\sigma_4)=0.53$.
If instead group A is designated as the disadvantaged group, $\sigma^{FS}=[0.7, 0.7, 0.7, 0.7]$ with all candidates selected from group A, and none from group B, resulting in $\delta(\sigma_4^{FS})=0.93$. 
Note that for both FA*IR rankings, far fewer relevant candidates are chosen from group B, even though both groups have an equal number of relevant candidates in expectation.

In all the empirical evaluations in this paper, we assign group B as the minority group for $\FS$ and use the fairsearch core library \footnote{\url{https://github.com/fair-search/fairsearch-fair-python}} with default parameters of $\alpha=0.1$.

Next, we discuss two exposure-based formulations $\EXP$ and $\RA$. 
\paragraph{Exposure-based Disparate Treatment ($\EXP$)}
This policy enforces that the allocation of exposure to each group is proportional to their average utility. Specifically for two groups A and B,
\[ \frac{\text{Exposure}(A|\Sigma)}{U(A)} =  \frac{\text{Exposure}(B|\Sigma)}{U(B)} \]
where $\Sigma$ is the doubly stochastic ranking matrix obtained from solving the Linear Program in \citep{10.1145/3219819.3220088}. For multiple groups, the above constraint is added for each pair of groups. $\text{Exposure}(g|\Sigma) = \frac{\Sigma_{i,j}v_j}{S(g)}$, $v_j = \frac{1}{\log{(j+1)}}$ for the $j^{th}$ position, and $U(g)=\frac{\sum_{i \in g}p_i}{S(g)}=\frac{\relCand(g)}{S(g)}$. In particular for two groups A, B, we solve the following LP \citep{10.1145/3219819.3220088}

\begin{eqnarray}
    \text{Maximize}  & P^{T} \Sigma v  \label{eq:lp_exp}& \hspace{10pt} \text{utility to the principal}\\
    \text{subject to} & \basis^{T} \Sigma = \basis^{T} &   \hspace{10pt} \text{(sum of probabilities for each position)}    \nonumber \\
       & \Sigma \basis = \basis&   \hspace{10pt} \text{(sum of probabilities for each candidate)}    \nonumber \\
      &  0 \leq \Sigma_{i,j} \leq 1 &   \hspace{10pt} \text{(valid probability)}    \nonumber \\
    &  \left( \frac{\indicator_{i \in A}}{\relCand(A)} - \frac{\indicator_{j \in B}}{\relCand(B)}\right)\Sigma v = 0& \hspace{10pt} \text{(exposure constraint)}    \nonumber 
\end{eqnarray}

The group cost is computed as $\frac{\indicator_g P\Sigma}{\relCand(g)}$, total cost as $\frac{P\Sigma}{\sum_g \relCand(g)}$ and EOR criterion as $\max_g \left\{ \frac{\indicator_g P\Sigma}{\relCand(g)}\right\} - \min_g \left\{ \frac{\indicator_g P\Sigma}{\relCand(g)}\right\}$
\paragraph{Rank Aggregation w. proportional allocation of Exposure}
For $\RA$, we modify the baseline for fair rank aggregation in \citep{10.1145/3593013.3594085} as follows. In fair rank aggregation, all $n$ candidates are ranked by $m$ voters to achieve a ranking with maximum consensus accuracy, where consensus may be according to different aggregation methods while achieving fairness of exposure w.r.t groups. 
\citep{10.1145/3593013.3594085} proposes an algorithm that finds the consensus maximizing ranking and then swaps the candidates such that the equality of exposure is satisfied in that ranking. To adapt this baseline, we use the ranking from utility maximizing $\PRP$ as the consensus ranking and use the algorithm from \citep{10.1145/3593013.3594085} to swap elements in PRP ranking until the exposure constraint below is satisfied,
\[ \frac{\min_g{Exposure}(g)}{\max_g{Exposure}(g)} \geq \text{threshold} \]
A threshold of $0.95$ is used in experiments and on average over 100 runs, an exposure of $ 0.96\pm{0.01}, 0.96\pm{0.00}, 0.97\pm{0.00}$ is achieved for high, medium, and low levels of disparate uncertainty respectively in \Cref{tab:vary_disp_unc}.

\subsection{Synthetic Dataset}\label{synthetic_extension}
To simulate disparate uncertainty between groups, we draw 
$\prob(r_i|\Data)$ directly from specific probability distributions as follows.
For Group A, we obtain $p_i \sim Beta(\frac{1}{20}, \frac{1}{20})$ and keep them fixed. We simulate $100$ runs and in each run, $p_i$ for group B are sampled as follows until $\relCand(B) \approx \relCand(A)$ (total expected relevance for groups can only differ by 1.0). 
\begin{itemize}
    \item High Disparate Uncertainty: $Beta(5,5)$
    \item Medium Disparate Uncertainty: $Beta(\frac{1}{2}, \frac{1}{2})$
    \item Low Disparate Uncertainty: $Beta(\frac{1}{20}, \frac{1}{20})$. Note that even when both groups are drawn from the same distribution, any sampled instance still contains some amount of disparate uncertainty.
\end{itemize}

Results for unfairness and effectiveness of rankings are reported with standard error in \Cref{tab:vary_disp_unc} (left). The posterior distributions in \Cref{tab:vary_disp_unc} (right) uses 50 samples for each candidate in group A, while for group B, the number of samples increases from 10 to 30 to 50 as the setting changes from high to medium to low disparate uncertainty respectively.

To estimate $\prob(i \in \sigmaKPi)$ for stochastic policies-- $\Uniform$ and $\TS$, we draw $d=10^3$ Monte Carlo samples and compute Monte Carlo estimate according to \eqref{eq:stochastic_policy}. 
\begin{align}
    1-\prob(i \in \sigmaKPi) = \frac{1}{d} \sum_{d}{\indicator_{i \notin {\sigma_k}}} \label{eq:stochastic_policy}
\end{align} 
We compute the costs using $\prob(r_i|\Data), \prob(i \in \sigmaKPi)$ according to Eqs.~\eqref{eq:subgroup_cost}, and \eqref{eq:total_cost}.

In Figure~\ref{fig:high_medium_low_disp_unc}, we plot a random sample from \Cref{tab:vary_disp_unc} according to the generation process described above. 
We also qualitatively analyze a commonly used measure of utility to the principal, namely, the expected Normalized Discounted Cumulative Gain (nDCG), which according to our model is, 
\begin{equation}
    nDCG(\sigma_k) = \frac{DCG(\sigma_k)}{iDCG} = \frac{\sum_{i\in \sigma_k} v_i r_i}{\sum_{i\in \sigma^{Ideal}_k} v_i r_i}; \hspace{20pt} \sigma^{Ideal} = \arg \text{sort}_{i} r_i \nonumber
\end{equation}
where $v_i = log_2\frac{1}{(1+i)}$ for the $i^{th}$ position. When true relevance labels are known, for instance in US Census experiments in Figure~\ref{fig:USCensus_NY_2_groups}, $r_i \in \{0,1\}$ consists of the true relevance labels, otherwise in synthetic experiments in Figure~\ref{fig:high_medium_low_disp_unc}, $r_i \in [0,1]$ consists of the calibrated $\prob(r_i=1|\Data)$.

As shown in Figure~\ref{fig:nDCG_high_disp_unc}, the nDCG for EOR ranking is only slightly lower than the nDCG optimal PRP ranking and competitive with all other ranking policies.
 \begin{figure}[h]
    \centering
    \includegraphics[width=0.4\textwidth]{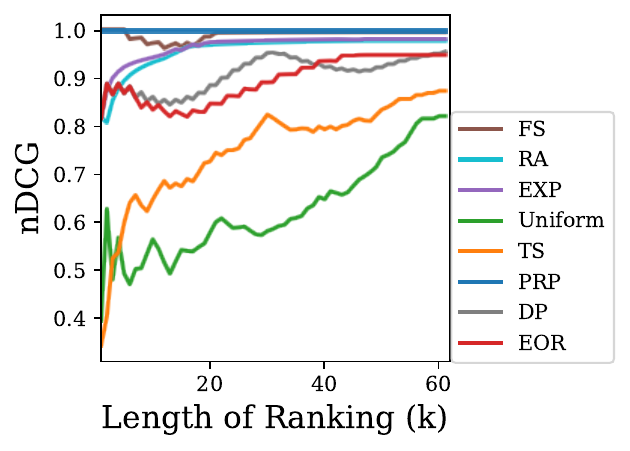}
      \caption[]{nDCG for High disparate uncertainty setting shown in Figure~\ref{fig:synthetic_main_disp_unc}}\label{fig:nDCG_high_disp_unc}
  \end{figure}
In all of these experiments, we confirm our findings that $\EOR, \Uniform$ distribute the subgroup and total costs evenly while other ranking policies $\PRP, \DP$, and $\TS$ place a high cost burden on one of the groups. Further, for $\EOR$, the total cost to the principal and nDCG utility is close to the optimal (but unfair) total cost and utility of $\PRP$, indicated by overlapping lines in subplots (c) of Figure~\ref{fig:high_medium_low_disp_unc}.
\begin{figure}[t!]
\begin{subfigure}[t]{1.0\linewidth}
    \hspace*{\fill}
    \includegraphics[width=0.95\textwidth] {Figures/synthetic/synthetic_legend.pdf}
    \label{fig:synth_legend}
  \end{subfigure}
  \bigskip
  \begin{subfigure}[t]{1.0\linewidth}
    \centering
    \includegraphics[width=1.0\textwidth]{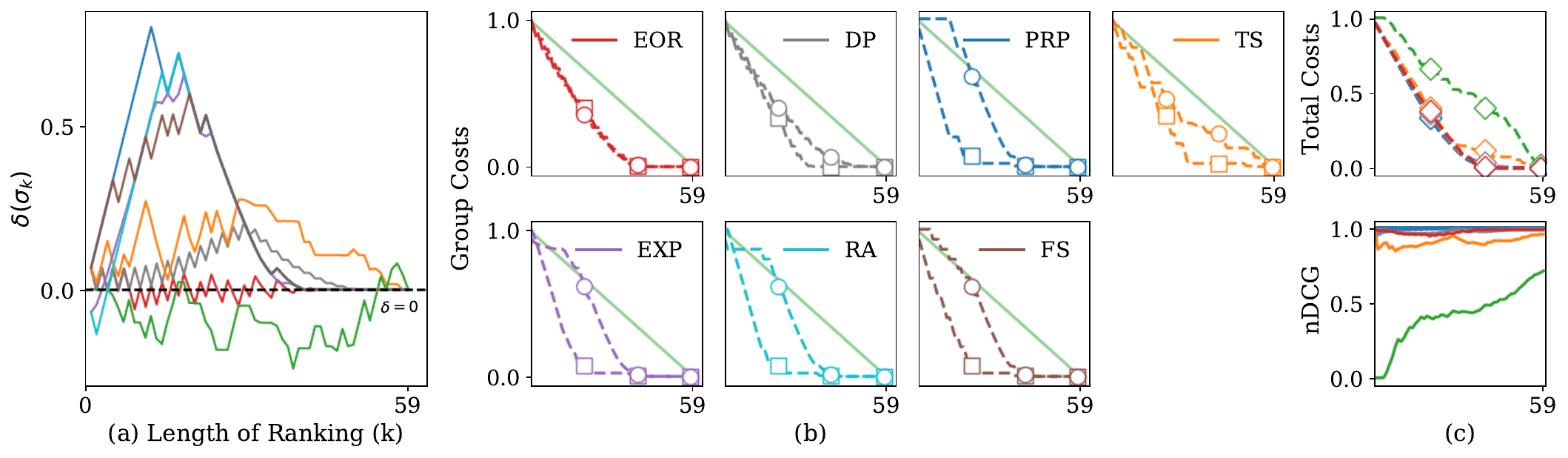}
    \phantomcaption{Medium Disparate Uncertainty}
    \label{fig:medium_disp_unc}
  \end{subfigure}
  \begin{subfigure}[t]{1.0\linewidth}
    \centering
    \includegraphics[width=1.0\textwidth]{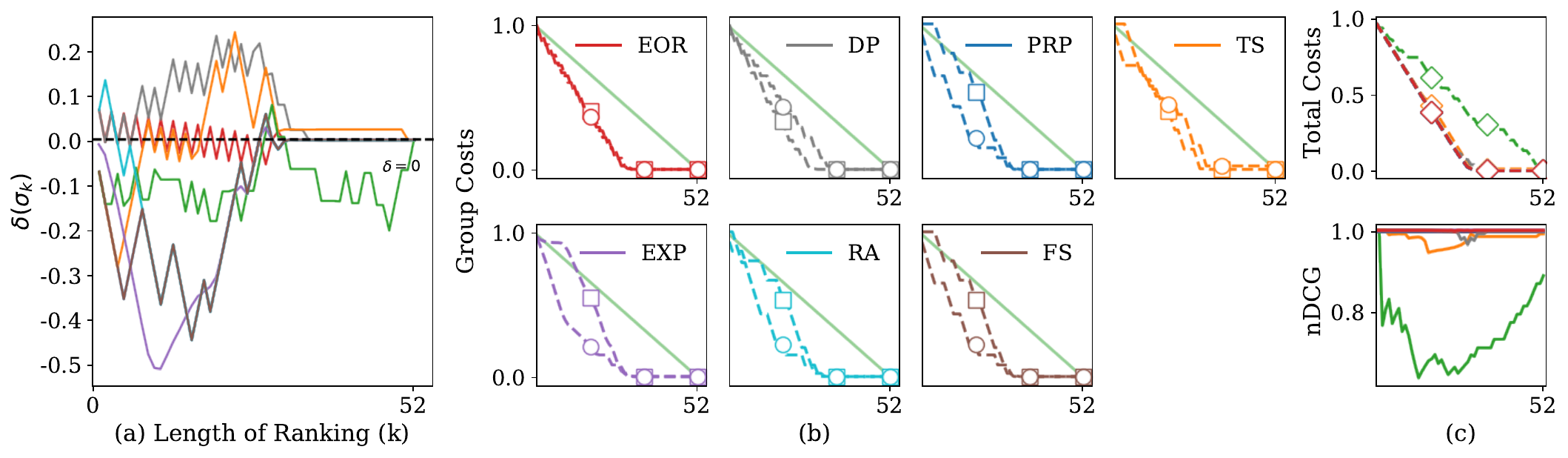}
    \phantomcaption{Low Disparate Uncertainty}
    \label{fig:low_disp_unc}
  \end{subfigure}
  \caption{\textbf{Top}: Medium disparate uncertainty \textbf{Bottom} Low disparate uncertainty for a randomly sampled instance.}
  \label{fig:high_medium_low_disp_unc}
\end{figure}

\subsection{US Census Survey Dataset} \label{USCensus_details}
We use the ACSIncome task with default settings \cite{NEURIPS2021_32e54441} for the state of New York and Alabama for 2018, with 1-year horizon. The dataset consists of 10 features, out of which 8 are categorical. Race is among the features that we include in the prediction task following \cite{NEURIPS2021_32e54441}. There are 103,021 records for New York and 22,268 records for Alabama. For pre-processing, the categorical features are one-hot encoded, while the other two numerical features (`AGE' and `WKHP') are standardized to have mean $0$ and standard deviation $1$. 
We divide this dataset into 60/20/20 for train/val/test split and fit a Gradient Boosting Classifier \footnote{\href{https://scikit-learn.org/stable/modules/generated/sklearn.ensemble.GradientBoostingClassifier.html}{scikit-learn Gradient Boosting Classifier}} with the parameters loss as `exponential' and max\_depth as 5 following hyperparameter configuration of \cite{NEURIPS2021_32e54441}. This gives a DP violation $P(\hat{Y}=1|White) - P(\hat{Y}=1|Black)$ of $0.19$ and an EO violation $P(\hat{Y}=1|Y=1, White) - P(\hat{Y}=1|Y=1, Black)$ of $0.18$ for New York and a a DP violation of $0.22$, EO violation of $0.29$ for Alabama, which is roughly similar to Figure 2 and 6 of \cite{NEURIPS2021_32e54441} before any fairness interventions are applied in the classification setting.

We subset the dataset to contain records with White or Black/African American racial membership (Alabama and New York) and subset records with White, Black, Asian, and Others racial membership (New York only) for two and four groups respectively. To calibrate relevance probabilities, we fit a Platt Scaling \cite{PlattProbabilisticOutputs1999} calibrator on the validation data split group-wise and apply Platt Scaling to the test set probability estimates. Figure~\ref{fig:USCensus_calibration_NY}, \ref{fig:USCensus_calibration_AL} and \ref{fig:USCensus_calibration_4_groups} show that calibrated $\prob(r_i|\Data)$ on the test set, binned across 20 equal sized bins, lie close to the perfectly calibrated line. 

\begin{figure*}[t!]
    \centering
        \begin{subfigure}[t]{0.3\linewidth}
            \centering
            \includegraphics[width=1.0\textwidth]{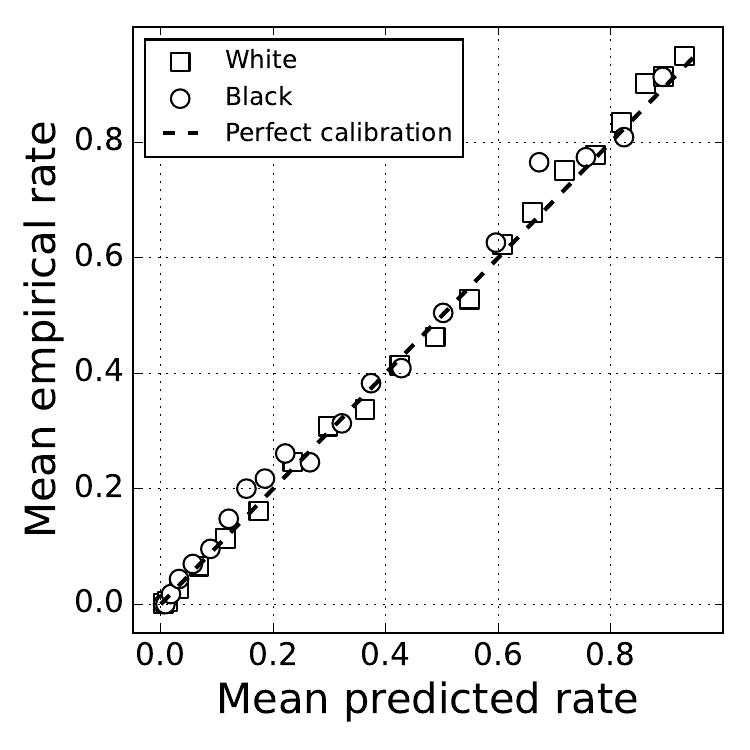}
            \phantomcaption{(a) New York}
            \label{fig:USCensus_calibration_NY}
        \end{subfigure}
        \begin{subfigure}[t]{0.3\linewidth}
            \centering
            \includegraphics[width=1.0\textwidth]{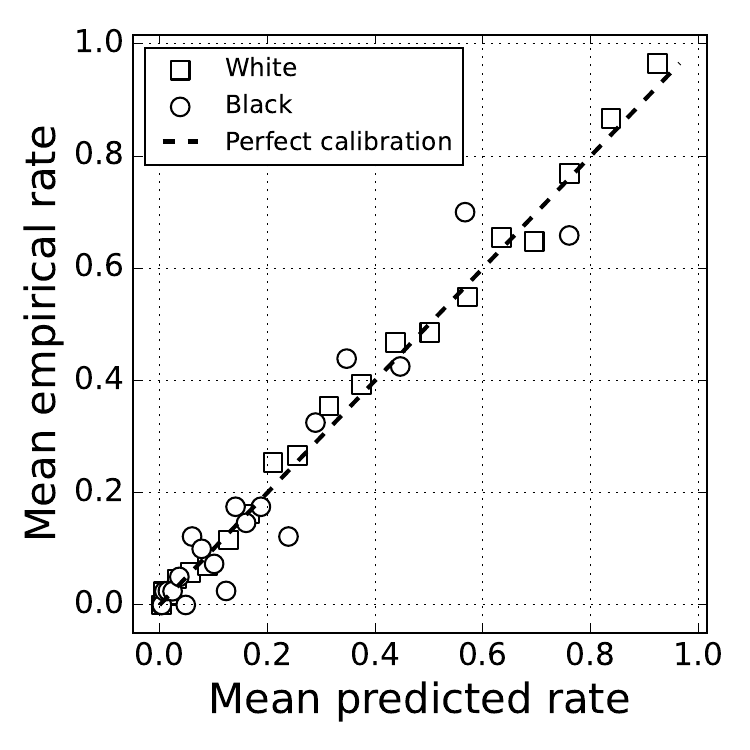}
          \phantomcaption{(b) Alabama}
          \label{fig:USCensus_calibration_AL}
       \end{subfigure}
       \begin{subfigure}[t]{0.3\linewidth}
            \centering
            \includegraphics[width=1.0\textwidth]{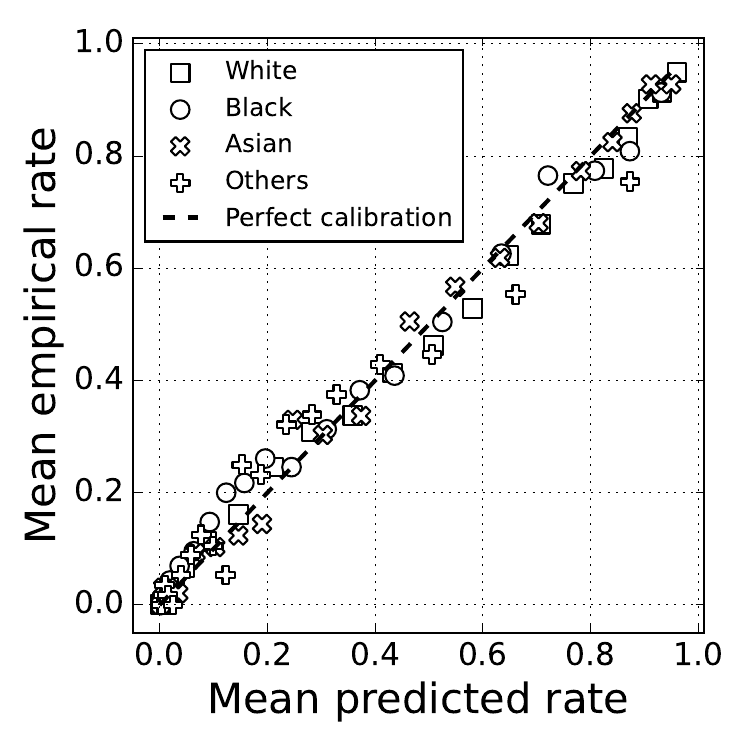}
          \phantomcaption{(c) New York}
          \label{fig:USCensus_calibration_4_groups}
       \end{subfigure}
       \caption{\normalfont Calibration plot for $\prob(r_i|\Data)$ for the state of New York and Alabama}
\end{figure*}

In Figure~\ref{fig:USCensus_figs}, estimates $\relCand(A|\sigma_k), \relCand(A), \relCand(B|\sigma_k),$ and $\relCand(B)$ are computed with the true relevance labels from the test set for computing EOR criterion, costs, and nDCG. Figure~\ref{fig:USCensus_no_labels}, shows EOR criterion and costs with $\relCand(A|\sigma_k), \relCand(A), \relCand(B|\sigma_k), \relCand(B)$ estimated from the calibrated $\prob(r_i|\Data)$. Note that the evaluation on true relevance labels in Figure~\ref{fig:USCensus_figs}, though noisier is qualitatively similar to the evaluation using the calibrated $\prob(r_i|\Data)$ in Figure~\ref{fig:USCensus_no_labels}. 
Additional experiment for two groups with true relevance labels for New York in Figure~\ref{fig:USCensus_NY_2_groups} (top) and with calibrated $\prob(r_i|\Data)$ in Figure~\ref{fig:USCensus_NY_2_groups} (bottom) further confirm our findings, that $\EOR$ is the only ranking policy that consistently achieves $\delta(\sigma_k)$ close to zero at every prefix $k$ with near optimal total cost to the principal.

Note the overlapping of $\RA$ and $\PRP$ in Figure~\ref{fig:USCensus_no_labels} and \ref{fig:USCensus_NY_2_groups}. This is expected because $\RA$ swaps the candidates in PRP ranking to satisfy proportional exposure as described in Appendix~\ref{baselines}. Since the amortized exposure between groups is already satisfied with the PRP ranking for this dataset, $\RA$ and $\PRP$ compute similar rankings.

\begin{figure*}[t!]
      \begin{subfigure}[t]{0.98\linewidth}
    \centering
    \includegraphics[width=1.0\textwidth]{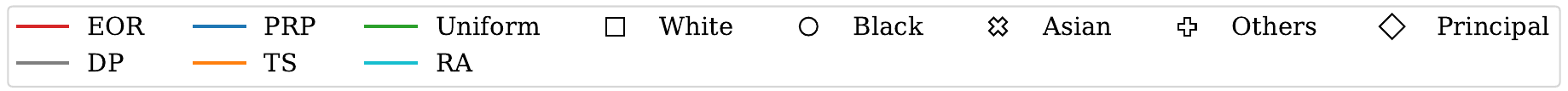}
    \vspace{-4mm}
  \end{subfigure}
        \includegraphics[width=0.98\textwidth]{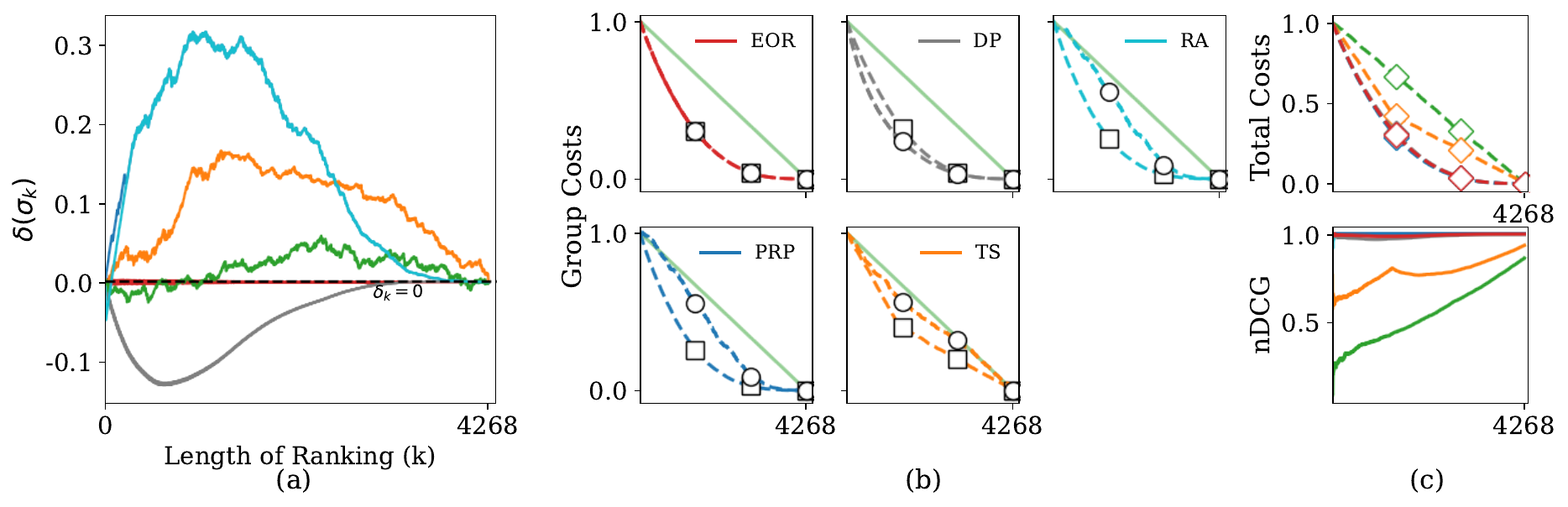}
    \label{fig:USCensus_no_labels_AL}
    \centering
    \includegraphics[width=0.98\textwidth]{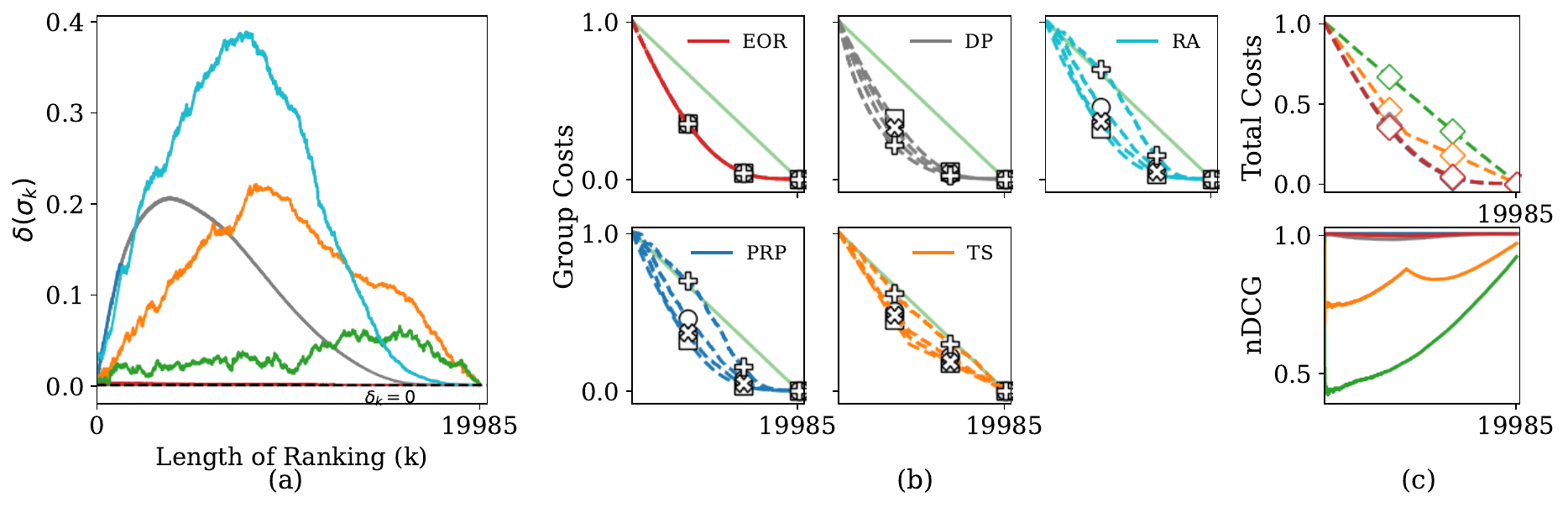}
    \label{fig:USCensus_no_labels_4_groups}
    \vspace*{-4mm}
  \caption{\textbf{Top:} \normalfont EOR criterion $\delta(\sigma_k)$ and Costs computed using calibrated $\prob(r_i|\Data)$ for two groups for the state of Alabama. \textbf{Bottom:} \normalfont EOR criterion $\delta(\sigma_k)$ and Costs computed using calibrated $\prob(r_i|\Data)$ for four groups for the state of New York.}
    \label{fig:USCensus_no_labels}
\vspace*{-2mm}
\end{figure*}

\begin{figure*}[t!]
      \begin{subfigure}[t]{0.98\linewidth}
    \centering
    \includegraphics[width=1.0\textwidth]
    {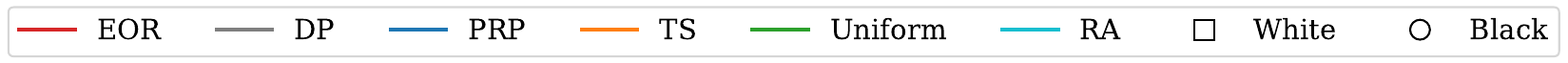}
    \vspace{-4mm}
  \end{subfigure}
    \includegraphics[width=0.98\textwidth]{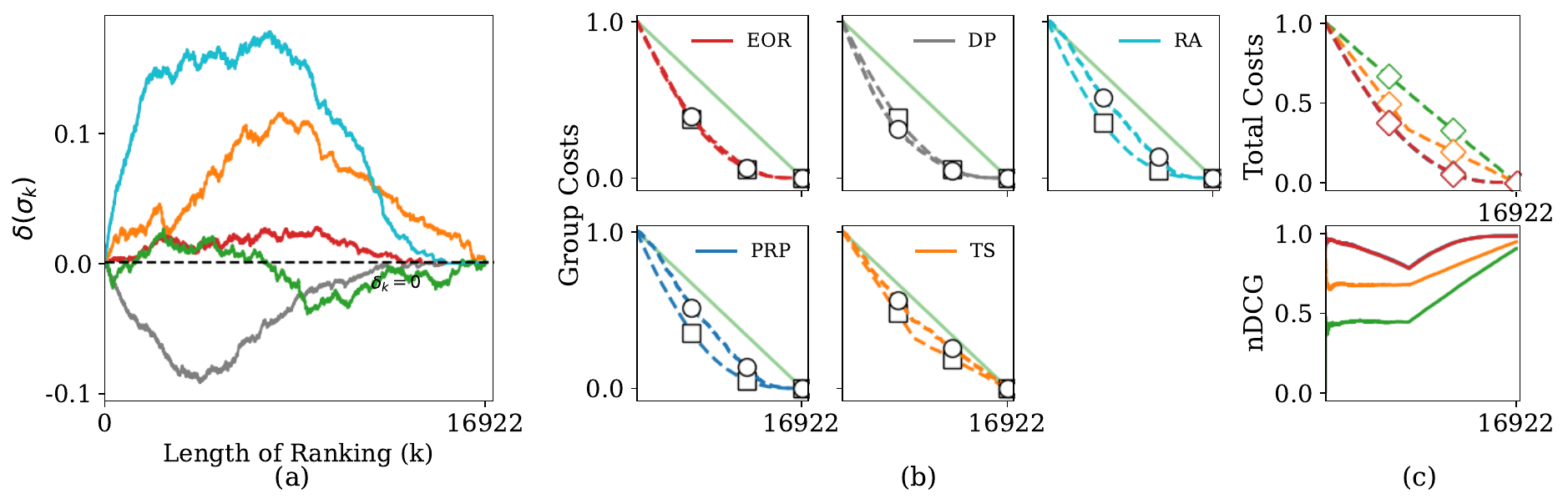}
    \label{fig:USCensus_labels_NY}
    \includegraphics[width=0.98\textwidth]{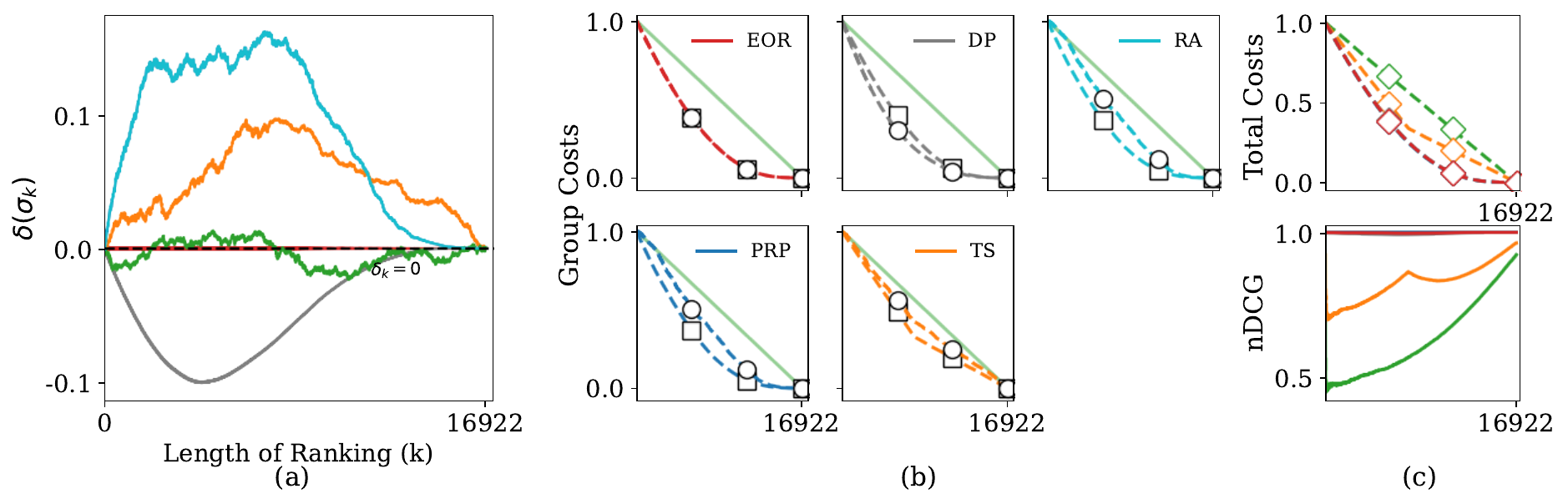}
    \label{fig:USCensus_no_labels_NY}
    \vspace*{-4mm}
    \caption{\textbf{Top:} \normalfont EOR criterion $\delta(\sigma_k)$ and Costs computed using true relevance labels from the test subset.  \textbf{Bottom:} \normalfont EOR criterion $\delta(\sigma_k)$ and Costs computed using calibrated $\prob(r_i|\Data)$ for the state of New York.}\label{fig:USCensus_NY_2_groups}
\vspace*{-2mm}
\end{figure*}

\subsection{Amazon shopping queries dataset} \label{amazon_details}
Amazon's shopping queries \citep{reddy2022shopping} consists of a large scale query-product pair dataset with baseline models for tasks related to predicting the relevance of items given a search query. Each query-product pair has an associated human annotated label of an exact, substitute, complement, or irrelevant label.

For our analysis, we focus on their task 1 of query-product ranking \footnote{\href{https://github.com/amazon-science/esci-data}{https://github.com/amazon-science/esci-data}} to sort the list of products in the decreasing order of relevance for every query. We use the publicly available baseline model for this task, consisting of Cross Encoders for the MS Marco dataset \citep{reimers-2019-sentence-bert}. This pretrained model encodes the query and product titles and is fine-tuned on the US part of the small version of training dataset. We use the default hyperparameters for the Cross Encoder as maximum length=512, activation function=identity,
and number of labels=1 (binary task). Similarly, for training following the default configuration, all exact labels are mapped to $1.0$, while the rest (substitute, complement, and irrelevant) are mapped to $0.0$. Default hyperparameter configuration includes MSE loss function, evaluation steps=5000,
warm-up steps=5000, learning rate=7e-6, training epochs=1, and number of development queries=400. Inference from the trained model provides relevance scores and we apply a sigmoid function to transform these scores to probabilities of relevance $\prob(r_i|\Data)$.

To evaluate the calibration of predicted $\prob(r_i|\Data)$, we use the test split of the dataset \citep{reddy2022shopping} for the large version containing 22,458. We filtered these queries so that they contain at least three products owned by one of the 158 brands owned by Amazon (we discuss in the next paragraph the source of identifying these Amazon-owned brands) and at least three products owned by brands other than Amazon. These result in 395 queries, out of which half are used for calibration with a Platt-scaling calibrator while the remaining half is used to evaluate the calibration curve for the test dataset. $\prob(r_i|\Data)$ of the query-product pairs for the remainder half of the test dataset after calibration is binned across 20 equal sized bins as shown in Figure~\ref{fig:calibration_amazon} and lies close to the perfectly calibrated line.

We further augmented this with another dataset \footnote{\href{https://github.com/the-markup/investigation-amazon-brands/tree/master}{https://github.com/the-markup/investigation-amazon-brands}} collected from the Markup report \citep{Yin2021}, which investigated Amazon's placement of its own brand products as compared to other brands based on star ratings, reviews etc. The authors for the Markup report identified 158 brand products that are trademarked by Amazon. We use these 158 brands to form the Amazon owned group. Products belonging to any other brand form the non-Amazon group.
Importantly, this dataset contains logged rankings from Amazon's website with 4566 queries for popularly searched query terms. We filtered these such that each query contains exactly 60 products and at least three of them are owned by Amazon, resulting in 1485 search queries. 

Next, we obtain relevance probabilities $\prob(r_i|\Data)$ from Amazon's pretrained baseline model described above and evaluate $\delta(\sigma_k)$ both for the logged ranking as well as our computed EOR ranking. Figure~\ref{fig:markup_eo} shows that our EOR ranking is closer to $\delta(\sigma_k)=0$ as compared to logged rankings on Amazon's platform. We note that this analysis is subject to confounding due to the use of features other than product titles that may be used in practice for logged rankings. However, the analysis does demonstrate how the EOR criterion can be used for auditing, if the auditor is given access to the production ranking model to avoid confounding.

\end{document}